\documentclass[10 pt]{article}
\usepackage{float}
\usepackage[utf8]{inputenc}
\usepackage[T1]{fontenc}
\usepackage[margin=1in]{geometry}
\usepackage{amsmath,amsfonts,amssymb}
\usepackage{mathtools}
\usepackage{amsthm}
\usepackage{bbm}
\usepackage{graphicx}
\usepackage{booktabs}
\usepackage{array}
\usepackage{algorithm}
\usepackage{algpseudocode}
\usepackage{medmath}
\usepackage[dvipsnames]{xcolor}
\usepackage{cite}
\usepackage[hidelinks]{hyperref}
\usepackage{orcidlink}
\usepackage{subcaption}
\newcommand{\mc}{\mathcal}
\usepackage{microtype}
\usepackage{xurl}
\newtheorem{problem}{Problem}

\newtheorem{theorem}{Theorem}
\newtheorem{lemma}{Lemma}

\newtheorem{assumption}{Assumption}
\newcommand{\TV}{\texttt{TV}}
\newcommand{\KL}{\texttt{KL}}

\renewcommand{\P}{\mathbb{P}}

\newcommand{\paren}[1]{\left( #1\right)}
\definecolor{mygreen}{rgb}{0.0, 0.5, 0.0}
\definecolor{winered}{rgb}{0.8,0,0}
\definecolor{myblue}{rgb}{0,0,0.8}
\newcommand{\smallsquare}{\scalebox{0.7}{$\blacksquare$}}
\hypersetup{
    colorlinks=true,
    linkcolor={winered},
    citecolor={myblue}
}
\usepackage[most]{tcolorbox}

\definecolor{abstractblue}{RGB}{232,242,250}
\definecolor{abstractborder}{RGB}{135,175,205}

\newtcolorbox{abstractbox}{
    enhanced,
    breakable,
    colback=abstractblue,
    colframe=abstractborder,
    boxrule=0pt,
    arc=2mm,
    left=3mm,
    right=3mm,
    top=2.5mm,
    bottom=2.5mm,
    before skip=8pt,
    after skip=10pt,
}
\definecolor{resultsred}{RGB}{180,0,0}
\definecolor{resultsbackground}{RGB}{253,242,242}

\newtcolorbox{theoreticalresults}[1][Theoretical Results]{
    enhanced,
    breakable,
    title={#1},
    colback=resultsbackground,
    colframe=resultsred,
    colbacktitle=resultsred,
    coltitle=white,
    fonttitle=\bfseries\large,
    boxrule=0.8pt,
    arc=2.5mm,
    outer arc=2.5mm,
    left=3mm,
    right=3mm,
    top=2.5mm,
    bottom=2.5mm,
    toptitle=1mm,
    bottomtitle=1mm,
    before skip=10pt,
    after skip=10pt,
    equal height group=frontresults
    titlerule=0pt,
}
\definecolor{experimentgreen}{RGB}{113,166,0}
\definecolor{experimentbackground}{RGB}{247,251,238}

\newtcolorbox{experimentalresults}[1][Experimental Results]{
    enhanced,
    title={#1},
    colback=experimentbackground,
    colframe=experimentgreen,
    colbacktitle=experimentgreen,
    coltitle=white,
    fonttitle=\bfseries\large,
    boxrule=0.8pt,
    arc=2.5mm,
    outer arc=2.5mm,
    left=3mm,
    right=3mm,
    top=2.5mm,
    bottom=2.5mm,
    toptitle=1mm,
    bottomtitle=1mm,
    equal height group=frontresults
}
\definecolor{coverblue}{RGB}{48,79,112}
\definecolor{coverbackground}{RGB}{247,249,252}

\newtcolorbox{coverresults}{
    enhanced,
    colback=coverbackground,
    colframe=black!25,
    boxrule=0.45pt,
    arc=1.5mm,
    outer arc=1.5mm,
    left=3mm,
    right=3mm,
    top=3mm,
    bottom=3mm,
    before skip=5pt,
    after skip=8pt
}
\usepackage{mdframed}

\surroundwithmdframed[
  topline=false,
  bottomline=false,
  rightline=false,
  leftline=true,
  linecolor=red!65!black,
  linewidth=2.5pt,
  innerleftmargin=10pt,
  innerrightmargin=0pt,
  innertopmargin=2pt,
  innerbottommargin=2pt,
  skipabove=10pt,
  skipbelow=10pt
]{theorem}

\surroundwithmdframed[
  topline=false,
  bottomline=false,
  rightline=false,
  leftline=true,
  linecolor=blue!65!black,
  linewidth=2.5pt,
  innerleftmargin=10pt,
  innerrightmargin=0pt,
  innertopmargin=2pt,
  innerbottommargin=2pt,
  skipabove=10pt,
  skipbelow=10pt
]{lemma}
\title{\textbf{Learning from Unreliable Trajectories: Adversarially-Robust Federated Q-Learning}}

\author{
\textbf{Sreejeet Maity \orcidlink{0009-0006-0209-6693}} \quad \textbf{Aritra Mitra \orcidlink{0009-0004-4590-553X}}\\
Department of Electrical and Computer Engineering\\
North Carolina State University, Raleigh\\
\texttt{\{\textcolor{winered}{smaity2},\textcolor{winered}{amitra2}\}@ncsu.edu}
}
\date{}

\begin{document}

\maketitle
\begin{abstractbox}
\begin{center}
    \textbf{\large Abstract}
\end{center}
\vspace{-0.5em}

We study federated reinforcement learning in which multiple agents interact with a common Markov decision process and communicate through a central server to collaboratively learn the optimal state-action value function. Our goal is to understand whether the sample-efficiency benefits of collaboration can be retained when a fraction of the agents behave adversarially and transmit arbitrarily corrupted information. To address
this problem, we introduce \textcolor{winered}{\texttt{Robust Async-Fed-Q}}, an epoch-based federated learning algorithm that combines variance-reduced estimation of the Bellman optimality operator at the agents with robust aggregation at the server. We establish high-probability finite-time guarantees showing that the proposed method preserves the statistical gains of collaboration
among the honest agents while tolerating adversarial corruption. In particular, the effect of the adversarial agents decreases as the amount of data collected by each honest agent grows and eventually vanishes in
the infinite-sample limit. We complement these guarantees with information-theoretic lower bounds that characterize the unavoidable statistical cost of adversarial corruption, leading to the first nearly
matching upper and lower bounds for adversarially robust federated reinforcement learning. We further extend our framework to accommodate single-trajectory Markovian sampling and heterogeneous partial coverage,
where different agents may explore different regions of the state-action space and learning relies on their collective coverage. Finally, our epoch-based design substantially improves the best known communication
complexity for federated Q-learning under asynchronous sampling.
\end{abstractbox}
\newpage
\tableofcontents
\newpage
\section{Introduction}
\label{sec:Intro}
Reinforcement learning (RL) has become a central framework for sequential decision-making, with applications spanning robotics, autonomous driving, games, and recommendation systems. Despite the promise of RL, it is now both empirically and theoretically understood that for contemporary applications with large state-action spaces, training a learning policy that achieves desirable performance can require several samples. A natural way to alleviate this sample-complexity bottleneck is to leverage data generated by multiple similar environments, modeled as Markov Decision Processes (MDPs). Motivated by the success of federated supervised learning~\cite{konevcny} in this regard, a new paradigm called federated reinforcement learning (FRL) has emerged that seeks to expedite the process of learning policies by merging information from multiple MDPs, while keeping raw trajectory data private. The fact that collaboration can indeed lead to improvements in sample-complexity has been observed in empirical studies~\cite{qiFRL}, and backed up by a recent line of theoretical work~\cite{jinFRL, khodadadian, woo2023blessing, dal, shen2023, wang2023TMLR, lan2023improved, liuMARL, zhang2024finite, tian2024one, salgia2024sample, beikmohammadi2024, mangold2024, zhu2024towards, zheng2024federated, labbi2024federated, wang2024momentum, Woo_communication}. 

However, the collaborative gains achieved in these works are predicated upon the idealistic assumption that the information collected from all agents is reliable and accurate. Such an assumption is hard to ensure or verify in practice, where certain agents can be faulty or even adversarial. The more realistic scenario where a portion of the training data is potentially corrupted leads to an interesting tension: while collaboration can reduce variance (and hence uncertainty), it can also introduce biases. This raises the question: \emph{Does more data help or hurt?} There is little to no theoretical understanding of this tension in the current FRL literature. In addition to this gap, existing FRL papers either suffer from high communication costs~\cite{jinFRL, wang2023TMLR, woo2023blessing, zhang2024finite, zhu2024towards, wang2024momentum}, and/or make the limiting assumption of \emph{full coverage}, i.e., each agent is assumed to have access to all state-action pairs of the MDP~\cite{wang2023TMLR, zhang2024finite, salgia2024sample, liuMARL}. In this work, we take a step towards enabling FRL in harsh, real-world environments by developing a novel algorithmic framework that accounts for adversarial robustness and partial state-space coverage, incurs modest communication costs, and achieves order-optimal finite-time rates under asynchronous, Markov sampling, backed up by (nearly) matching fundamental lower bounds. 

To formalize the study of adversarial robustness in FRL, we consider an infinite-horizon discounted RL setting involving $N$ agents that interact with a common MDP. These agents collaborate via a server to speed up the process of learning the optimal state-action value function. Inspired by the worst-case adversary model in distributed supervised learning~\cite{chensu, chen, yin}, we capture corruption by assuming that a fraction $\varepsilon \in [0, 1/2)$ of the agents can act \emph{arbitrarily}, and potentially disrupt the learning process. Given this premise, we ask whether (and to what extent) one can still hope for statistical gains from collaboration? If so, how much communication is needed to achieve such gains? To our knowledge, precise answers to these questions are missing. 

\textbf{Our Contributions.} We address the above research gap by making the following contributions. 
\newline
$\bullet$ \textbf{Novel Robust Federated \(Q\)-learning Algorithm.} In Section~\ref{sec:algorithm}, we propose a robust and communication-efficient federated \(Q\)-learning algorithm called \texttt{Robust Async-Fed-Q} for learning from asynchronous trajectory data in the presence of adversarial agents. Unlike the simplified generative/synchronous sampling model~\cite{kearns, even2003learning, sidford, li2024q} where each agent can observe every state-action pair at each time-step, under the more challenging single-trajectory asynchronous sampling model we consider here, only one state-action pair is observed at each time-step. This inherent sampling randomness can lead to high variance in the information received from even uncorrupted honest agents. Unless accounted for carefully, naively aggregating such noisy high-variance information can lead to vacuous bounds that feature persistent bias effects from adversarial corruption. This challenge remains even if one employs robust aggregation rules such as those  in~\cite{lugosi}. 

To navigate around this issue, our proposed approach departs fundamentally from the standard FRL algorithmic template~\cite{jinFRL, khodadadian, woo2023blessing, dal, shen2023, wang2023TMLR, lan2023improved, zhang2024finite,  mangold2024, zhu2024towards, wang2024momentum, Woo_communication} where agents perform multiple local (value or policy) parameter updates between communication rounds. Instead, we develop a federated \(Q\)-learning scheme that runs in epochs, where the samples collected within an epoch are used to construct a refined variance-reduced estimate of the Bellman optimality operator. The variance reduced estimates are robustly aggregated at the server, and used to make \emph{only one single update} to the \(Q\)-table at the end of the epoch. Thus, unlike existing FRL schemes that make multiple local updates per epoch/round using noisy directions, \emph{our approach relies on making a single update along a more precise variance-reduced direction.} Effectively, our algorithm makes synchronous updates to the entire \(Q\)-table using asynchronous data.  

$\bullet$ \textbf{Guarantees under i.i.d. and Markov Sampling.} Under the standard i.i.d. sampling model~\cite{korda,dalal,narayanan,lakshmi}, when each agent has access to $T$ samples, we prove a high-probability $\ell_{\infty}$ error bound of the form $\widetilde{\mc{O}}(1/\sqrt{\lambda_{\min} NT}) + \widetilde{\mc{O}}(\varepsilon/\sqrt{\lambda_{\min} T})$, where $\lambda_{\min}$ captures the frequency of visits to the least visited state-action pair; for a precise statement, see Theorem~\ref{theorem:theoremmainresult}. When $\varepsilon=0$, our bound is order-optimal in its dependence on $\lambda_{\min}, N$, and $T$, matches existing federated \(Q\)-learning rates~\cite{woo2023blessing}, and demonstrates a clear linear speedup in sample-complexity w.r.t. the number of agents $N$. Crucially, our Bellman operator estimation idea ensures that the corruption bias is not persistent, but rather vanishes as $T \to \infty$. This is the first result in FRL to simultaneously achieve a linear speedup and vanishing corruption bias in the presence of adversarial agents. In Theorem~\ref{theorem:markov}, using a coupling argument, we show that our results can be extended to the challenging Markovian sampling setting. To our knowledge, no prior work in robust distributed learning has dealt with \emph{time-correlated} trajectory data. 

$\bullet$ \textbf{Fundamental Lower Bounds.} While fundamental lower bounds for robust distributed supervised learning have been established in~\cite{yin}, no such bounds were known for robust FRL. In Theorem~\ref{thm:async_fedq_lower_bound}, we close this gap by establishing an information-theoretic lower bound. Our bound reveals that the $\varepsilon/\sqrt{T}$ corruption bias term is unavoidable, and represents the fundamental price of collaboration in the presence of adversarial agents. Collectively, the nearly matching upper and lower bounds in Theorems~\ref{theorem:theoremmainresult} and~\ref{thm:async_fedq_lower_bound} provide a fairly complete picture of the theme of adversarial robustness in FRL. 

$\bullet$ \textbf{Extension to Partial Coverage Setting.} To further demonstrate the generality of our framework, in Section~\ref{sec:partial}, we consider a realistic setting where each agent can only observe data from a subset of the state-action pairs, as opposed to the entire state-action space. To handle adversaries in this scenario, we identify a notion of ``informative source agents" associated with each state-action pair. Under intuitive notions of information redundancy for each source agent set, we derive finite-time rates in Theorem~\ref{theorem:theoremmainresult_S2} that are analogous to those in the full coverage setting. 

$\bullet$ \textbf{Improved Communication Complexity.} In addition to achieving near-optimal rates under realistic sampling models in the presence of adversarial agents, \emph{the communication complexity of our algorithm is a strict improvement over prior work on asynchronous federated RL.} Specifically, the communication costs in prior work either scale linearly with the number of samples per agent $T$~\cite{jinFRL, wang2023TMLR, zhang2024finite, wang2024momentum}, or linearly with the number of agents $N$~\cite{woo2023blessing, Woo_communication}. In sharp contrast, the communication cost incurred by our method is \emph{logarithmic} in both $N$ and $T$. This is a significant improvement relative to prior asynchronous FRL work, since both $N$ and $T$ can be large in modern distributed learning paradigms. 


\textbf{Related Work.} A large body of work has investigated the problem of adversarial robustness in distributed supervised learning~\cite{su_vaidya, yin, chen, chensu, blanchard, sundaram, pillutla}. In an attempt to extend these guarantees to cooperative decision-making,~\cite{dubey2022private, jadbabaie2022byzantine, mitra2022collaborative} investigate the presence of corrupted agents in stochastic multi-armed bandit problems, and develop variants of either sequential elimination or UCB (upper-confidence-bound) type algorithms. 
However, analogous results for general RL and stochastic approximation are scarce, especially when it comes to the finite-sample regime. More recently, robust multi-agent TD/\(Q\)-learning has been considered in~\cite{xie2023communication,ye2024resilient}. These works only provide asymptotic guarantees that fail to quantify statistical gains from collaboration.

Closest to our current work is the paper~\cite{fan2021fault} which studies robust policy gradient algorithms in FRL. There are several differences between~\cite{fan2021fault} and our work. In terms of results, \cite{fan2021fault} only provides convergence to a stationary point, and does not establish any global optimality result. Furthermore, even when there is no corruption, the bounds in~\cite{fan2021fault} are sub-optimal, and do not achieve the linear speedup effect. In contrast, our bounds pertain to estimation of the optimal state-action value function, are order-optimal, and achieve collaborative linear speedups with respect to (w.r.t.) the number of agents. Moreover, \cite{fan2021fault} neither considers Markov sampling nor the partial coverage setting. Finally, unlike our work, none of the papers~\cite{xie2023communication, ye2024resilient, fan2021fault} come with any lower bounds. 

\noindent\textbf{Comparison with Our Prior Work.} In a preliminary conference version~\cite{maitymitrafedq}, we  studied robust federated \(Q\)-learning under the basic i.i.d. synchronous sampling model. Our present work substantially extends upon~\cite{maitymitrafedq} by (i) considering the challenging single-trajectory Markovian sampling model, (ii) providing (nearly) matching information-theoretic lower bounds, and (iii) studying the partial coverage setting in which agents collectively cover the state-action space (but not necessarily in isolation); none of these issues were addressed in~\cite{maitymitrafedq}. Together, these contributions provide a comprehensive finite-time characterization of robust federated \(Q\)-learning under realistic sampling, coverage, and adversarial conditions.
\section{Background and Problem Formulation}
\label{sec:backgrounD}
\textbf{MDP Model.} An MDP \((\mc M)\) is represented by the tuple $(\mc S,\mc A,\mc P,R,\gamma)$, where $\mc S$ and $\mc A$ are finite state and action spaces, $\mc P(\cdot\mid s,a)$ is the transition law associated with state-action pair $(s,a)$, and $\gamma\in(0,1)$ is the discount factor. In this work, rewards are stochastic: when an agent takes action $a$ at state $s$, it observes a stochastic reward sample $r(s,a) \sim \mc{R}(s,a)$ with mean $R(s,a)=\mathbb E[r(s,a)]$; for clarity of exposition, we assume the rewards are uniformly bounded\footnote{The analysis can be readily  extended to sub-Gaussian reward distributions using standard arguments.}, i.e., $\exists \bar{R} \geq 1$ such that $|r(s,a)|\le \bar R$ for all $(s,a)\in\mc S\times\mc A$. The next state is then drawn according to $s'\sim\mc P(\cdot\mid s,a)$. A deterministic policy $\pi:\mc S\to\mc A$ maps states to actions, and its quality is measured by the discounted infinite-horizon value function 
\(V^{(\pi)}(s)=\mathbb E[\sum_{t=0}^{\infty}\gamma^t r(s_t,a_t)\mid s_0=s,\pi],\)
where $s_t$ and $a_t$ represent the state and action at time $t$, respectively, under the action of policy $\pi$ on $\mc{M}$. The basic goal in reinforcement learning is to find an optimal policy $\pi^\star$ that maximizes $V^{(\pi)}(s)$ simultaneously for all $s\in\mc S$, without prior knowledge of the transition kernel $\mc P$ or the mean reward function $R$. To explain how this is done, consider the state-action value function for a policy $\pi$:  
\begin{equation}\label{eqn:dis_Q}
    Q^{(\pi)}(s,a)=\mathbb E\left[\sum_{t=0}^{\infty}\gamma^t r(s_t,a_t)\mid (s_0,a_0)=(s,a), \pi\right].
\end{equation} 

The classical \(Q\)-learning algorithm~\cite{watkins1992q} learns the optimal state-action value function $Q^*$ by exploiting the fact that $Q^*$ is the fixed point of the Bellman optimality operator $\mc{T}^*$, i.e., $\mc{T}^* Q^* =Q^*$, and that $\mc{T}^*$ is strictly contractive in the infinity norm~\cite{suttonRL}. While the asymptotic convergence of \(Q\)-learning has been studied in~\cite{tsitsiklis94, csabaasymp}, more recent papers~\cite{wainwright, Qu, li2024q} provide finite-time rates, showing that after $T$ iterations (samples), the iterates generated by \(Q\)-learning converge to $Q^*$ at a rate of $\widetilde{\mathcal O}(1/\sqrt{T})$ with high probability. 

\textbf{Our setting.} We consider a collaborative RL setting involving $N$ agents that interact with a common MDP \(\mc M\) and exchange information via a central server. As is common in FL, the agents do not share raw trajectories (i.e., states, actions, and rewards); instead, they transmit local statistical summaries that are aggregated by the server to estimate the optimal \(Q\)-function \(Q^\star\). A key modeling issue is the extent to which individual agents can access the state-action space. In this context, we first introduce the ``full-coverage setting".

\((\textcolor{winered}{\texttt{S}_1})\)~\textbf{Full-coverage setting and Asynchronous Sampling}.
In this setting, every agent can acquire samples about the entire state-action space \(\mc S\times\mc A\). To do so, each agent $i$ plays a behavior policy $\mu_i$ that generates an online trajectory \((s_{i,t},a_{i,t},r_{i,t},s_{i,t+1})\), where $s_{i,t}, a_{i,t},$ and $r_{i,t}$ represent the state, action, and (noisy) reward, respectively, of agent $i$ at time-step $t$. Such a sampling model is called the \emph{asynchronous} model since only one state-action pair is observed at each time-step. To ensure that asynchronous sampling provides information about every state-action pair in setting \((\textcolor{winered}{\texttt{S}_1})\), we impose the following standard coverage condition on each agent \(i\in[N]\)~\cite{tsitsiklis94, tsitsiklisroy,  bhandari_finite, Qu, li2024q, wang2023TMLR, khodadadian}.

\begin{assumption}
\label{ass:ergodic}
For each agent $i \in [N]$, the Markov chain induced by the behavior policy \(\mu_i\) is aperiodic and irreducible. Moreover, \(\mu_i(a|s)>0\) for every \((s,a)\in\mc S\times\mc A\).
\end{assumption}

Under Assumption~\ref{ass:ergodic}, the Markov chain induced by \(\mu_i\) admits a unique stationary distribution \(\pi_i\) with \(\pi_i(s)>0\) for all \(s\in\mc S\). Hence, at stationarity, the visitation probability of \((s,a)\) for agent \(i\) is \(\lambda_i(s,a):=\pi_i(s)\mu_i(a|s)\), which is strictly positive for every state-action pair. We define the minimum visitation probability by \(\lambda_{\min}:=\min_{i\in[N]}\min_{(s,a)\in\mc S\times\mc A}\lambda_i(s,a)\). The full-coverage condition in Assumption~\ref{ass:ergodic} is thus encoded by \(\lambda_{\min}>0\), ensuring that each agent visits every state-action pair infinitely often. To convey our main ideas, we will assume the \emph{asynchronous i.i.d. sampling}  model~\cite{korda,dalal,narayanan,lakshmi} in Sections~\ref{sec:algorithm}, \ref{sec:main_result_S_1}, and~\ref{sec:partial} where, for each agent $i$, at each time-step $t$, the state $s_{i,t}$ is sampled \emph{independently} from its stationary distribution $\pi_i$. Later, in Section~\ref{sec:lMarkov}, we will relax this assumption, and consider the challenging single-trajectory Markovian setting.

As is standard in FRL, we assume that the agents' trajectories are statistically independent~\cite{woo2023blessing,khodadadian,wang2023TMLR}. Suppose each agent $i \in [N]$ acquires $T$ samples. Given that data across agents are independent, intuition dictates that via collaboration with the server, each agent can learn $Q^*$ faster than it could in isolation. Indeed, this intuition has recently been formalized in~\cite{Woo_communication}, showing that one can achieve a sample-complexity rate of $\widetilde{\mc{O}}(1/\sqrt{NT})$, thus achieving a linear speedup w.r.t. the number of agents $N$. In this paper, we investigate to what extent such speedups are attainable (via algorithms and fundamental lower bounds) when a portion of the agents act adversarially as per the following model. 

\textbf{Adversarial Agent Model for} \((\textcolor{winered}{\texttt{S}_1})\)\textbf{.}
We allow a fraction \(\varepsilon\in[0,1/2)\) of agents to be adversarially corrupted based on the worst-case corruption model in robust distributed learning~\cite{chen,chensu,yin}. Specifically, denoting by \(\mc B\subset [N]\) the set of corrupted agents, each agent $i \in \mc{B}$ is assumed to have full knowledge of the MDP, the agents' data, and the algorithm. Using such knowledge, agents in $\mc{B}$ can act \emph{arbitrarily}, transmitting biased (and even unbounded) messages to the server. Given this premise, we can formally state our problem of interest below. 

\begin{problem}
\label{prob:Problem}
 Suppose each agent acquires $T$ samples using its behavior policy. Given a confidence parameter $\delta \in (0,1)$, our goal is to develop a robust federated \(Q\)-learning algorithm that uses data from the $N$ agents to compute an estimate $\hat{Q}$ of $Q^*$ such that with probability $1-\delta$, the $\ell_{\infty}$ error $ 
 \Vert \hat{Q} - Q^* \Vert_{\infty}$ meets the following criteria: (i) decays to zero as $T \to \infty$, despite adversaries; and (ii) decays at the optimal $1/\sqrt{NT}$ rate in the absence of adversaries.
\end{problem}

Concretely, we ask: (\textbf{Q1.}) Is it possible to devise an approach that retains collaborative statistical gains under asynchronously collected data despite adversaries? (\textbf{Q2.}) If so, how much communication is needed to achieve such gains? (\textbf{Q3.}) What are the fundamental performance limits imposed by the presence of adversarial agents? (\textbf{Q4.}) What happens when each agent can only access a portion of the state-space (partial coverage setting \((\textcolor{winered}{\texttt{S}_2})\) in Section~\ref{sec:partial})? To our knowledge, precise answers to these questions are lacking in the FRL literature. In Section~\ref{sec:algorithm}, we introduce a framework that addresses this research gap. Before doing so, we close this section by highlighting why this is a non-trivial task. 

\textbf{Challenges.}
Several technical challenges arise in our setting. First, as alluded to in the Introduction, asynchronous sampling even from the same MDP can lead to different agent-level trajectories. In particular, the number of visits to each state-action pair is a random object that can vary across agents. Unless carefully accounted for, the variance stemming from such sampling randomness can cause the estimates generated by even the honest uncorrupted agents to appear quite different to the server. Thus, our problem requires disentangling statistical uncertainty from the uncertainty injected by adversaries. Second, the above challenge becomes even more pronounced in the partial coverage setting (Section~\ref{sec:partial}) where all honest agents do not necessarily receive data from the same set of state-action pairs, leading to \emph{heterogeneity in observations}. Third, our RL problem requires contending with time-correlated, Markovian data, which is known to be challenging even in a single-agent setting without adversaries. We next develop the algorithmic ideas needed to overcome these challenges and achieve optimal rates.
\section{Robust Asynchronous Federated \texorpdfstring{\(Q-\)}~Learning}
\label{sec:algorithm}
In this section, we introduce \textbf{Robust Asynchronous Federated \texorpdfstring{\(Q\)-}{Q-}Learning}, abbreviated throughout as \texttt{Robust Async-Fed-Q}, and outlined in Algorithm~\ref{alg:async-fedq}. Algorithm~\ref{alg:async-fedq} involves two main ideas: (i) \textbf{Operator Refinement}: batching the online data into epochs, and using such batched data to compute ``low-variance" estimates of the Bellman optimality operator; and (ii) \textbf{Robust Aggregation}: aggregating such estimates using a trimmed mean estimator from~\cite{lugosi}. Crucially, we emphasize here that naively using the estimator from~\cite{lugosi} can lead to vacuous bounds and incur high communication costs. Thus, the operator estimation step is the main novel component of our approach. 

We first provide the high-level template of our algorithm. The horizon $T$ is partitioned into $K$ epochs of length $H$ each (i.e., $T=KH$), where $K,H$ are design parameters that will be specified later in Section~\ref{sec:main_result_S_1}. In each epoch $k =0, 1, \ldots, K-1$, the server first broadcasts an estimate $Q_k$ of $Q^*$ to every agent. Every uncorrupted (good) agent $i$ uses $H$ asynchronous samples (collected by playing $\mu_i$) to construct an accurate estimate of the object $\mc{T^*}Q_k$, which is then transmitted to the server. The server then aggregates these estimates using the robust mean estimation primitive from~\cite{lugosi} to generate the next \(Q\)-table estimate $Q_{k+1}$. We now proceed to supply the details.

\(\blacksquare\)~\textbf{Operator Refinement.} Let the time indices corresponding to the $k$-th epoch be denoted by \(\mc I_k:=\{kH, kH+1, \ldots,(k+1)H-1\}\). In epoch $k$, every agent $i$ interacts with the MDP $\mc{M}$ by playing its behavior policy $\mu_i$, and collects data tuples of the form $(s_{i,t}, a_{i,t}, r_{i,t}, s_{i,t+1}), t \in \mc{I}_k$. Under the asynchronous i.i.d. sampling model, $s_{i,t}$ is sampled independently from $\pi_i$, $a_{i,t} \sim \mu_i(\cdot \mid  s_{i,t})$, $r_{i,t} \sim \mc{R}(s_{i,t}, a_{i,t})$, and $s_{i,t+1} \sim \mc{P}(\cdot \mid s_{i,t}, a_{i,t})$. The $H$ samples thus collected during epoch $k$ are used to generate an estimate $d_{i,k}$ of the object $\mc{T}^*Q_k$ at the ``frozen" table $Q_k$ from the beginning of the epoch. To see how this is done, fix a pair $(s,a)$, and let $N_{i,k}(s,a)
 :=
 \sum_{t \in \mc{I}_k}
 \mathbbm{1}\{(s_{i,t}, a_{i,t})=(s,a)\}$ denote the number of visits to this pair made by agent $i$ during epoch $k$. 
If \(N_{i,k}(s,a)=0\), then no information about $(s,a)$ is acquired by agent $i$ during epoch $k$, and hence, we set \(d_{i,k}(s,a)=0\). Else, if  \(N_{i,k}(s,a)>0\), agent \(i\) forms the local estimate \(d_{i,k}(s,a)\) as follows:
\begin{equation}\label{eqn:dik_update}
d_{i,k}(s,a)
 :=
 \frac{1}{N_{i,k}(s,a)}
 \sum_{j=1}^{N_{i,k}(s,a)}
 \left(
 r_j(s,a)
 +
 \gamma
 \sum_{s'\in\mathcal{S}}
 \mathbbm{1}^{(i,j,k)}(s'|s,a)
 \max_{a'\in\mc A} Q_k(s',a')
 \right),
\end{equation}
where $r_j(s,a)$ is the noisy reward observed by agent $i$ during its $j$-th visit to pair $(s,a)$, and the indicator random variable \(\mathbbm{1}^{(i,j,k)}(s'\mid s,a)\) equals \(1\) if at agent $i$'s $j$-th visit to pair $(s,a)$ during epoch $k$, the next state is $s'$. The quantity \(d_{i,k}(s,a)\) should be interpreted as agent \(i\)'s local empirical approximation of the Bellman target \((\mc T^*Q_k)(s,a)\). This step is the operator-refinement step of the algorithm. Indeed, a standard one-step \(Q\)-learning update uses a single noisy Bellman sample~\cite{watkins1992q} of the form \(r+\gamma\max_{a'}Q_k(s',a')\), which can have high variance due to both reward noise and transition randomness. If such one-sample estimates were sent directly to the server, then even good agents could produce highly dispersed messages, making it difficult for the server to distinguish statistical noise from adversarial corruption. Instead, \texttt{Robust Async-Fed-Q} freezes \(Q_k\) throughout epoch \(k\) and lets each agent average all the Bellman samples collected for \((s,a)\) within that epoch. Since all these samples are centered around the same target \((\mc T^*Q_k)(s,a)\), the averaging in~\eqref{eqn:dik_update} produces a lower-variance estimate. This variance reduction is \emph{crucial for robustness}: once the honest messages are more tightly concentrated around the common Bellman target, the robust aggregation step at the server can more effectively suppress the influence of corrupted messages.

\(\blacksquare\)~\textbf{Robust Aggregation.} At the end of epoch \(k\), for each pair $(s,a)$, the server receives the collection \(\{d_{i,k}(s,a)\}_{i=1}^{N}\). For honest agents, conditioned on $Q_k$, these messages are independent estimates of \((\mc T^*Q_k)(s,a)\); however, they need not be identically distributed because the visit counts \(N_{i,k}(s,a)\) may differ across agents. Corrupted agents, on the other hand, may send arbitrary values. To mitigate such corruption, we leverage the scalar trimmed mean estimator from~\cite{lugosi}. Given a data set $\widetilde{\mc{D}}$ comprising of independent data samples of a scalar random variable $\mathbb{X}$, the approach in~\cite{lugosi} yields a robust estimate $\hat{\mu}_{\mathbb{X}}$ of the mean $\mu_{\mathbb{X}}$ of $\mathbb{X}$, even when $\varepsilon$ fraction of the samples in $\widetilde{\mc{D}}$ are arbitrarily corrupted. Succinctly, let $\hat{\mu}_{\mathbb{X}}:=\texttt{trimSC}\bigl(\widetilde{\mc{D}},\varepsilon,\delta\bigr)$ be the output of the trimmed mean estimator in~\cite{lugosi} when fed with data set $\widetilde{\mc{D}}$, corruption fraction $\varepsilon$, and a confidence parameter $\delta \in (0,1)$; the details of this process and the associated guarantees are provided in Appendix~\ref{app:rob_mean} for completeness.  

Using the estimator from~\cite{lugosi}, for each $(s,a) \in \mc{S} \times \mc{A}$, the server then constructs a robust estimate \(\widetilde d_k(s,a)\) of \((\mc T^*Q_k)(s,a)\) as follows: 

\begin{equation}\label{eqn:server_update}
\widetilde d_k(s,a)
:=
\texttt{trimSC}\bigl(\{d_{i,k}(s,a)\}_{i=1}^N,\varepsilon,\delta_1\bigr),
\end{equation}
where \(\delta_1:=\delta/(2T|\mc S||\mc A|).\) With these robust estimates in hand, at the end of epoch $k$, the server finally updates each component $(s,a)$ of the \(Q\)-table with a step-size \(\eta\in(0,1]\): 
\begin{equation}\label{eqn:server_Q_update}
    Q_{k+1}(s,a)
    =
    (1-\eta)Q_k(s,a)
    +
    \eta\,\widetilde d_k(s,a).
\end{equation}

\begin{algorithm}[t]
\caption{\texttt{Robust Async-Fed-Q} under \((\textcolor{winered}{\texttt{S}_1})\)}
\label{alg:async-fedq}
\begin{algorithmic}[1]
\Require Total samples \(T\), number of epochs \(K\), stepsize \(\eta\), discount factor $\gamma$, corruption fraction \(\varepsilon\), confidence parameter \(\delta\), behavior policies \(\{\mu_i\}_{i \in [N]}\).
\State Initialize \(Q_0(s,a)\gets 0,~d_{i,0}(s,a) \gets 0\) for all \(i \in [N]\), and \((s,a)\in\mc S\times\mc A\).
\For{epoch \(k=0,\dots,K-1\)}
    \State Server broadcasts \(Q_k\) to all agents.
    \For{each good agent \(i\in[N] \setminus \mc{B} \) in parallel}
        \State Agent \(i\) follows \(\mu_i\) for \(H\) steps and collects data tuples of the form $(s_{i,t}, a_{i,t}, {r}_{i,t}, s_{i,t+1})$.
        \State For each $(s,a) \in \mc{S} \times \mc{A}$, if $N_{i,k}(s,a)=0$, agent $i$ sets $d_{i,k}(s,a)=0$; else, it computes $d_{i,k}(s,a)$ via~\eqref{eqn:dik_update}. It then sends \(\{d_{i,k}(s,a)\}_{(s,a)\in\mc S\times\mc A}\) to the server.
    \EndFor
    \State For each \((s,a)\in\mc S\times\mc A\), the server computes \(\widetilde d_k(s,a)\) via \eqref{eqn:server_update} and \(Q_{k+1}(s,a)\) via  \eqref{eqn:server_Q_update}.
\EndFor
\State \Return \(Q_K\).
\end{algorithmic}
\end{algorithm}

\textbf{Key Messages.} This completes the description of  \texttt{Robust Async-Fed-Q}. It is worth highlighting that \emph{our approach departs fundamentally from standard FRL update rules}~\cite{jinFRL, khodadadian, wang2023TMLR, zhang2024finite, woo2023blessing, Woo_communication}, where each agent updates its \(Q\)-table locally multiple times in each epoch/round. This creates a ``drift" in the \(Q\)-tables held by different agents, a fact that can be leveraged by the corrupted agents to amplify uncertainty at the server. In sharp contrast, in Algorithm~\ref{alg:async-fedq}, agents do not update their \(Q\)-tables at all within an epoch; instead, they use the $H$ samples collected within the epoch to construct a \emph{variance-reduced} accurate estimate of the Bellman optimality operator. This simple yet key idea has two important implications. First, as evidenced by Lemma~\ref{lem:robust_aggregation_error}, it allows us to eliminate the adversarial bias effect by increasing the number of samples $T$; this is essential for achieving near-optimal rates. Second, note that the number of \(Q\)-table updates equals the communication frequency $K$. In our approach, since each update is made along a precise low-variance direction, it turns out that just $K = \widetilde{\mc{O}}(1)$ updates/communication rounds suffice, yielding a highly communication-efficient scheme. Finally, in addition to being robust, sample-efficient, and incurring low communication, Algorithm~\ref{alg:async-fedq} can also be implemented efficiently since it only involves \emph{scalar} robust mean estimation. 
\section{Main Results for \texttt{Robust Async-Fed-Q}}\label{sec:main_result_S_1}
In this section, we present the main results for the full-coverage setting \((\textcolor{winered}{\texttt{S}_1})\). We first establish a high-probability upper bound for \texttt{Robust Async-Fed-Q}, showing that the algorithm achieves a near-optimal statistical rate that preserves collaborative gains for small values of the corruption fraction $\varepsilon$. We then prove a matching information-theoretic lower bound. To state our results, for each \(k\ge 0\), define \(Q_k := [Q_k(s,a)]_{(s,a)\in\mc S\times\mc A}\) and \(e_k := \|Q_k-Q^*\|_{\infty}\). We then have the following finite-time rate for Algorithm~\ref{alg:async-fedq}. 
\begin{theorem}\textbf{(Finite-time rates under full coverage)}
\label{theorem:theoremmainresult}
Suppose Assumption~\ref{ass:ergodic} holds. Fix any confidence parameter \(\delta\in(0,1)\), and suppose \(\eta=\log(NT)/((1-\gamma)K)\), \(K=\left\lceil c_1\log(NT)/(1-\gamma)\right\rceil\) where \(c_1>0\) is a universal constant that ensures \(\eta<1\), \(N \geq \left\lceil
    2\log\!\left(16T|\mc S||\mc A|/\delta\right)
    \right\rceil,\) and \(T \ge (100 K/\lambda_{\min})\log\left(4|\mc S||\mc A|NT/\delta\right)\). Then, the output \(Q_K\) of Algorithm~\ref{alg:async-fedq} satisfies, with probability (w.p.) at least \(1-\delta\),
\begin{equation} 
\begin{aligned}
e_K
\le
\frac{e_0}{NT}
&+
\underbrace{\mathcal{O}\left(
\frac{
\bar R
\sqrt{\log(NT)\log\left(|\mathcal S||\mathcal A|T/\delta\right)}
}{
(1-\gamma)^{5/2}\sqrt{\lambda_{\min}NT}
}
\right)}_{T_1} +
\underbrace{\mathcal{O}\left(
\frac{
\bar R\varepsilon\sqrt{\log(1/\varepsilon)}\sqrt{\log(NT)}
}{
(1-\gamma)^{5/2}\sqrt{\lambda_{\min}T}
}
\right)}_{T_2}.
\end{aligned}
\label{eqn:main_conv_bnd}
\end{equation}
\end{theorem}
\textbf{Discussion.}
Theorem~\ref{theorem:theoremmainresult} gives a finite-time convergence rate for \texttt{Robust Async-Fed-Q} in the full coverage \((\textcolor{winered}{\texttt{S}_1})\) setting. The result shows that, despite the presence of corrupted agents, the final iterate \(Q_K\) remains close to the optimal \(Q\)-function \(Q^*\) with high probability. The bound in~\eqref{eqn:main_conv_bnd} consists of three terms: an asymptotically negligible $\mc{O}(1/(NT))$ term, a dominant statistical term $T_1$ and a corruption-induced bias term $T_2$. We discuss these terms below. 

\(\bullet\) \textbf{Near-optimal statistical rate and collaboration benefit.}
In the absence of corruption, i.e., when \(\varepsilon=0\), the dominant statistical term in~\eqref{eqn:main_conv_bnd} is $T_1 = \widetilde{\mc{O}}(1/\sqrt{\lambda_{\min} NT})$. When $N=1$, this rate matches single-agent \(Q\)-learning rates~\cite{wainwright, Qu, li2024q}. For $N > 1$, this term provides the optimal $\sqrt{N}$-factor improvement relative to the single-agent case, representing statistical benefits of collaboration~\cite{Woo_communication}.  

\(\bullet\) \textbf{Diminishing corruption effect.}
The term $T_2 = \widetilde{\mc{O}}(\varepsilon/\sqrt{\lambda_{\min} T})$ in~\eqref{eqn:main_conv_bnd} captures the adversarial bias effect, and is consistent with previous results on distributed supervised learning~\cite{yin, chensu}. Thus, when $\varepsilon$ is sufficiently small, our approach preserves the linear speedup property due to collaboration. To our knowledge, this is the first such result in the context of robust federated \(Q\)-learning. A key feature of the bias term $T_2$ is that it is not a fixed error floor: this term decreases as the per-agent sample size \(T\) grows. Thus, unlike robust learning guarantees where adversarial contamination leaves a persistent asymptotic bias, \texttt{Robust Async-Fed-Q} can wash out the effect of corrupted agents in the large-sample limit. This essential feature is a direct consequence of the operator refinement step in Algorithm~\ref{alg:async-fedq} which leads to variance reduction. Later in this section, we show that the dependence on the corruption fraction $\varepsilon$ we achieve is information-theoretically unavoidable. 

\(\bullet\) \textbf{Little communication overhead.} 
Algorithm~\ref{alg:async-fedq} requires only $K =  \Theta(\log(NT)/(1-\gamma))$ communication rounds, with no linear dependence on the number of samples $T$ as in~\cite{wang2023TMLR, zhang2024finite}, nor on the number of agents $N$ as in~\cite{Woo_communication}. Furthermore, the communication complexity of Algorithm~\ref{alg:async-fedq} is \emph{unaffected} by the state-visitation probabilities unlike~\cite{Woo_communication}, where the number of communication rounds scales inversely with such probabilities. Since both $T$ and $N$ can be large in a federated setting, and certain state-action pairs might be visited infrequently, our approach leads to a dramatic reduction in communication costs relative to these papers, while also accounting for the key challenge of adversarial robustness. In fact, \emph{for FRL under asynchronous sampling, the amount of communication needed to achieve a linear speedup for Algorithm~\ref{alg:async-fedq} is a strict improvement over prior work; this is an important contribution of our paper.} Once again, this improvement can be attributed to the central idea of constructing low-variance refined Bellman operators within each epoch (as opposed to making multiple \(Q\)-table updates in standard FRL algorithms).  

\textbf{Proof Sketch for Theorem~\ref{theorem:theoremmainresult}.} The key step in our analysis is to control the gap between the estimated update direction $\widetilde d_k$ and the ideal update direction $\mc T^*Q_k$. This is achieved in the following result which forms our main new technical tool. 

\begin{lemma}\textbf{(Robust Aggregation Error)}
\label{lem:robust_aggregation_error}
Suppose the conditions of Theorem~\ref{theorem:theoremmainresult} hold. Then, w.p. at least \(1-\delta\), the following bound holds simultaneously for all epochs $k \in [K]$:
\begin{equation}
\label{eqn:robust_aggregation_error}
\left\lVert \widetilde d_k-\mc T^*Q_k \right\rVert_{\infty}
\le
\mathcal{O}\left(
\frac{\bar R \sqrt{\log\left(|\mc S||\mc A|T/\delta\right)}}{(1-\gamma)\sqrt{\lambda_{\min}NH}}
\right)
+
\mathcal{O}\left(
\frac{\bar R\,\varepsilon\sqrt{\log(1/\varepsilon)}}{(1-\gamma)\sqrt{\lambda_{\min}H}}
\right).
\end{equation}
\end{lemma}

The collaboration benefit and the vanishing corruption term in~\eqref{eqn:main_conv_bnd} are both driven by Lemma~\ref{lem:robust_aggregation_error}. Let us sketch out the intuition behind this result. Within each epoch, each good agent $i$ visits every state-action pair at least $\Omega(\lambda_{\min} H)$ times with high-probability. Since all such samples are used to estimate $d_{i,k}$, the variance of such an estimate scales inversely with $\lambda_{\min}H$. Since all estimates from good agents enjoy such a $H$-fold variance reduction, the corruption term also gets hit by this lower variance. Intuitively, the variance reduction effect makes it harder for the adversarial agents to amplify uncertainty at the server. Since $T=KH$, and \(K=\Theta(\log(NT)/(1-\gamma))\), the adversarial bias essentially scales inversely with the total per-agent sample-size $T$ (up to log factors). The main technical hurdle in proving Lemma~\ref{lem:robust_aggregation_error} stems from the fact that the number of visits to each state-action pair in each epoch is a random variable that can vary across agents. A naive union bound over the random visitation counts yields vacuous guarantees that obscure the linear speedup. We instead carefully condition on the complete vector of visitation counts across agents, and exploit properties of sub-Gaussian random variables. With Lemma~\ref{lem:robust_aggregation_error} in hand, Theorem~\ref{theorem:theoremmainresult} can be proven using the contraction and fixed-point properties of the Bellman optimality operator $\mc{T}^*$. For detailed proofs, see  Appendix~\ref{app:proof_of_upper_bound_S1}.

\textcolor{winered}{\(\smallsquare\)} \textbf{Information-Theoretic Lower Bound.} The dominant statistical term $\widetilde{\mc{O}}(1/\sqrt{\lambda_{\min} NT})$ in our main bound~\eqref{eqn:main_conv_bnd} is unimprovable, as evidenced by the lower bounds in~\cite{li2024q}. To demonstrate the overall tightness of our bounds, we now argue that the corruption-induced bias $\widetilde{\mc{O}}(\varepsilon/\sqrt{T})$ is also unavoidable. To do so, it suffices to consider a class of MDPs $\mc{H}$ with just one state-action pair $(s,a)$, where the only randomness comes from the reward distribution $\mc{R}(s,a)$. We consider uni-variate Gaussian reward distributions with mean $\mu$ and variance $\bar{R}^2$. An instance $\mc{M}_\mu \in \mc{H}$ is specified by the mean $\mu$; the variance parameter $\bar{R}^2$ is the same across all instances within $\mc{H}$. For an instance $\mc{M}_{\mu} \in \mc{H}$, it can be easily verified that the optimal state-action value function is $Q^*_{\mu}(s,a) = \mu/(1-\gamma)$. 

The rationale behind the above construction is to relate the estimation of $Q^*_{\mu}(s,a)$ to that of the scalar parameter $\mu$. Next, we consider a slightly different attack model where each of the agents is adversarial with probability $\varepsilon$, independently of the other agents. Thus, the expected fraction of adversaries is $\varepsilon$. We allow each agent to transmit $T$ noisy reward samples to the server. Note that we are allowing the server to observe the raw observations from the agents. The idea here is to reveal that even with access to this raw information, the fundamental bottleneck identified in~\eqref{eqn:main_conv_bnd} still persists. Let $\widehat{Q}_T$ be the estimate of the optimal state-action value function maintained by the server, which is allowed to be \emph{any} measurable function of the observations. We then have the following result. 

\begin{theorem}\textbf{(Fundamental lower bound)} 
\label{thm:async_fedq_lower_bound}
Given any estimator $\widehat{Q}_T$, there exist two distinct instances $\mc{M}_{\mu_1}, \mc{M}_{\mu_2} \in \mc{H}$, and an universal constant $c>0$ such that
\begin{equation}
\max\left\{\mathbb P\left(
\lVert \widehat Q_T-Q^*_{\mu_1}(s,a)\rVert_{\infty}
\ge
\frac{c\varepsilon \bar{R}}{(1-\gamma)\sqrt{T}}
\right), \mathbb P\left(
\lVert \widehat Q_T-Q^*_{\mu_2}(s,a)\rVert_{\infty}
\ge
\frac{c\varepsilon \bar{R}}{(1-\gamma)\sqrt{T}}
\right) \right\} 
\ge
\frac{1}{4}.
\nonumber
\end{equation}
\end{theorem}

\textbf{Main takeaway.}
Theorem~\ref{thm:async_fedq_lower_bound} tells us that with a constant probability of at least $1/4$, any estimator is bound to suffer an error on the order of $\Omega(\varepsilon \bar{R}/( (1-\gamma) \sqrt{T}))$. Collectively, Theorems~\ref{theorem:theoremmainresult} and~\ref{thm:async_fedq_lower_bound} provide the \emph{first set of nearly matching upper and lower bounds for the adversarially robust FRL problem}. One limitation is that our current lower bound does not capture the effect of asynchronous sampling. While we conjecture that the minimum visitation probability $\lambda_{\min}$ should also be fundamentally present in the corruption-induced bias term, making this rigorous is part of our ongoing work. The proof of Theorem~\ref{thm:async_fedq_lower_bound} in Appendix~\ref{app:proof_of_lower_bound} constructs two instances with optimal state-action value functions separated by $\Omega(\varepsilon \bar{R}/( (1-\gamma) \sqrt{T}))$. We then create adversarial signals that cause the distribution of reward observations to appear identical to the server in these two instances.  
The rest relies on a change-of-measure argument  employing the Bretagnolle-Huber inequality~\cite{lattimore2020bandit}. 

\section{Extension to the Partial-Coverage Setting}
\label{sec:partial}
We next extend our framework to a more realistic partial-coverage setting, denoted by \((\textcolor{winered}{\texttt{S}_2})\), where, unlike before, individual agents do not need to cover the entire state-action space \(\mc S\times\mc A\). To formalize this idea, we relax Assumption~\ref{ass:ergodic}, and no longer impose that for each agent $i$, the Markov chain induced by $\mu_i$ is aperiodic and irreducible; instead, we only require that the chain converges to a stationary distribution $\pi_i$ (which may have some zero entries). For each agent $i$, recall that the visitation probability to $(s,a)$ at stationarity is given by \(\lambda_i(s,a):=\pi_i(s)\mu_i(a|s)\), which is non-zero if and only if $\pi_i(s) >0$ (since $\mu_i(a|s) > 0, \forall s,a$). We then define \(\mc I(s,a):=\{i\in[N]:\lambda_i(s,a)>0\}\) as the set of ``source agents" who can provide information about state-action pair $(s,a).$ Intuitively, to estimate every component of $Q^*$, for each pair $(s,a) \in \mc{S} \times \mc{A}$, we need at least one agent to be able to observe/cover it. This is captured by the following partial coverage assumption. 

\begin{assumption}
\label{ass:committee_coverage}
    For each $(s,a) \in \mc{S} \times \mc{A}$, the corresponding source set satisfies $\mc{I}(s,a) \neq \emptyset.$
\end{assumption}

The above assumption is significantly weaker than the full coverage setting \((\textcolor{winered}{\texttt{S}_1})\) captured by Assumption~\ref{ass:ergodic} and studied in prior FRL work~\cite{khodadadian, wang2023TMLR, zhang2024finite}, which effectively requires \(\mc I(s,a)=[N]\) for all \((s,a)\). Instead of requiring the state-action space to be fully observable at \emph{every} agent, Assumption~\ref{ass:committee_coverage} requires the state-action space to be jointly observable based on the \emph{collective} measurements of \emph{all} agents. However, in the presence of adversarial agents, such an assumption is inadequate. To see why, suppose $\mc{I}(s,a) \subseteq \mc{B}$, i.e., the entire source set for pair $(s,a)$ is adversarial (recall that $\mc{B}$ denotes the adversarial set). In such a case, there is no hope for accurately estimating $Q^*(s,a)$. This simple example suggests that some form of \emph{information redundancy} is \textbf{necessary} for \emph{every} pair $(s,a)$. Accordingly, we impose the following local constraint: \(|\mc B\cap \mc I(s,a)|\le \varepsilon |\mc{I}(s,a)|, \forall (s,a) \in \mc{S} \times \mc{A}\), with $\varepsilon \in [0, 1/2).$ This ensures that there is sufficient clean information for every state-action pair. 

\textbf{Algorithmic Modifications.} The only algorithmic changes pertain to how the local estimates are computed and aggregated. In Line~6 of Algorithm~\ref{alg:async-fedq}, agent \(i\) now forms \(d_{i,k}(s,a)\) only for the state-action pairs that it can observe, i.e., for those \((s,a)\) such that \(i\in\mc I(s,a)\). Similarly, in the server aggregation step in Line~8, the server does not aggregate over all agents in \([N]\). Instead, for each \((s,a)\), it applies the trimmed mean estimator \texttt{trimSC} only to the messages \(\{d_{i,k}(s,a):i\in\mc I(s,a)\}\) collected from the corresponding source set $\mc{I}(s,a)$. Everything else remains the same. 

To state our main result for the partial coverage setting, we define the minimum source set size and the minimum overall visitation probability (over source sets) as \(\overline{N}:=\min_{(s,a)\in\mc S\times\mc A}|\mc I(s,a)|\) and \(\Lambda_{\min}:=\min_{(s,a)\in\mc S\times\mc A}\min_{i\in\mc I(s,a)}\lambda_i(s,a)\), respectively. Each of these objects creates \emph{information bottlenecks}, and plays key roles in our main result for this section stated below. 

\begin{theorem}\textbf{(Finite-time rates under} \(\textcolor{winered}{\texttt{S}_2}\)\textbf{)}
\label{theorem:theoremmainresult_S2} 
Suppose Assumption~\ref{ass:committee_coverage} holds. 
Fix any confidence parameter \(\delta\in(0,1)\), and suppose \(\eta=\log(NT)/((1-\gamma)K)\), \(K=\left\lceil c_1\log(NT)/(1-\gamma)\right\rceil\) where \(c_1>0\) is a universal constant that ensures \(\eta<1\), \(\overline{N} \geq \left\lceil
    2\log\!\left(16T|\mc S||\mc A|/\delta\right)
    \right\rceil,\) and \(T\ge (100K/\Lambda_{\min})\log\left(4|\mc S||\mc A|NT/\delta\right)\). Then, the output \(Q_K\) of Algorithm~\ref{alg:async-fedq} satisfies, with probability at least \(1-\delta\),
\begin{equation} 
\begin{aligned}
e_K
\le
\frac{e_0}{NT}
&+
\mathcal{O}\left(
\frac{
\bar R
\sqrt{\log(NT)\log\left(|\mc S||\mc A|T/\delta\right)}
}{
(1-\gamma)^{5/2}\sqrt{\overline{N}\Lambda_{\min}T}
}
\right) +
\mathcal{O}\left(
\frac{
\bar R\varepsilon\sqrt{\log(1/\varepsilon)}\sqrt{\log(NT)}
}{
(1-\gamma)^{5/2}\sqrt{\Lambda_{\min}T}
}
\right).
\end{aligned}
\label{eqn:main_conv_bnd_S2}
\end{equation}
\end{theorem}
\textbf{Discussion.} Theorem~\ref{theorem:theoremmainresult_S2} shows that partial coverage changes the effective sample size but not the basic robust-learning mechanism. Since we are interested in the $\ell_{\infty}$ error, the state-action pair $(s,a)$ with the smallest source set and the least stationary visitation probability creates a bottleneck. The bound in~\eqref{eqn:main_conv_bnd_S2} formalizes this intuition. The complete proof of Theorem~\ref{theorem:theoremmainresult_S2} is deferred to Appendix~\ref{app:proof_of_upper_bound_S2}. 

Theorems~\ref{theorem:theoremmainresult} and~\ref{theorem:theoremmainresult_S2}, together with the fundamental lower bound in Theorem~\ref{thm:async_fedq_lower_bound}, provide a comprehensive finite-time characterization of robust asynchronous federated \(Q\)-learning under corrupted agent updates. The upper bounds establish collaborative statistical gains under both full and partial coverage, while the lower bound shows that the corruption-dependent terms are unavoidable. 
\section{Extension to Markovian Setting}
\label{sec:lMarkov}
We next consider the challenging case where each agent observes a single Markovian trajectory rather than i.i.d. samples. We present the result for the full-coverage setting \((\textcolor{winered}{\texttt{S}_1})\); the partial-coverage counterpart \((\textcolor{winered}{\texttt{S}_2})\) follows by the same argument, with \(\lambda_{\min}\) and \(N\) replaced by \(\Lambda_{\min}\) and \(\overline N\), respectively. Previously, each agent's samples were assumed to be drawn i.i.d. from the stationary distribution induced by its behavior policy. We now relax this assumption and let \(s_{i,t}\) denote the state of the Markov chain followed by agent \(i\) at time \(t\). Under Assumption~\ref{ass:ergodic}, the induced process \(Z_{i,t}:=(s_{i,t},a_{i,t},s_{i,t+1})\) is also ergodic~\cite{chenQ}. Let \(\Omega\) denote the state space of \(\{Z_{i,t}\}\), and let \(\rho_i\) be its stationary distribution. Following~\cite{dorfman}, define \(d_{i,\mathrm{mix}}(t):=\sup_{z\in\Omega}D_{\mathrm{TV}}\!\left(\mathbb P(Z_{i,t}\in\cdot\mid Z_{i,0}=z),\rho_i\right)\), where \(D_{\mathrm{TV}}\) denotes total variation distance, and define the mixing time of agent \(i\)'s chain as \(\bar\tau_i:=\inf\{t\ge 1:d_{i,\mathrm{mix}}(t)\le 1/4\}\). Next, let \(\bar\tau:=\max_{i\in[N]}\bar\tau_i\) represent the worst mixing time across agents. Using this object, we design a sub-sampling gap \(\tau:=\lceil \ell\bar\tau\rceil\), where \(\ell:=\left\lceil \log(2NT/\delta)/\log 2\right\rceil\) and $\delta \in (0,1)$ is the prescribed failure probability. Our Markovian extension then analyzes a \emph{sub-sampled} variant of \texttt{Robust Async-Fed-Q}, where each honest agent uses only every \(\tau\)-th transition from its trajectory. The spacing \(\tau\) is chosen proportional to the worst-case mixing time \(\bar\tau\), so that the samples used by each honest agent are approximately independent, reducing the Markovian analysis to the asynchronous i.i.d. case with effective epoch length \(H/\tau\). The coupling argument which makes this intuition precise is provided in Appendix~\ref{app:Markovian}. We have the following finite-time guarantee. 


\begin{theorem}\textbf{(Markovian Extension)}
\label{theorem:markov}
Fix any confidence parameter \(\delta\in(0,1)\), and suppose Assumption~\ref{ass:ergodic} holds. For each agent \(i\in[N]\), let \(Z_{i,t}:=(s_{i,t},a_{i,t},s_{i,t+1})\) denote the Markov chain induced by its behavior policy, and suppose this chain is initialized from its stationary distribution \(\rho_i\), i.e., \(Z_{i,0}\sim\rho_i\). Then, for the choice \(\eta=\log(NT)/((1-\gamma)K)\), \(K=\left\lceil c_1\log(NT)/(1-\gamma)\right\rceil\), \(N \geq \left\lceil
    2\log\!\left(32T|\mc S||\mc A|/\delta\right)
    \right\rceil,\) and \(T \ge (100 K\tau/\lambda_{\min})\log\left(4|\mc S||\mc A|NT/\delta\right)\), where \(c_1>0\) is a universal constant ensuring \(\eta<1\), the output \(Q_K\) of the \(\tau\)-sub-sampled version of Algorithm~\ref{alg:async-fedq} satisfies, with probability at least \(1-\delta\),
\begin{equation} 
\begin{aligned}
e_K
\le
\frac{e_0}{NT}
&+
\mathcal{O}\left(
\frac{
\bar R
\sqrt{\tau\log(NT)\log\left(|\mathcal S||\mathcal A|T/\delta\right)}
}{
(1-\gamma)^{5/2}\sqrt{\lambda_{\min}NT}
}
\right) +
\mathcal{O}\left(
\frac{
\bar R\varepsilon\sqrt{\tau\log(1/\varepsilon)}\sqrt{\log(NT)}
}{
(1-\gamma)^{5/2}\sqrt{\lambda_{\min}T}
}
\right).
\end{aligned}
\label{eqn:main_conv_bnd_markov}
\end{equation}
\end{theorem}
\textbf{Main Takeaway.} Comparing Theorems~\ref{theorem:theoremmainresult} and~\ref{theorem:markov}, the Markovian guarantee mimics the i.i.d. asynchronous guarantee with an effective sample size \(T/\tau\). Thus, temporal dependence incurs only the standard \(\sqrt{\tau}\) mixing-time slowdown~\cite{Qu} in the statistical and corruption-dependent terms. To our knowledge, this is the first robustness guarantee in distributed learning under correlated data.
\section{Simulation Results}
We evaluate \texttt{Robust Async-Fed-Q} on a grid-world with $|\mc S|=100$, $|\mc A|=50$, $\gamma=0.7$, and $\bar{R} \le 10$ under full and partial coverage. The proposed algorithms remain stable under adversarial message corruption and converge to a small neighborhood of $Q^\star$. The estimation error decreases with the epoch length and the number of contributing agents, corroborating our theoretical guarantees. We measure the \(\ell_{\infty}\) error after epoch \(k\) as \(E_k:=\lVert Q_k-Q^\star\rVert_\infty\), using Markovian trajectories generated under behavior policies satisfying Assumption~\ref{ass:ergodic} for full coverage \((\textcolor{winered}{\texttt{S}_1})\) (see ~Fig.~\ref{fig:sim}) and Assumption~\ref{ass:committee_coverage} for partial coverage \((\textcolor{winered}{\texttt{S}_2})\) (see Fig.~\ref{fig:sim-2}). Each plot in Figures~\ref{fig:sim} and~\ref{fig:sim-2} reports the average estimation error across 100 independent runs.
\begin{figure}[h]
\begin{center}
\begin{tabular}{cc}
   \hspace{-8 mm}\includegraphics[scale=0.30]{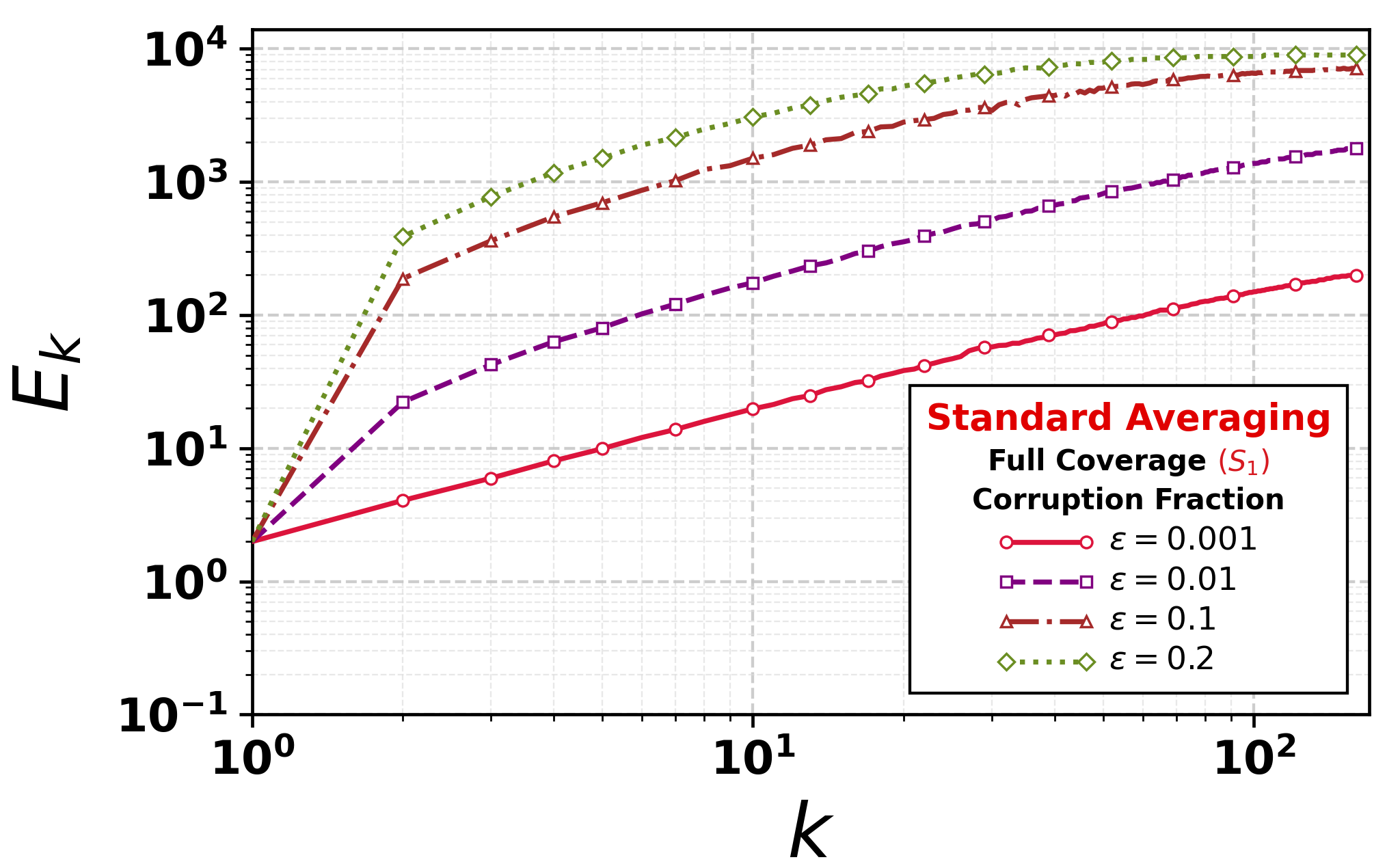}&\hspace{-4 mm}\includegraphics[scale=0.30]{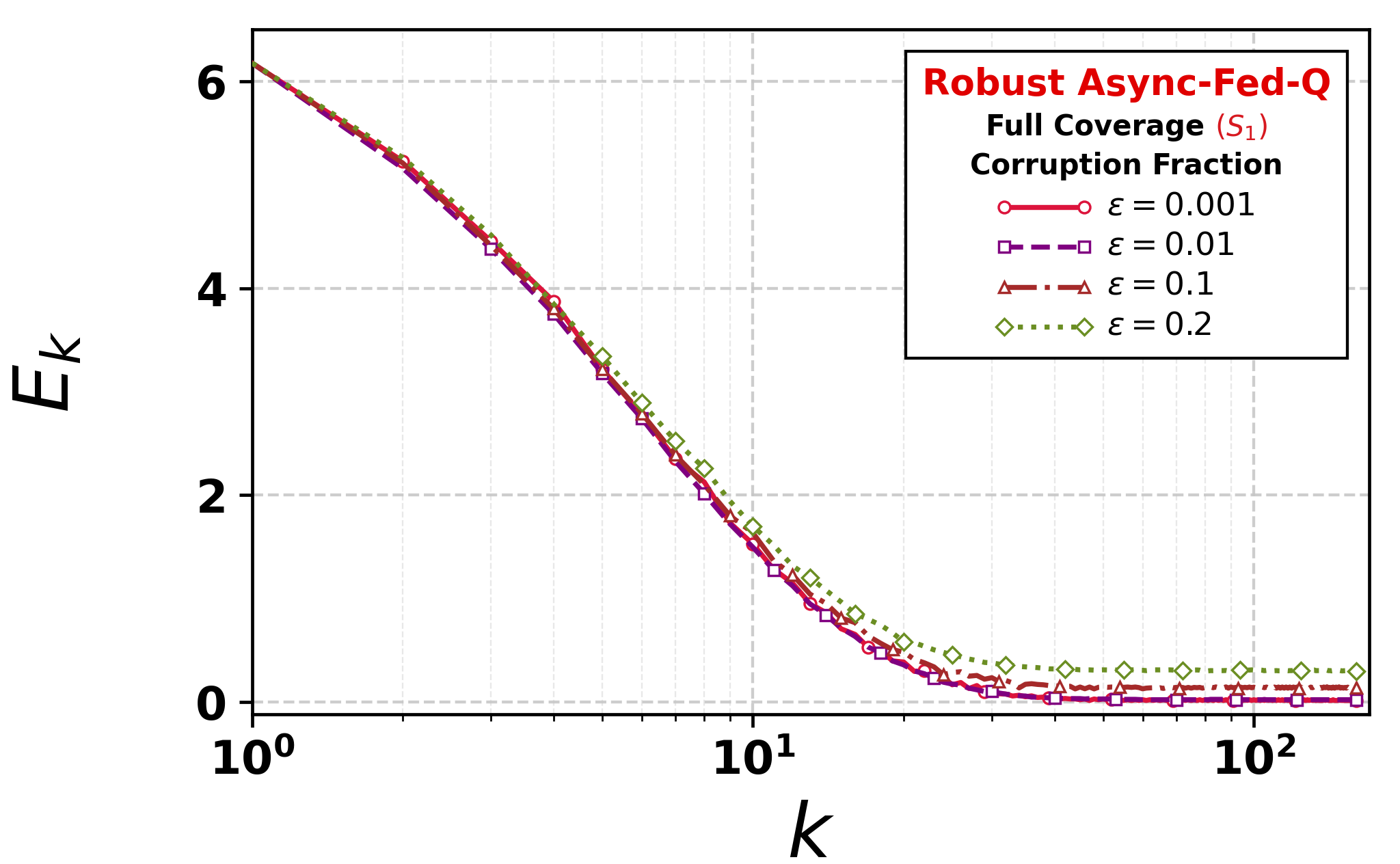}
    \end{tabular}
    \begin{tabular}{cc}
   \hspace{-8 mm}\includegraphics[scale=0.30]{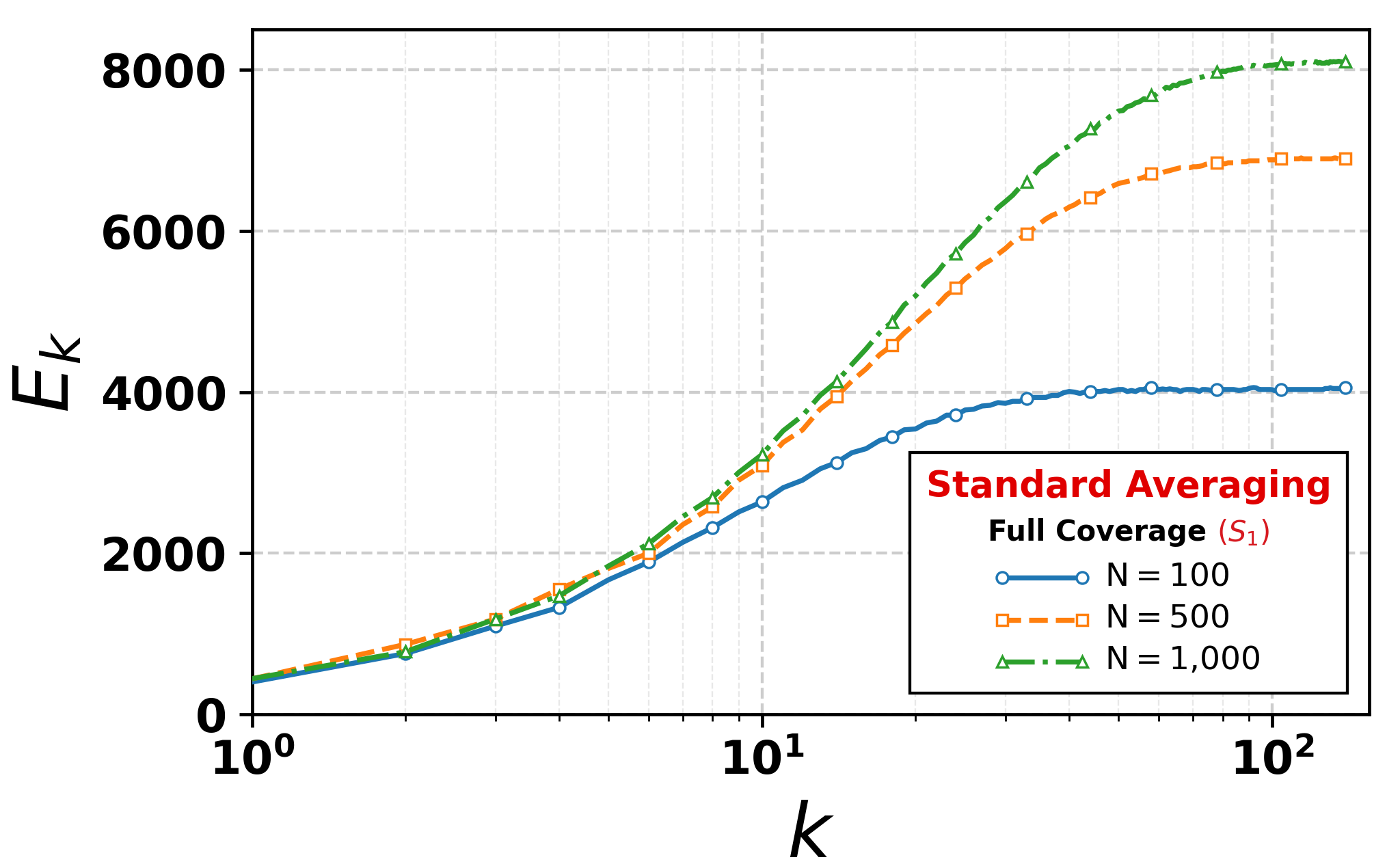}&\hspace{-4 mm}\includegraphics[scale=0.30]{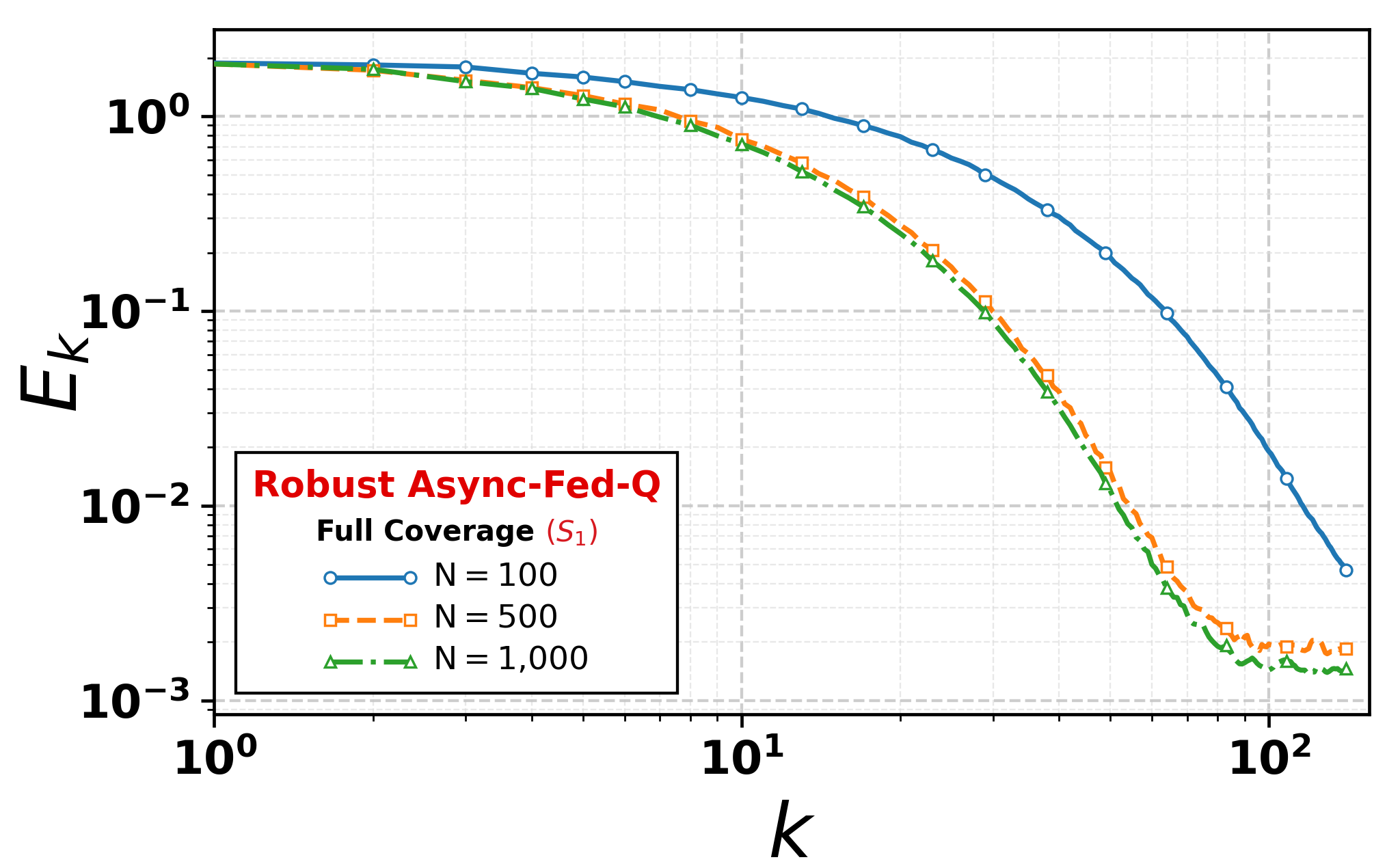}
    \end{tabular}
\vspace{-6mm}
\end{center}
\caption{\textbf{(Upper Left)}Asynchronous federated $Q$-learning with standard averaging under varying corruption fractions $\varepsilon\in\{0.001,0.01,0.1,0.2\}$ and a fixed number of agents $N=2000$, where each corrupted agent adds a bias of magnitude $10^4$ to its messages. \textbf{(Upper Right)} \texttt{Robust Async-Fed-Q} under the same values of \(\varepsilon\) and with \(N=2000\).
 \textbf{(Lower Left)} Federated $Q$-learning with standard averaging for
$N\in\{100,500,1000\}$, with a fixed number of corrupted agents
$|\mc B|=25$, each introducing a bias of magnitude $10^3$.
\textbf{(Lower Right)} \texttt{Robust Async-Fed-Q} under the same fixed-corruption setting \(|\mc{B}|=25\). The results demonstrate that the standard-averaging baseline is highly sensitive to corrupted messages, whereas our proposed algorithm remains stable, converges to a neighborhood of \(Q^\star\), and achieves lower estimation error as the number of agents increases, demonstrating clear collaborative gains.
}
\label{fig:sim}
\end{figure}
\begin{figure}[h]
\begin{center}
\begin{tabular}{cc}
   \hspace{-8 mm}\includegraphics[scale=0.30]{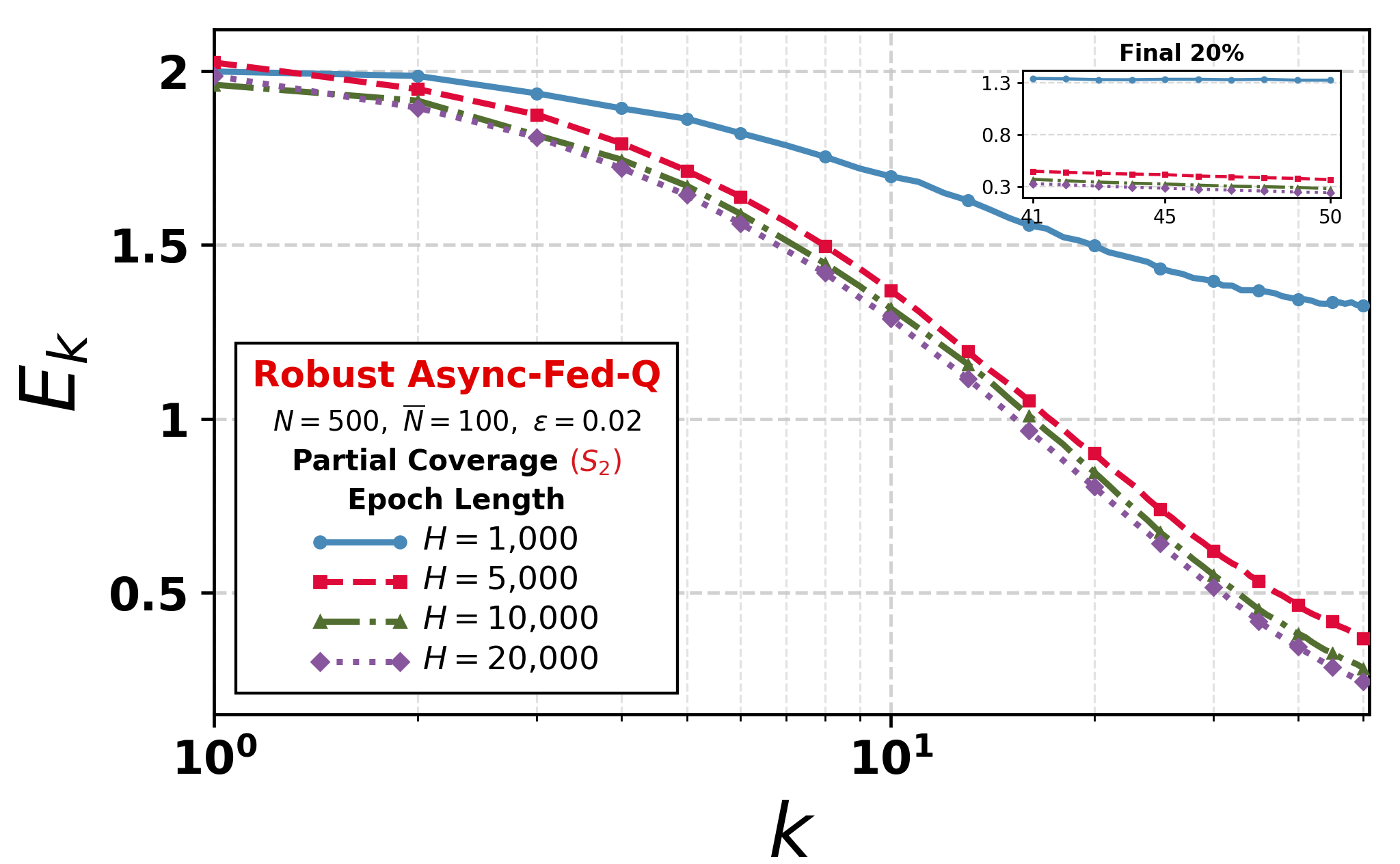}&\hspace{-4 mm}\includegraphics[scale=0.30]{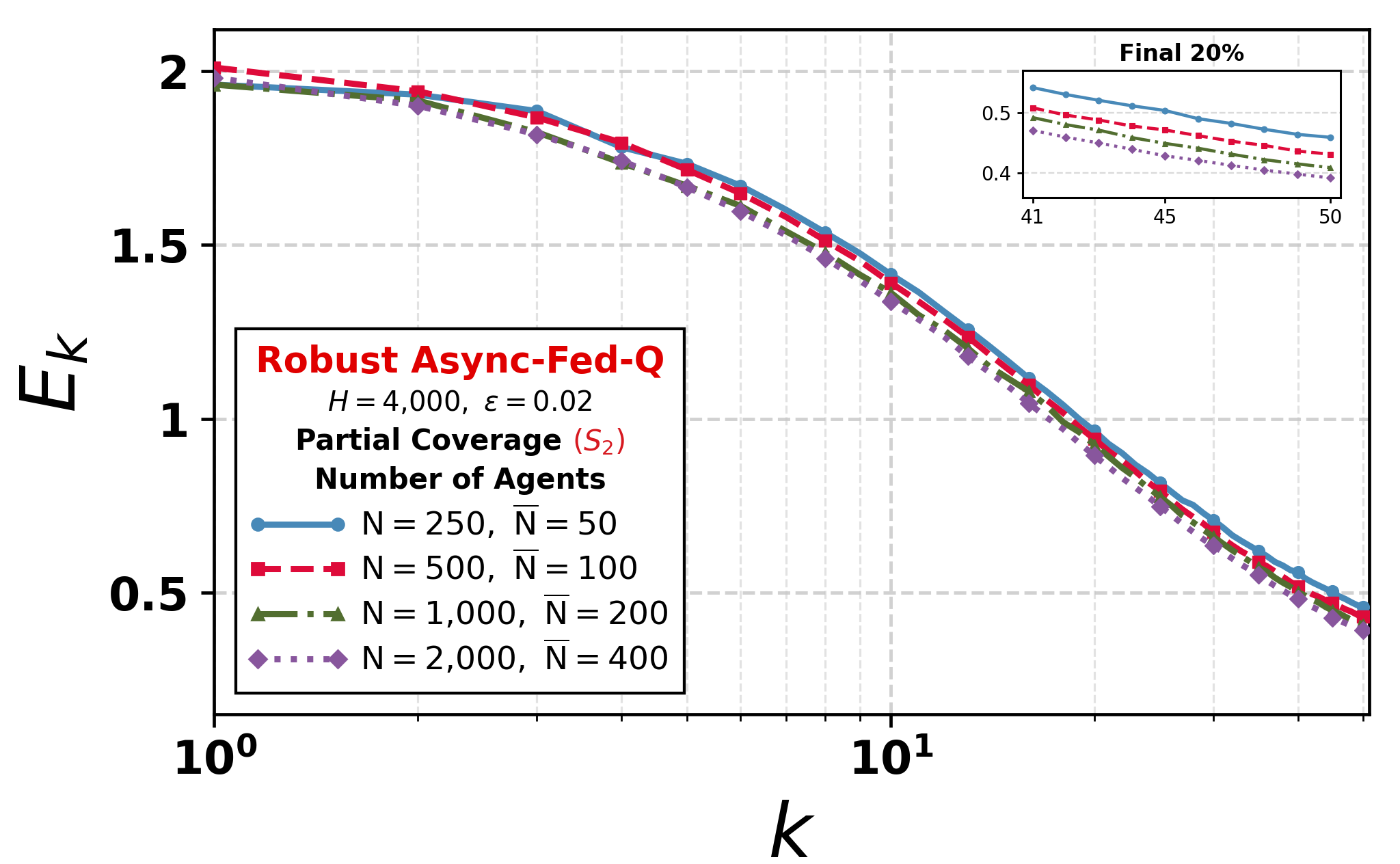}
    \end{tabular}
\vspace{-6mm}
\end{center}
\caption{Performance of \texttt{Robust Async-Fed-Q} under partial coverage $(\textcolor{winered}{\texttt{S}_2})$. 
\textbf{(Left)} The error $E_k$ for 
$H\in\{1000,5000,10000,20000\}$, with $N=500$ total agents and 
$\overline N=100$ agents covering each state-action pair. 
\textbf{(Right)} The error for $N\in\{250,500,1000,2000\}$, corresponding to 
$\overline N\in\{50,100,200,400\}$, with $H=4000$. 
In both panels, $\varepsilon=0.02$, corrupted messages are replaced by $-10^6$, 
and the reward-noise variance is $\sigma^2=5$.}
\label{fig:sim-2}
\end{figure}
\subsection{Ablation Studies and Robustness Diagnostics}
\label{subsec:ablation_studies}

We next isolate the roles of robust aggregation, collaboration, sample size,
coverage, communication, and temporal dependence.  These experiments are performed on the synthetic tabular instance used for the main algorithmic evaluation so that
each parameter can be varied independently while holding the underlying MDP
fixed.
\paragraph{Attack models.}
We fix the corruption fraction at \(\varepsilon=0.0025\), corresponding to five corrupted agents when \(N=2000\), and keep their identities fixed throughout each run. Under the \emph{negative-outlier attack}, every corrupted agent transmits the extreme value \(-10^{6}\) for each state--action coordinate. Under the \textbf{sign-flip attack}, the corrupted message is set to \(-10\) times the coordinatewise mean of the honest messages, thereby reversing and amplifying the honest update direction. Under the \textbf{dispersion-shift attack}, corrupted agents collude and transmit the coordinatewise honest mean plus one honest standard deviation, producing a less conspicuous attack whose magnitude adapts to the natural dispersion of the honest messages. Finally, under the \textbf{adaptive-sign attack}, the adversary evaluates both extreme messages \(-10^{6}\) and \(10^{6}\) for every coordinate and selects the sign that maximizes the instantaneous deviation of the aggregated update from the population Bellman target. These attacks range from easily identifiable extreme outliers to distribution-aware and co-ordinatewise adaptive corruptions.

\paragraph{Effect of robust aggregation and corruption.}
We first examine whether the stability observed in the benchmark environments
is genuinely produced by the robust aggregation mechanism. We compare
 \textcolor{winered}{\texttt{Robust Async-Fed-Q}} with an otherwise identical implementation
that replaces the robust estimator with the arithmetic mean. We additionally
vary the attack mechanism and reward-noise scale.

\begin{figure}[h]
\begin{center}
\begin{tabular}{cc}
   \hspace{-8mm}
   \includegraphics[scale=0.15]{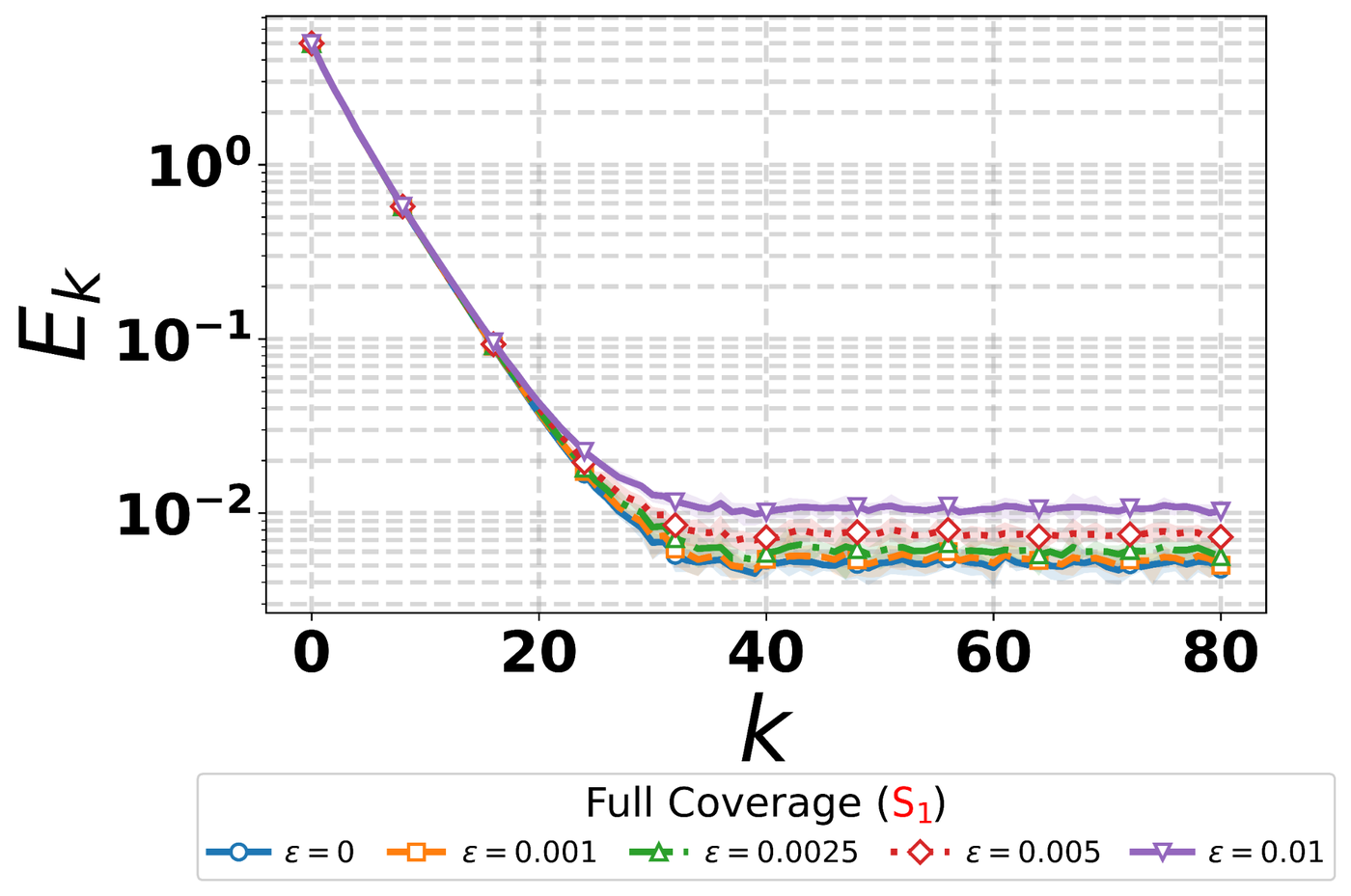}
   &
   \hspace{-4mm}
   \includegraphics[scale=0.15]{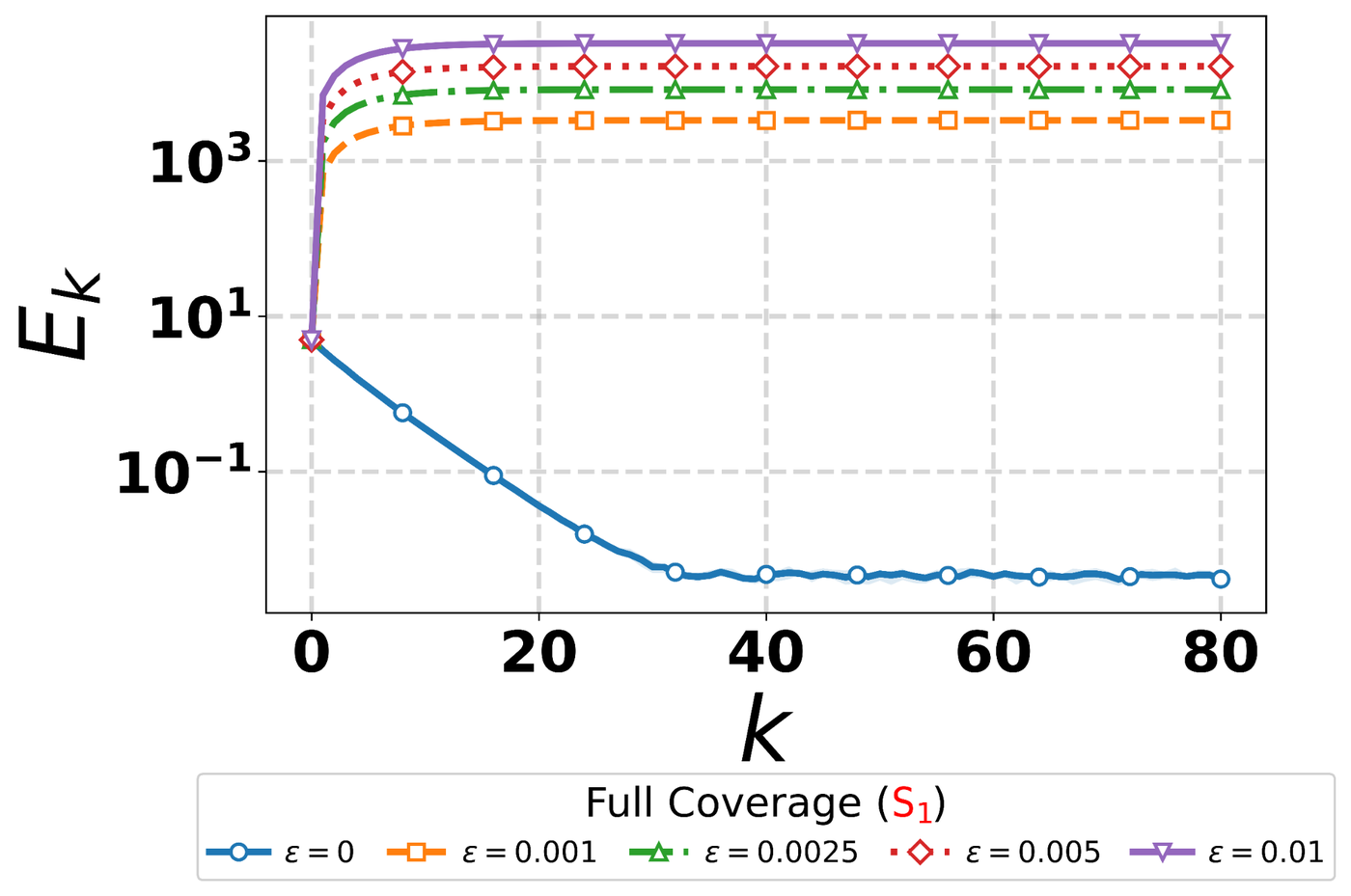}
\end{tabular}
\begin{tabular}{cc}
   \hspace{-8mm}
   \includegraphics[scale=0.30]{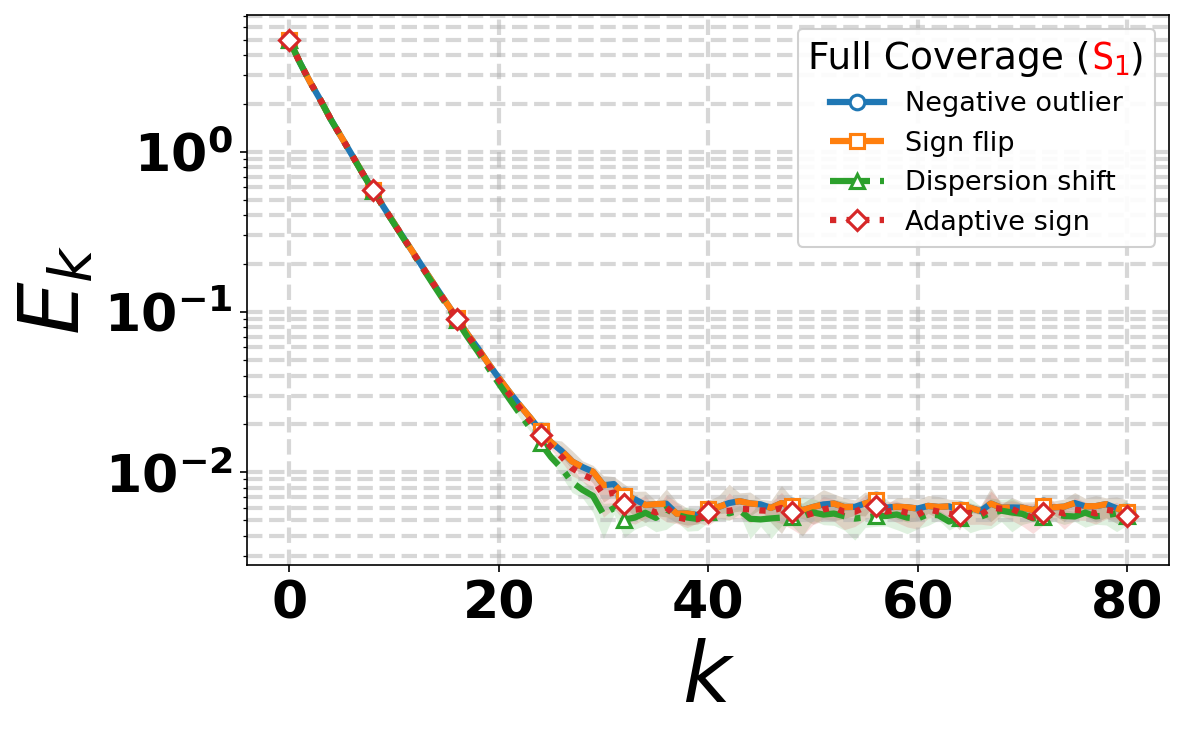}
   &
   \hspace{-4mm}
   \includegraphics[scale=0.30]{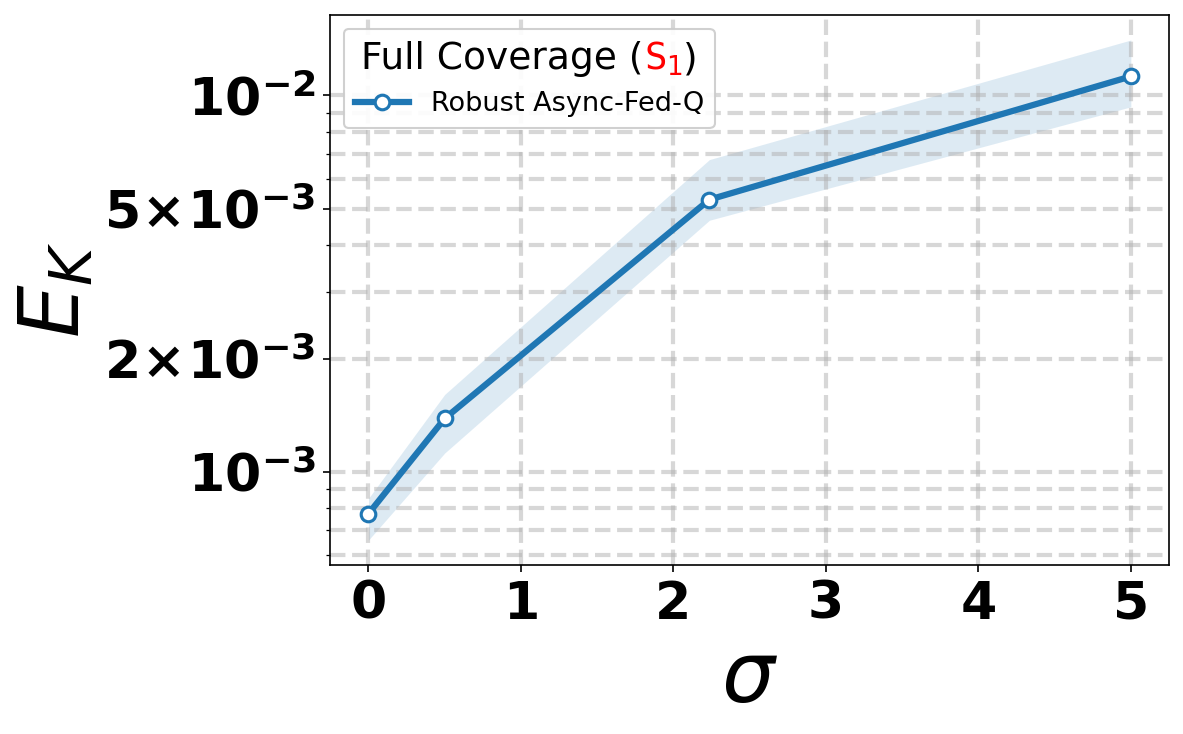}
\end{tabular}
\vspace{-6mm}
\end{center}
\caption{\textbf{(Upper Left)} Convergence of  \textcolor{winered}{\texttt{Robust Async-Fed-Q}} for different corruption fractions \(\varepsilon\in\{0,0.001,0.0025,0.005,0.01\}\). \textbf{(Upper Right)} Performance of standard arithmetic-mean aggregation under the same corruption levels. \textbf{(Lower Left)} Performance under negative-outlier, sign-flip, dispersion-shift, and adaptive-sign attacks. \textbf{(Lower Right)} Terminal error as a function of the reward-noise standard deviation \(\sigma\).}
\label{fig:ablation_robustness}
\end{figure}

The upper-left panel of Figure~\ref{fig:ablation_robustness} shows that all
curves initially follow essentially the same contraction phase. The influence
of corruption becomes visible only after the optimization transient reaches
the statistical-error regime. The clean curve stabilizes near the lowest error
level, while larger values of \(\varepsilon\) produce progressively higher
error floors. Nevertheless, even at \(\varepsilon=0.01\), the iterates remain
stable and the terminal error remains on the order of \(10^{-2}\). Thus,
corruption primarily changes the limiting neighborhood rather than destroying
convergence.

The upper-right panel demonstrates that this stability is not obtained from
collaboration alone. With arithmetic-mean aggregation, the clean curve still
converges, but every corrupted curve rapidly grows by several orders of
magnitude. Even a corruption fraction of \(0.001\) is sufficient to make the
mean-aggregation estimate unusable. Because the remaining learning,
communication, and sampling parameters are identical, the contrast between
the two upper panels identifies robust aggregation as the mechanism preventing
a small number of extreme messages from dominating the update.

The lower-left panel tests four qualitatively different attacks. The negative
outlier applies a fixed extreme bias, the sign-flip attack reverses the
direction of the message, the dispersion-shift attack increases the spread of
the corrupted messages, and the adaptive-sign attack selects its direction
based on the current Bellman error. The four trajectories are nearly
indistinguishable after the initial transient. Hence, the observed robustness
is not specific to one fixed attack direction. The adaptive attack also fails
to produce a materially larger error than the nonadaptive alternatives.

The lower-right panel varies the reward-noise standard deviation \(\sigma\).
The terminal error increases monotonically with \(\sigma\), from below
\(10^{-3}\) in the noiseless case to approximately \(10^{-2}\) at
\(\sigma=5\). This behavior is consistent with the statistical-error term
increasing with the scale of the Bellman-target noise. Importantly, the growth
is gradual, and no instability is observed over the tested range.

\paragraph{Collaboration and agent-population scaling.}
We next study how the number of participating agents changes the convergence
trajectory and terminal error. We consider both the clean setting and a fixed
positive corruption fraction, and then separate the fixed-fraction and
fixed-number corruption models.

\begin{figure}[h]
\begin{center}
\begin{tabular}{cc}
   \hspace{-12mm}
   \includegraphics[scale=0.12]{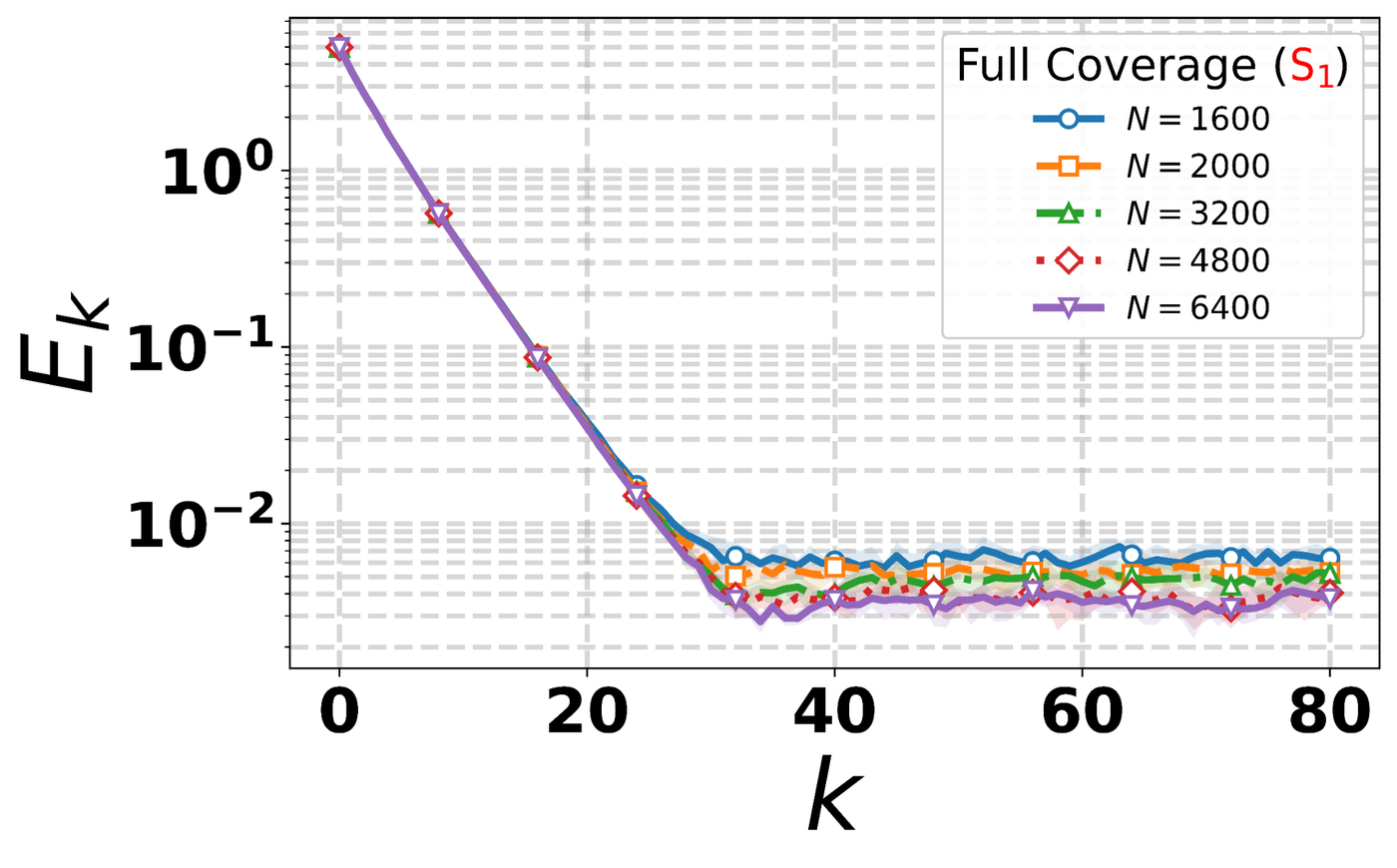}
   \hspace{-3mm}
   \includegraphics[scale=0.30]{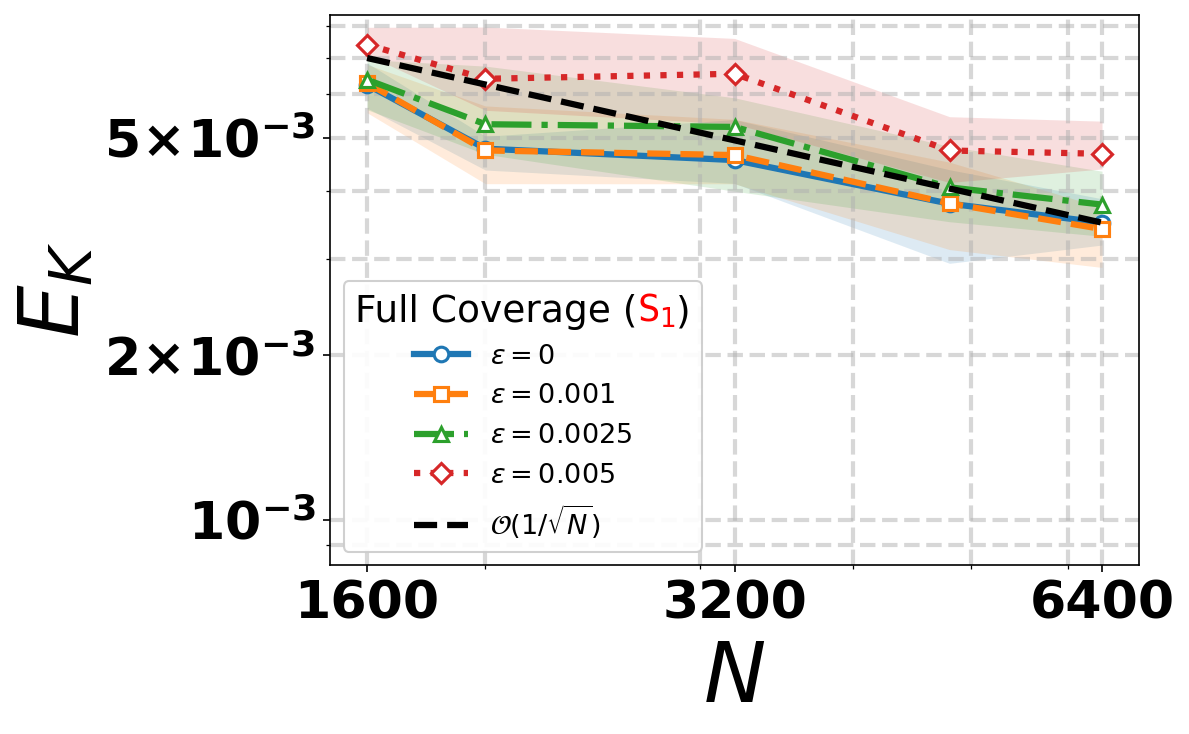}
   \hspace{-3mm}
   \includegraphics[scale=0.12]{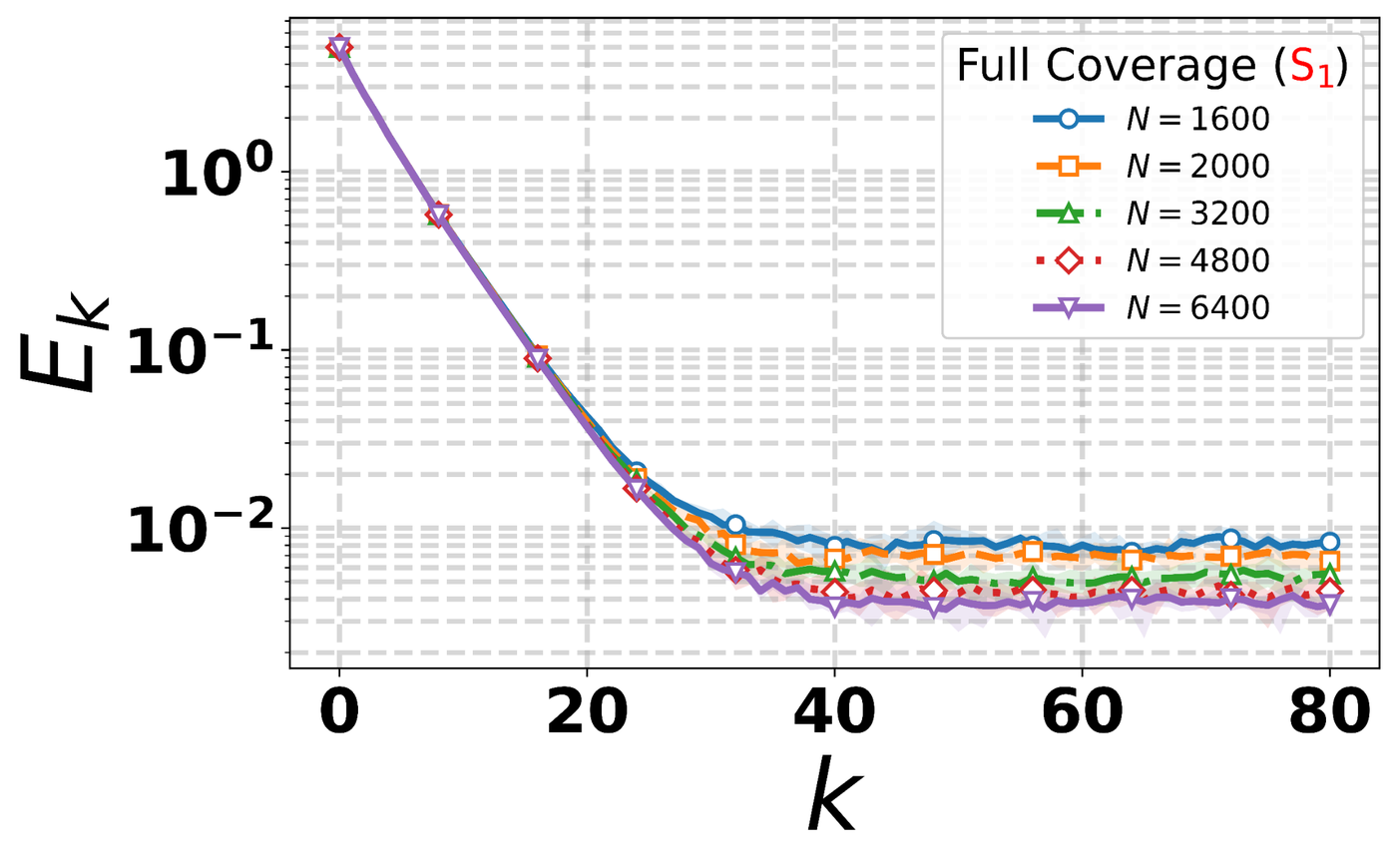}
\end{tabular}
\vspace{-6mm}
\end{center}
\caption{Collaboration ablations under full coverage
\textup{(\textcolor{winered}{\(\texttt{S}_1\)})}.
\textbf{(Left)} Evolution of the worst-agent error for different network sizes
at a fixed positive corruption fraction, showing that larger networks converge
to lower error levels despite corrupted messages.
\textbf{(Middle)} Terminal error as a function of \(N\) for several corruption
fractions, together with the \(\mathcal{O}(1/\sqrt{N})\) reference slope. The
error decreases with \(N\) throughout the corruption sweep, demonstrating that
the collaboration benefit persists under corruption.
\textbf{(Right)} Evolution of the error as \(N\) varies while the absolute
number of corrupted agents is held fixed. In this setting, increasing \(N\)
adds honest agents and reduces the realized corruption fraction, leading to a
progressively lower error floor.}
\label{fig:ablation_collaboration}
\end{figure}

The left panel of Figure~\ref{fig:ablation_collaboration} show that
changing \(N\) has little influence on the initial contraction phase. This is
expected because the deterministic component of the Bellman recursion is
controlled primarily by \(\gamma\) and the update stepsize. The effect of
collaboration appears after the curves enter their statistical neighborhoods.
In the clean setting, the error floor decreases systematically as \(N\)
increases from \(1600\) to \(6400\). The same ordering persists under
corruption, demonstrating that robust collaboration continues to reduce the
worst-agent error even when a positive fraction of the messages is
adversarial.

The middle panel provides a direct terminal-error comparison. For each
fixed corruption level, \(E_K\) decreases as \(N\) increases. The empirical
curves are broadly aligned with the vertically rescaled
\(\mathcal{O}(1/\sqrt{N})\) guide, although the guide is not fitted to the
data. Increasing \(\varepsilon\) shifts the curve upward, while leaving its
decreasing dependence on \(N\) visible. This separation supports the
interpretation that the statistical collaboration term decays with \(N\),
whereas corruption contributes an additional error component. The right panel holds the absolute number of corrupted agents fixed.
Consequently, the realized corrupted fraction decreases as the network grows.
The terminal neighborhood becomes smaller with increasing \(N\), reflecting
two simultaneous benefits: more honest messages are available for estimating
each Bellman coordinate, and the relative influence of the fixed corrupted
set is diluted. Together, the fixed-fraction and fixed-number experiments show
that collaboration remains beneficial under both common corruption models.

\paragraph{Per-epoch sampling and total sample complexity.}
The next ablation varies the number \(H\) of samples collected by each agent
during an epoch and the resulting total sample budget \(T\).

\begin{figure}[h]
\begin{center}
\begin{tabular}{cc}
   \hspace{-8mm}
   \includegraphics[scale=0.30]{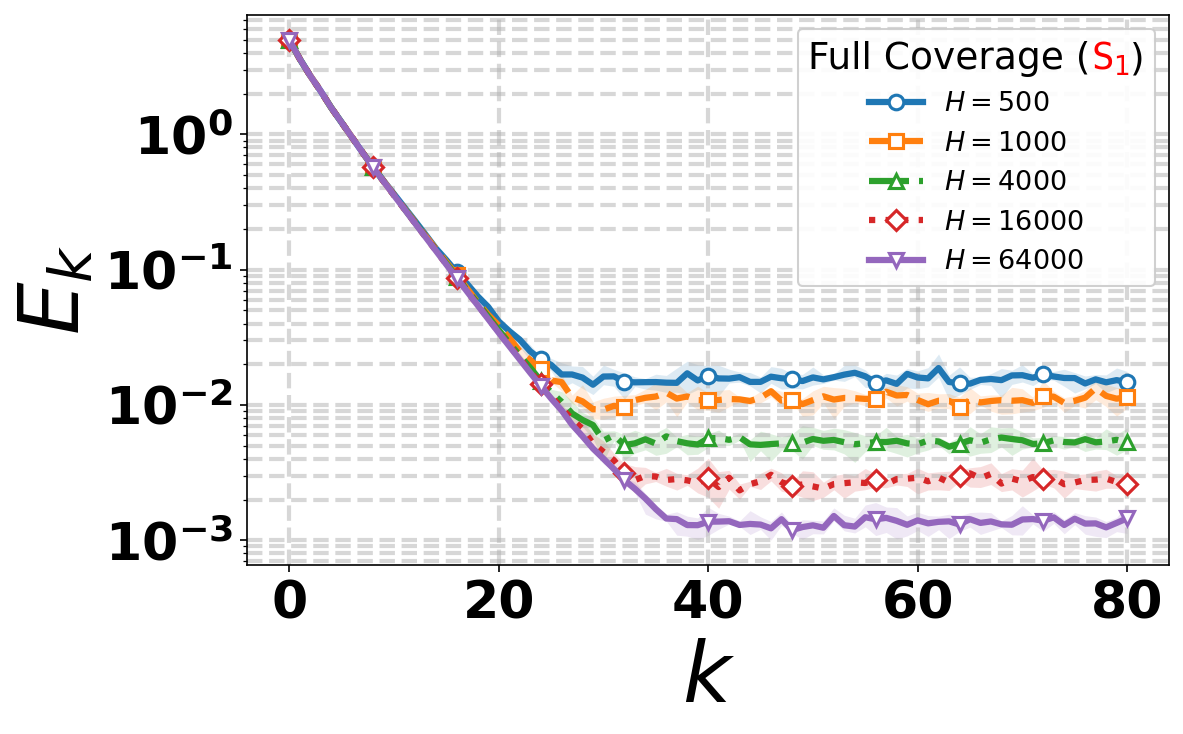}
   &
   \hspace{-4mm}
   \includegraphics[scale=0.30]{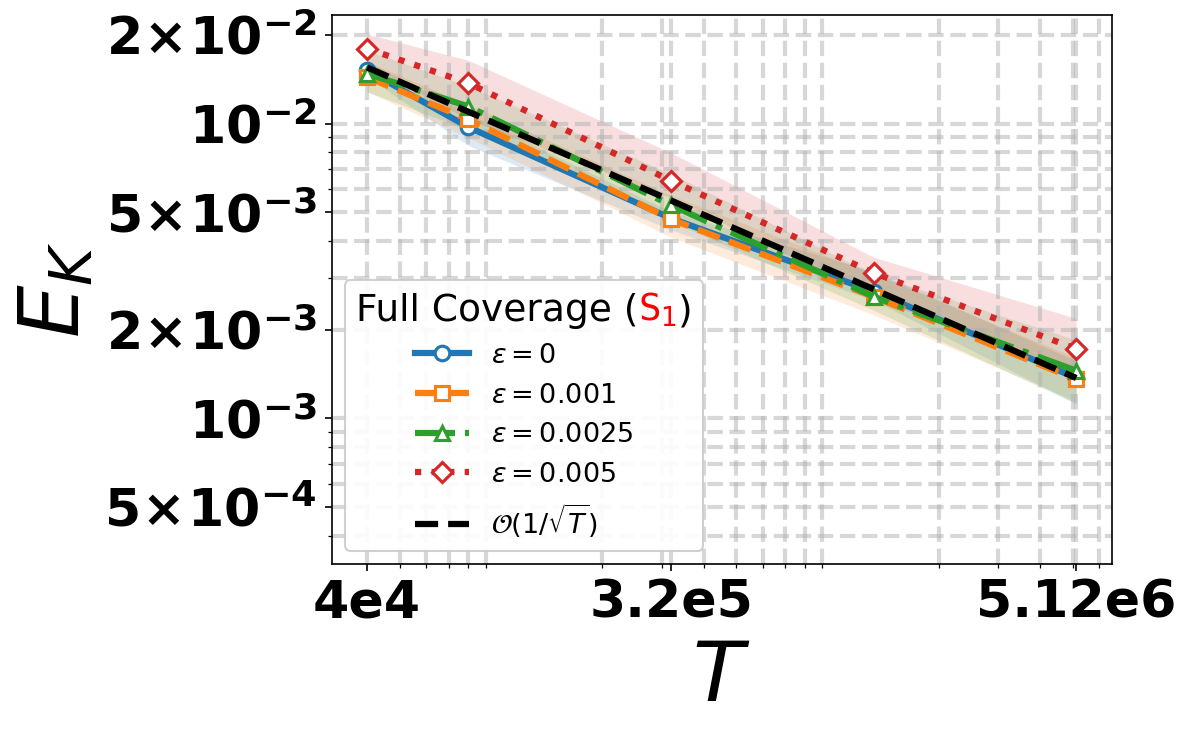}
\end{tabular}
\vspace{-6mm}
\end{center}
\caption{
Sample-size ablations under full coverage
\textup{(\textcolor{winered}{\(\texttt{S}_1\)})}.
\textbf{(Left)} Convergence under corruption for
\(H\in\{500,1000,4000,16000,64000\}\).
\textbf{(Right)} Terminal error versus the total sample budget \(T\) for
different corruption fractions, together with the
\(\mathcal{O}(1/\sqrt{T})\) reference curve.}
\label{fig:ablation_sample_complexity}
\end{figure}

In the left panel of Figure~\ref{fig:ablation_sample_complexity}, all choices
of \(H\) exhibit nearly identical initial contraction. Once the curves
approach the noise-dominated regime, however, their terminal neighborhoods
separate clearly. Increasing \(H\) from \(500\) to \(64000\) reduces the error
floor by approximately one order of magnitude. Larger batches produce more
accurate empirical Bellman operators and therefore allow the contraction to
continue to a smaller neighborhood. The right panel eliminates the distinction between the number of epochs and the number of samples per epoch by plotting \(E_K\) against the total sample budget \(T\). The curves decay approximately parallel to the \(\mathcal{O}(1/\sqrt{T})\) guide over more than two orders of magnitude in \(T\). The curves corresponding to larger corruption fractions are shifted upward, but their sample-dependent decay remains visible. Thus, additional
samples reduce the stochastic component of the error, while corruption
determines the residual separation between the curves.

\paragraph{Partial-coverage ablations.}
We now consider the partial-coverage setting
\textup{(\textcolor{winered}{\(\texttt{S}_2\)})}. Here, a state--action
coordinate may be sampled by only a subset of agents. The relevant
collaboration parameter is therefore the minimum number of source agents
\(\overline{N}\), rather than the total network size.

\begin{figure}[h]
\begin{center}
\begin{tabular}{cc}
   \hspace{-8mm}
   \includegraphics[scale=0.15]{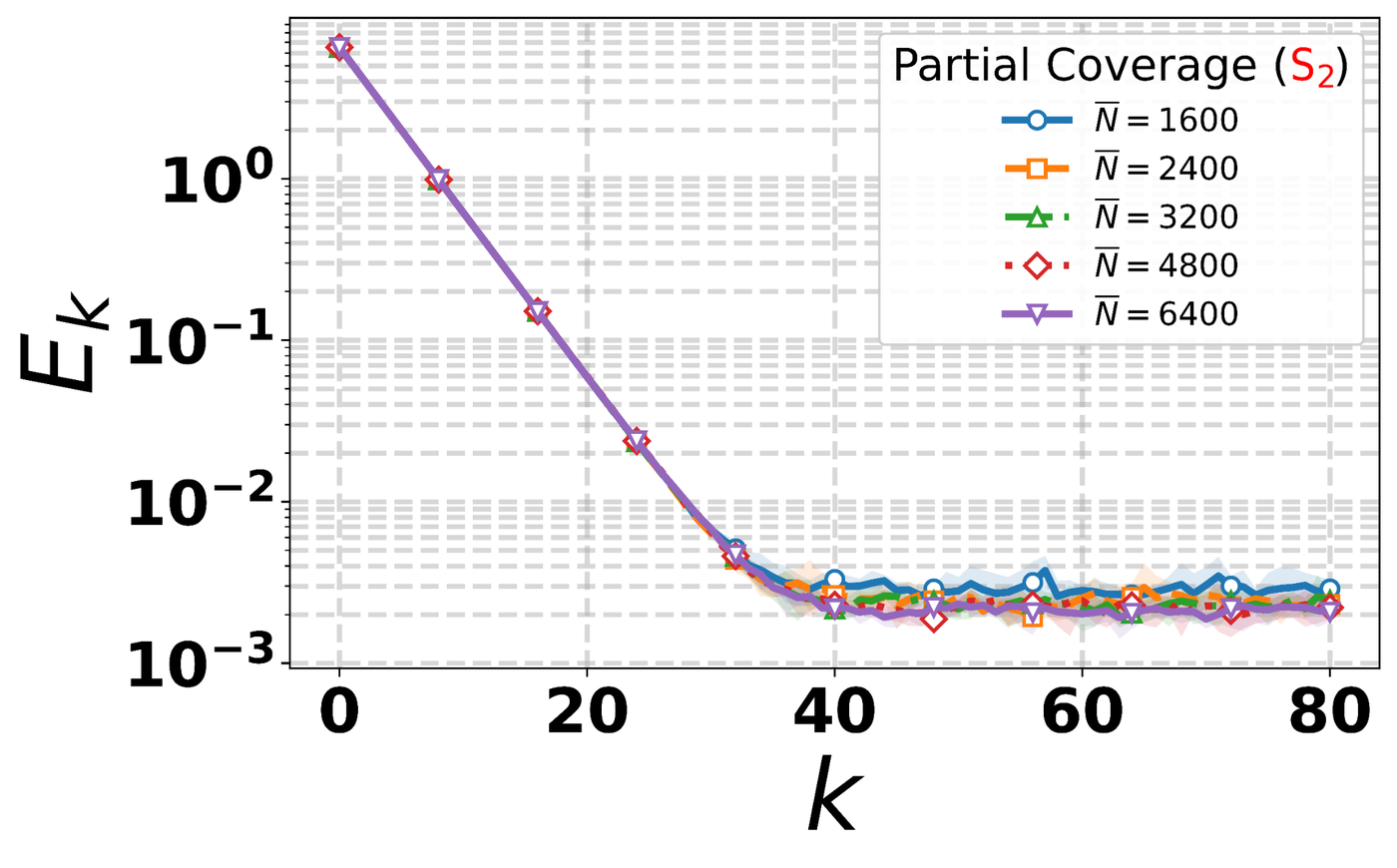}
   &
   \hspace{-4mm}
   \includegraphics[scale=0.15]{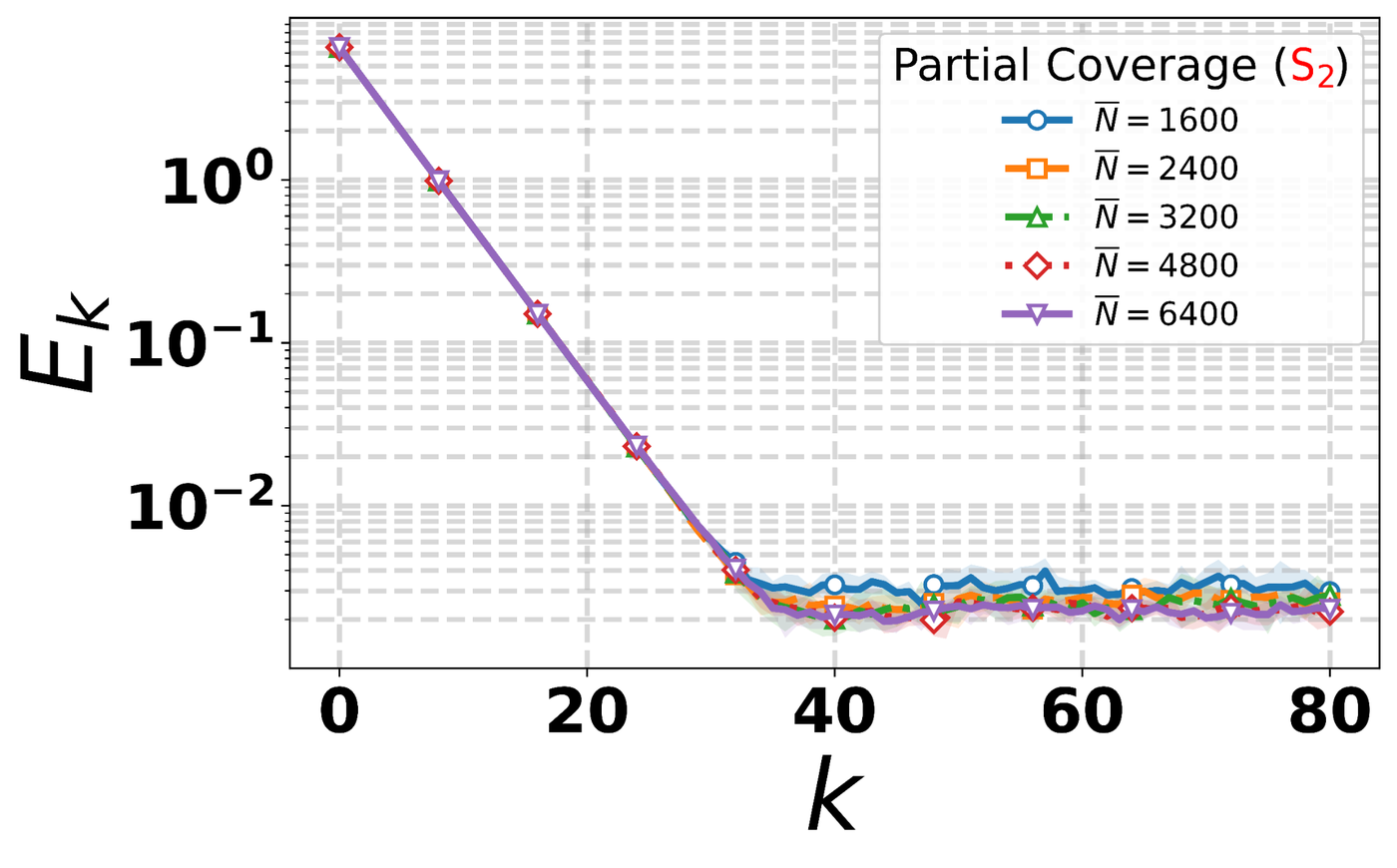}
\end{tabular}
\begin{tabular}{cc}
   \hspace{-8mm}
   \includegraphics[scale=0.30]{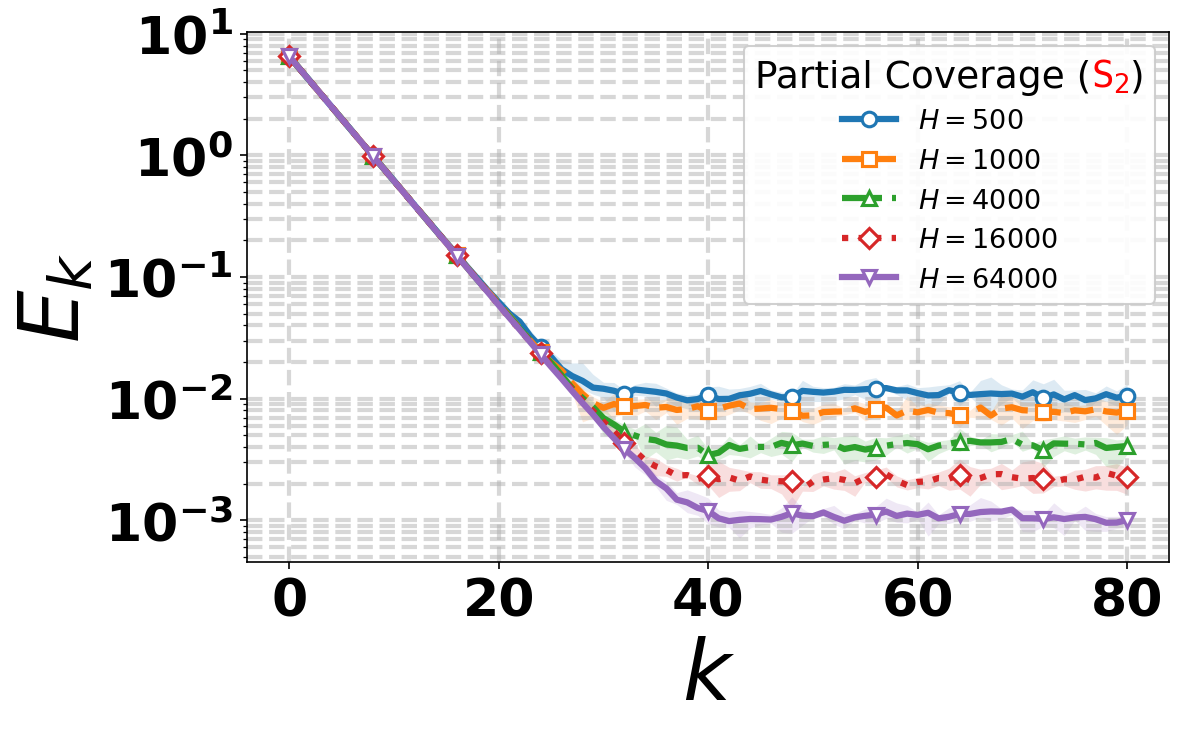}
   &
   \hspace{-2mm}
   \includegraphics[scale=0.30]{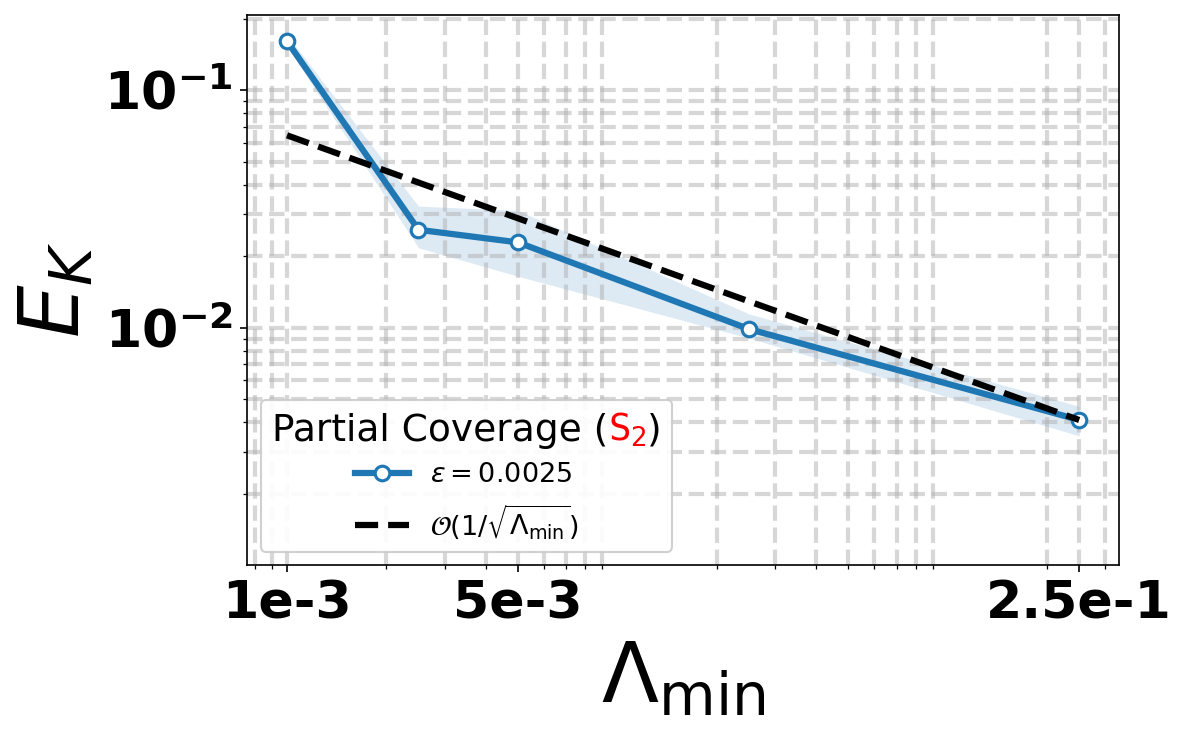}
\end{tabular}
\vspace{-6mm}
\end{center}
\caption{
Ablations under partial coverage
\textup{(\textcolor{winered}{\(\texttt{S}_2\)})}.
\textbf{(Upper Left)} Clean convergence for different minimum source-set sizes
\(\overline{N}\).
\textbf{(Upper Right)} Convergence under corruption for the same values of
\(\overline{N}\).
\textbf{(Lower Left)} Dependence on the per-epoch sample budget \(H\).
\textbf{(Lower Right)} Terminal error as a function of the minimum positive
source visitation probability \(\Lambda_{\min}\), together with the
\(\mathcal{O}(1/\sqrt{\Lambda_{\min}})\) reference curve.}
\label{fig:ablation_partial_coverage}
\end{figure}

The upper panels of Figure~\ref{fig:ablation_partial_coverage} demonstrate that
no individual agent needs to cover the complete state--action space for the
method to converge. In both the clean and corrupted settings, all curves
contract rapidly to stable neighborhoods. Increasing \(\overline{N}\) lowers
the terminal error because every coordinate is supported by a larger
committee of informative agents. The effect remains visible under corruption,
showing that the partial-coverage collaboration gain is not restricted to the
clean setting. The lower-left panel varies \(H\) under partial coverage. As in the
full-coverage experiment, larger values of \(H\) reduce the terminal
neighborhood without materially changing the initial contraction rate.
Therefore, the benefit of additional sampling survives the heterogeneous
source-set structure. The lower-right panel isolates the role of the minimum positive visitation probability. The terminal error decreases sharply as \(\Lambda_{\min}\)
increases, and its slope is broadly consistent with the displayed
\(\mathcal{O}(1/\sqrt{\Lambda_{\min}})\) guide. Small values of
\(\Lambda_{\min}\) create a statistical bottleneck because the least frequently
observed coordinate receives relatively few informative samples. This
experiment confirms that partial coverage is controlled not only by the number
of source agents but also by the sampling frequency of the least-covered
coordinate.

\paragraph{Effective source size, communication.} Finally, we isolate the effective source-set scaling, the allocation of a
fixed sample budget across communication epochs, and the effect of Markovian
dependence.

\begin{figure}[h]
\begin{center}
\begin{tabular}{cc}
   \hspace{-8mm}
   \includegraphics[scale=0.30]{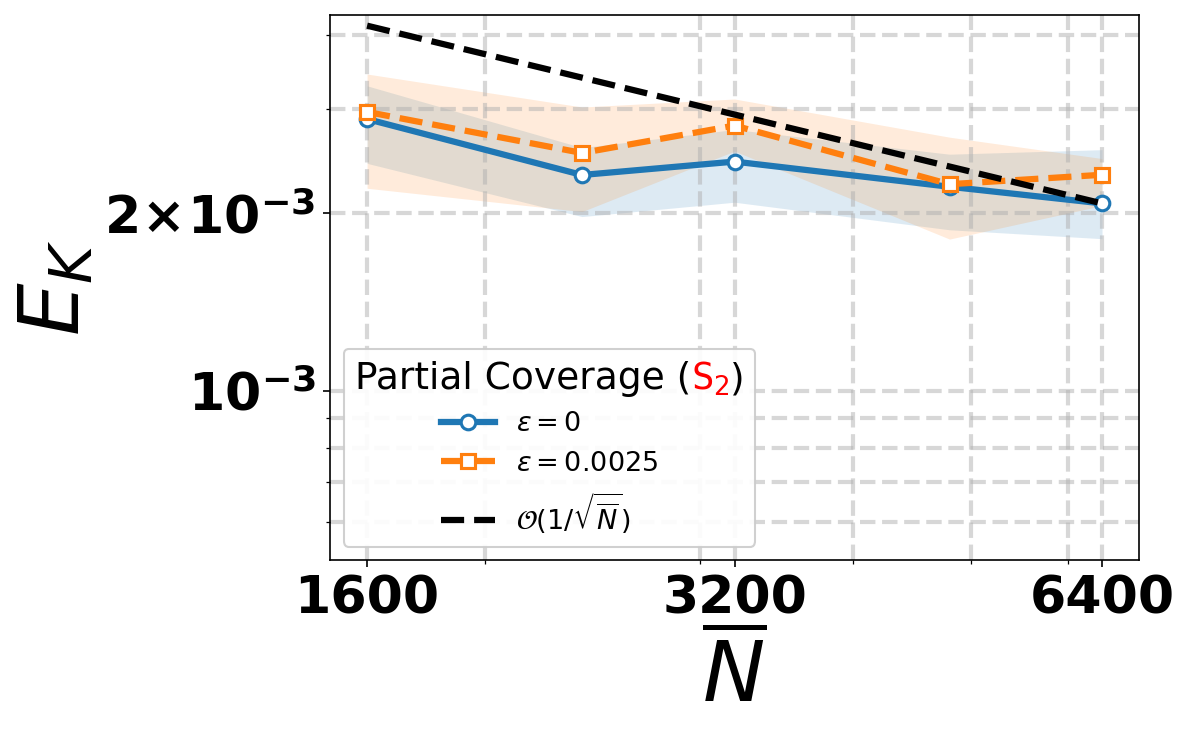}
   &
   \hspace{-4mm}
   \includegraphics[scale=0.30]{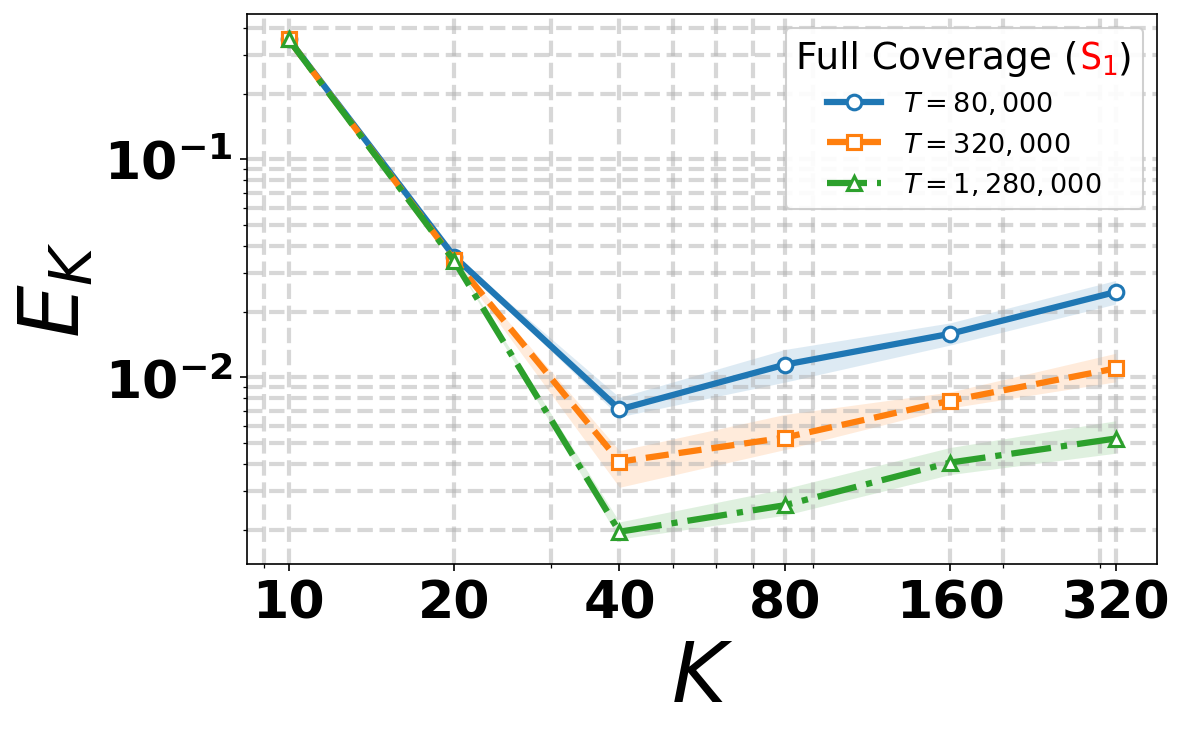}
\end{tabular}
\vspace{-6mm}
\end{center}
\caption{\textbf{(Left)} Terminal error versus the minimum source-set size
\(\overline{N}\) in the clean and corrupted partial-coverage settings,
together with the \(\mathcal{O}\left(\overline{N}^{-1/2})\right)\) reference curve.
\textbf{(Right)} Terminal error versus the number of epochs \(K\) for
three fixed total sample budgets \(T\).}
\label{fig:ablation_source_communication_markov}
\end{figure}

The left panel of
Figure~\ref{fig:ablation_source_communication_markov} provides a terminal-error
view of the partial-coverage collaboration gain. Both the clean and corrupted
curves decrease as \(\overline{N}\) increases. The trend is broadly comparable
to the \(\mathcal{O}(1/\sqrt{\overline{N}})\) guide, although local
fluctuations remain within the displayed variability bands. Corruption
slightly raises the error level without eliminating the dependence on the
effective number of source agents.

The right panel of
Figure~\ref{fig:ablation_source_communication_markov} studies communication accuracy while holding the total
sample budget \(T\) fixed. For each \(T\), increasing \(K\) initially improves
the error because the agents perform more Bellman and communication updates.
After approximately \(K=40\), the trend reverses. Since \(T\) is fixed, a
larger \(K\) leaves fewer samples per epoch, making each empirical Bellman
update noisier. The resulting U-shaped curves expose an allocation tradeoff:
too few epochs provide insufficient iterative refinement, while too many
epochs divide the available samples into batches that are too small. Increasing
\(T\) lowers the entire curve and weakens this tradeoff.

Taken together, the ablations identify four distinct contributions to the
terminal error. Robust aggregation prevents catastrophic sensitivity to
corrupted messages. Increasing \(N\), \(\overline{N}\), \(H\), \(T\), or
\(\Lambda_{\min}\) reduces the statistical component of the error. Corruption
and reward noise increase the terminal neighborhood without changing the
initial contraction mechanism. Finally, communication frequency and Markovian
dependence affect how efficiently the available samples are converted into
accurate Bellman updates. These observations support the individual
dependencies appearing in the finite-time rates for both full and
partial coverage.

\subsection{Robust Learning with Markovian Data}
\label{subsec:markovian_experiments}

We next evaluate  \textcolor{winered}{\texttt{Robust Async-Fed-Q}} when the samples available to each agent form a temporally dependent Markov trajectory. This experiment is conducted separately from the Gymnasium benchmarks because it is designed to isolate the effects of temporal dependence, thinning, and physical sample cost. In
particular, we distinguish the number of transitions generated by the
environment from the smaller number of transitions retained by the learning
algorithm after thinning.

\paragraph{Markov trajectory and mixing-time convention.}
For each agent, let \(Z_t=(s_t,a_t,s_{t+1})\) denote the
state--action--next-state triple generated by the behavior policy. Each
trajectory is initialized from stationarity and is never reset at epoch
boundaries. After retaining one transition, the sampler discards the next
\(\tau-1\) transitions before recording another. Thus, successive retained
observations are separated by exactly \(\tau\) physical transitions, while
dependence across consecutive epochs is preserved. Using the mixing-time
quantity \(\bar{\tau}\) defined in Section~\ref{sec:lMarkov}, we set the
theoretical thinning interval to
\begin{equation}
    \tau_{\mathrm{th}}
    =
    \bar{\tau}
    \left\lceil
        \log_2\!\left(
            \frac{2N T_{\mathrm{raw}}}{\delta}
        \right)
    \right\rceil.
    \label{eq:markov_tau_th}
\end{equation}
For the nominal configuration used below, the mixing time is
\(\bar{\tau}=4\), and the theoretical prescription yields
\(\tau_{\mathrm{th}}=140\). Let \(H\) denote the number of physical
transitions generated by each agent in one epoch. Over \(K\) epochs, the
physical per-agent budget is \(T_{\mathrm{raw}}=KH\), while, since thinning
continues across epoch boundaries, the exact number of transitions retained
by the learner is
\(T_{\mathrm{used}}(\tau)=\lceil KH/\tau\rceil\). Thus, setting
\(\tau=140\) uses approximately \(140\) physical transitions for every
retained transition. This distinction is essential: a comparison at fixed
\(T_{\mathrm{used}}\) isolates the effect of residual temporal dependence
while controlling for the amount of retained information, whereas a
comparison at fixed \(T_{\mathrm{raw}}\) additionally captures the
statistical cost of discarding observations through thinning.

\paragraph{Experimental configuration.}
We use the same full-coverage setting
\textup{(\textcolor{winered}{\(\texttt{S}_1\)})} and robust split-clipping
aggregator as in the preceding experiments. The controlled testbed is a
lazy-ring MDP with six states, two actions, and stay probability \(0.3\). We
set \(\gamma=0.7\), use a constant stepsize \(\eta=0.7\), and run the method
for \(K=30\) epochs. Unless it is the quantity being swept, we use
\(N=1600\), \(\varepsilon=0.0025\), reward-noise variance \(5\), and a
coordinatewise colluding shift attack at the scale of one honest standard
deviation. Each condition is evaluated over \(20\) independent repetitions, organized into five batches of four runs. The solid curve reports the mean across all repetitions, while the shaded region is the pointwise minimum and maximum of the five batch means. The shaded region is therefore a batch envelope and not a confidence interval. Every dashed black rate curve is an endpoint-anchored slope guide rather than a fitted model.

\begin{figure}[h]
\begin{center}
\begin{tabular}{cc}
\hspace{-12mm}
\includegraphics[scale=0.30]{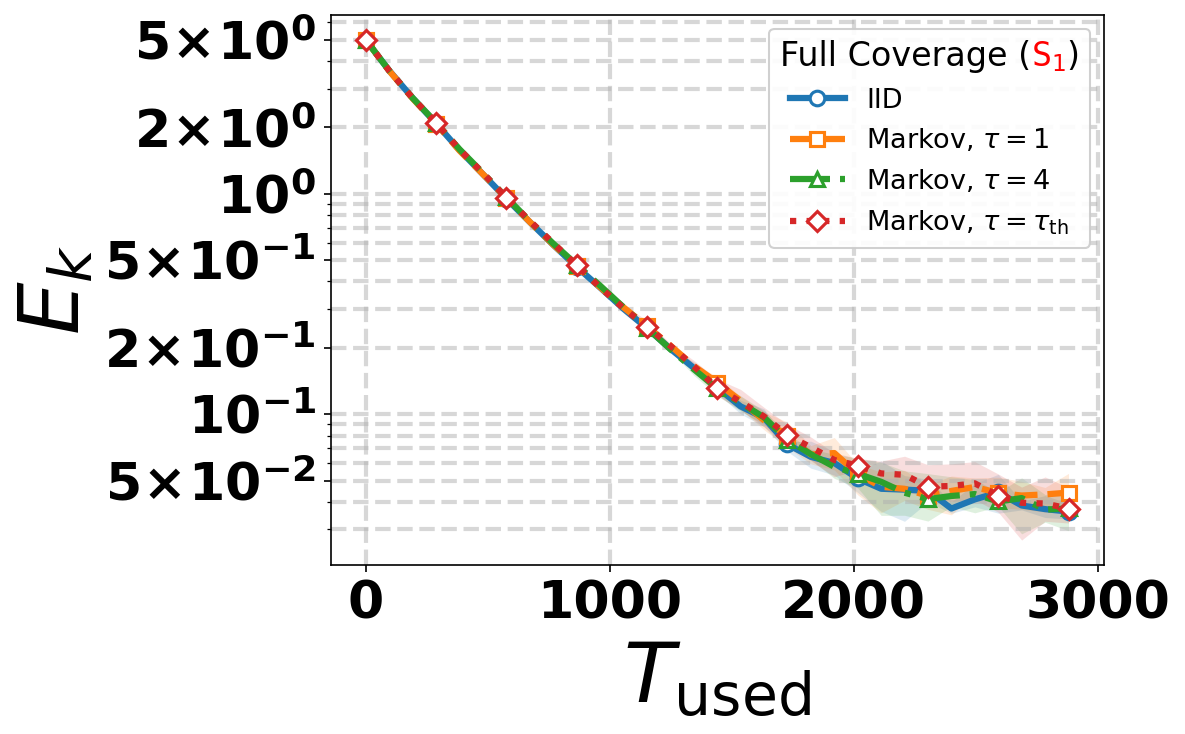}
\hspace{-3mm}
\includegraphics[scale=0.30]{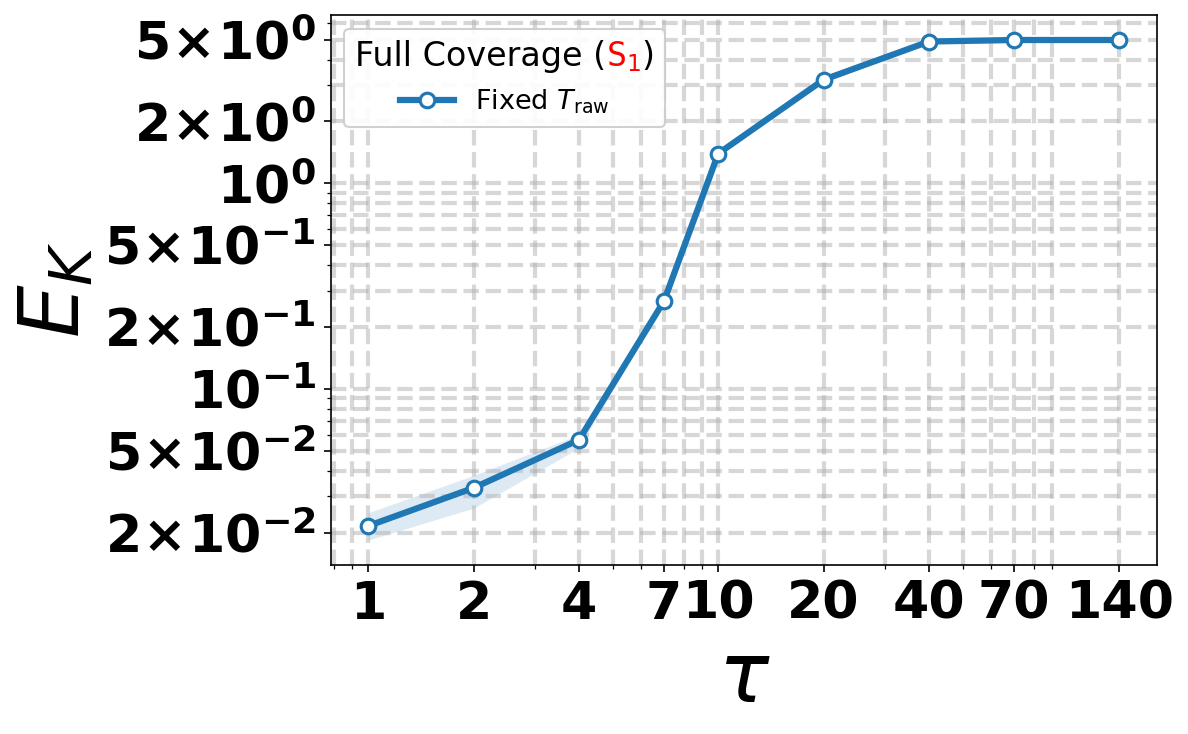}
\hspace{-2mm}
\includegraphics[scale=0.30]{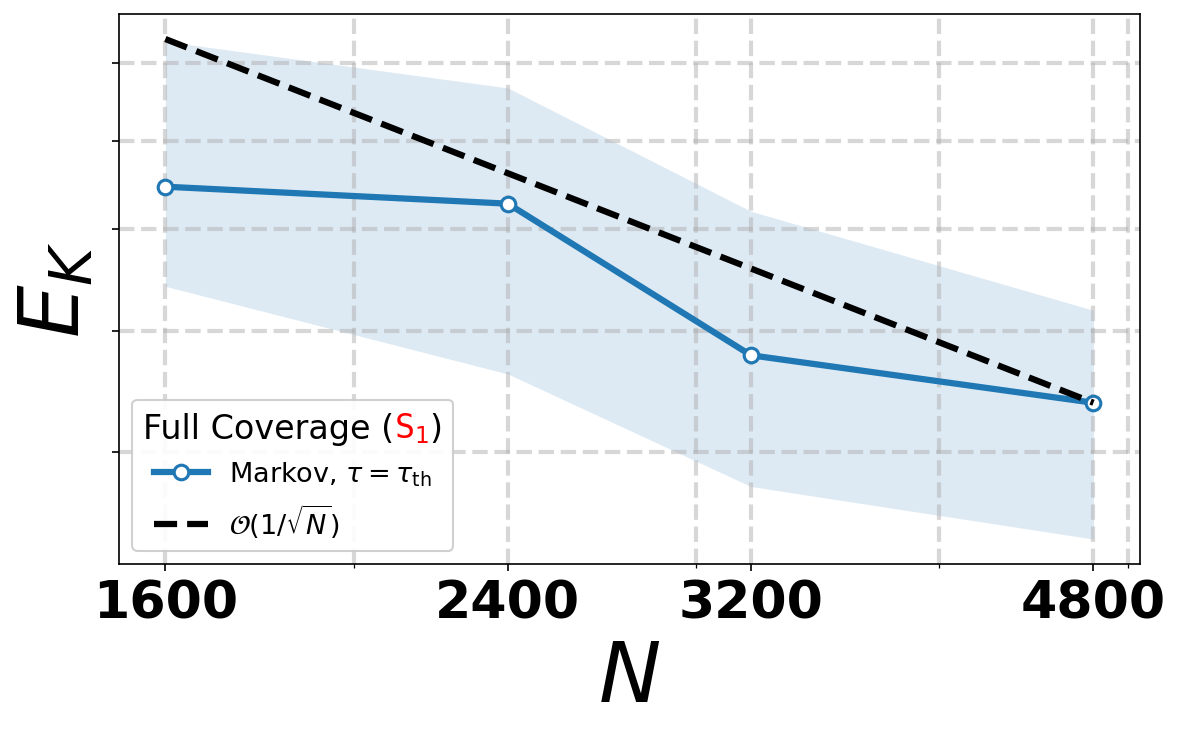}
\end{tabular}
\vspace{-6mm}
\end{center}
\caption{Markovian-data experiments under full coverage
\textup{(\textcolor{winered}{\(\texttt{S}_1\)})}.
\textbf{(Left)} Convergence against the retained transition count
\(T_{\mathrm{used}}\), comparing IID data with continuous Markov trajectories
using \(\tau\in\{1,4,\tau_{\mathrm{th}}\}\).
\textbf{(Middle)} Terminal error as a function of \(\tau\) at the fixed
physical budget \(T_{\mathrm{raw}}=8400\).
\textbf{(Right)} Collaboration benefit for
\(N\in\{1600,2400,3200,4800\}\) under the same Markovian sampling and
corruption setting.}
\label{fig:markovian_main_ablation}
\end{figure}

\paragraph{Dependence at a fixed retained budget.}
The left panel of Figure~\ref{fig:markovian_main_ablation} compares IID
sampling with unthinned and thinned Markov trajectories while holding the
number of retained transitions fixed. At
\(T_{\mathrm{used}}=2880\), the mean terminal errors are \(0.0362\) for IID
sampling, \(0.0441\) for \(\tau=1\), \(0.0374\) for \(\tau=4\), and
\(0.0373\) for \(\tau=\tau_{\mathrm{th}}=140\). Hence, the unthinned Markov
trajectory exhibits a visible dependence penalty, while moderate or
theoretically prescribed thinning reduces this penalty to the scale of the
across-run variation. This comparison does not imply that thinning is free. The \(\tau=140\) condition uses 140 times as many physical transitions as the unthinned
condition to produce the same retained sample count. The experiment therefore
isolates the effect of residual dependence only after conditioning on an equal
amount of retained information.

\paragraph{Thinning at a fixed physical budget.}
The middle panel presents the complementary experiment. Here,
\(T_{\mathrm{raw}}=8400\) is fixed and \(\tau\) is varied. The error increases
from \(0.0215\) at \(\tau=1\) to \(0.0565\) at \(\tau=4\), \(0.268\) at
\(\tau=7\), \(1.396\) at \(\tau=10\), and \(4.982\) at
\(\tau=\tau_{\mathrm{th}}=140\). Although increasing \(\tau\) weakens the dependence between retained
observations, it simultaneously reduces \(T_{\mathrm{used}}\) from \(8400\)
at \(\tau=1\) to only \(60\) at \(\tau=140\). Once the retained trajectory
becomes too sparse, the loss of state--action coverage dominates any benefit
obtained from decorrelation. The result therefore identifies a genuine
thinning tradeoff rather than a monotone improvement with the thinning
interval.

\paragraph{Collaboration under Markovian sampling.}
The right panel fixes \(\varepsilon=0.0025\),
\(\tau=\tau_{\mathrm{th}}=140\), and the retained budget, and varies the
number of agents. Increasing \(N\) from \(1600\) to \(4800\) reduces the mean
terminal error from \(0.0373\) to \(0.0269\). The intermediate errors are
\(0.0364\) at \(N=2400\) and \(0.0289\) at \(N=3200\). Thus, collaboration
continues to improve the final accuracy despite both temporal dependence and
corrupted messages. The empirical finite-grid log--log slope is approximately \(-0.329\), which supports a clear collaboration benefit but is not, by itself, evidence of an exact \(N^{-1/2}\) scaling law. The dashed
\(\mathcal{O}(1/\sqrt{N})\) curve should therefore be interpreted only as a
reference slope.

\begin{figure}[h]
\begin{center}
\begin{tabular}{cc}
\hspace{-8mm}
\includegraphics[scale=0.30]{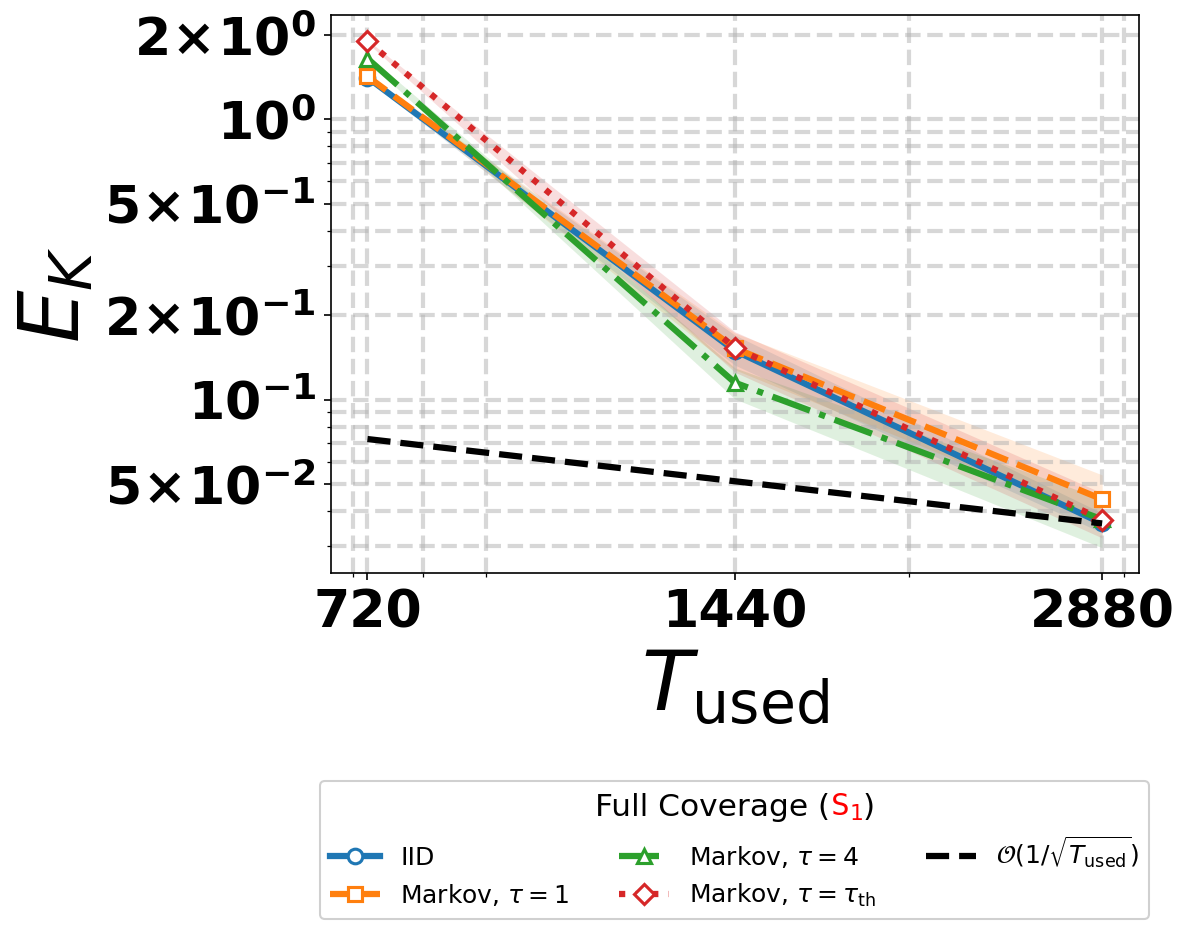}
&
\hspace{-4mm}
\includegraphics[scale=0.30]{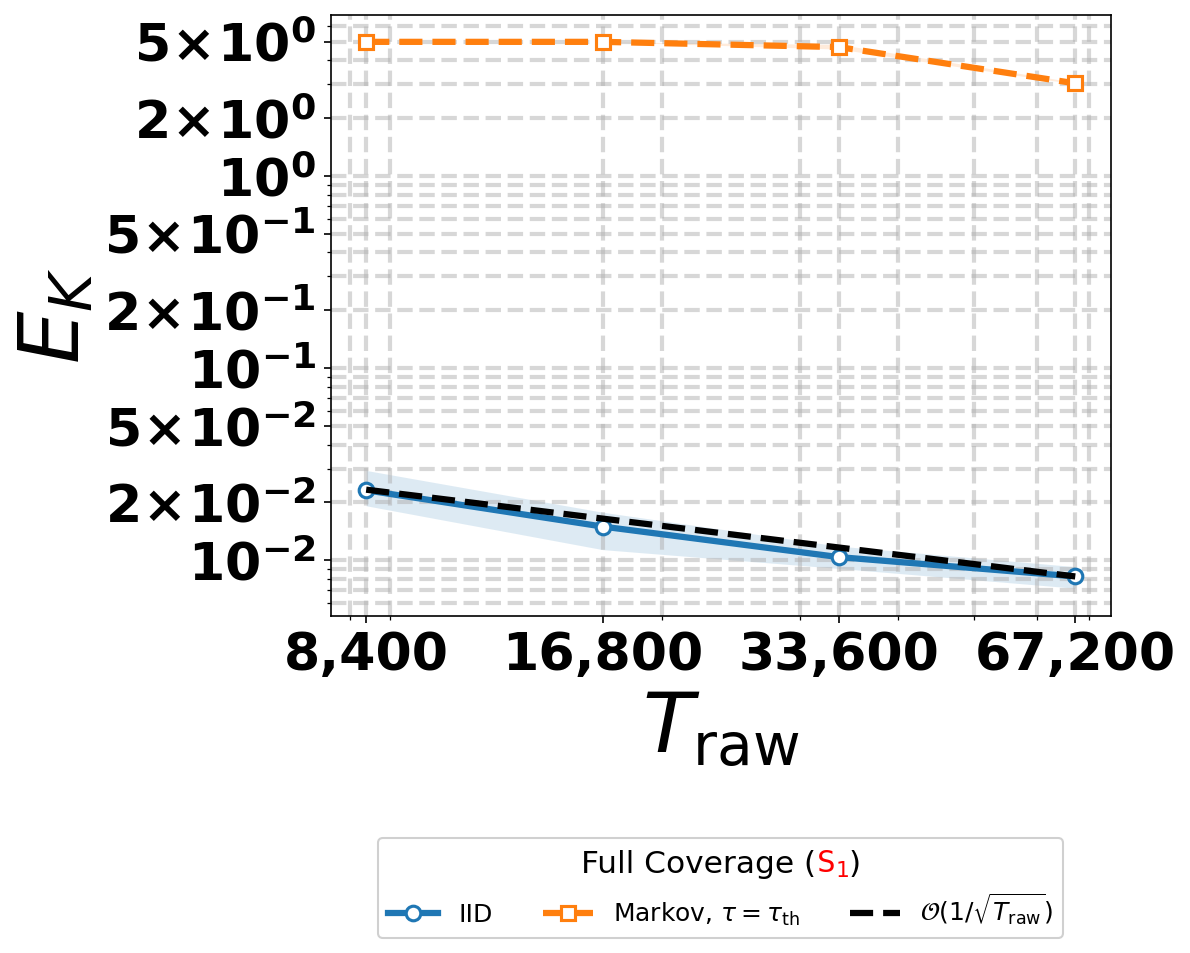}
\end{tabular}
\vspace{-6mm}
\end{center}
\caption{Retained-sample and physical-sample views of the Markovian-data
experiment under full coverage
\textup{(\textcolor{winered}{\(\texttt{S}_1\)})}.
\textbf{(Left)} Terminal error versus
\(T_{\mathrm{used}}=K(H/\tau)\) for IID observations and Markov trajectories
with \(\tau\in\{1,4,\tau_{\mathrm{th}}\}\).
\textbf{(Right)} Terminal error versus the physical budget
\(T_{\mathrm{raw}}=KH\) for IID observations and theoretically thinned
Markov trajectories. The dashed black curves are vertically rescaled
inverse-square-root slope guides.}
\label{fig:markovian_budget_ablation}
\end{figure}

\paragraph{Retained-sample complexity.} The left panel of Figure~\ref{fig:markovian_budget_ablation} varies the number
of retained transitions while holding \(K\) fixed. All four curves decrease
rapidly as \(T_{\mathrm{used}}\) increases from \(720\) to \(2880\). At the
largest retained budget, the thinned Markov errors become nearly
indistinguishable from the IID error.

At the smallest retained budget, the separation is larger: the errors are
\(1.408\), \(1.421\), \(1.645\), and \(1.896\) for IID, \(\tau=1\),
\(\tau=4\), and \(\tau=\tau_{\mathrm{th}}\), respectively. In addition to
residual finite-sample effects, increasing \(\tau\) changes the physical
horizon in the confidence allocation used by the robust aggregator. The
corresponding split-clipping ranks are therefore not identical across the
equal-retained conditions. Small differences among these curves should not
be attributed solely to temporal dependence.

\paragraph{Physical-sample complexity.}
The right panel fixes the physical-budget grid and compares IID observations
with Markov trajectories thinned using \(\tau_{\mathrm{th}}=140\). For IID
data, increasing \(T_{\mathrm{raw}}\) from \(8400\) to \(67200\) reduces the
mean error from \(0.0231\) to \(0.00824\), with an empirical log--log slope of
approximately \(-0.500\).

Under theoretical thinning, the same physical budgets yield only \(60\),
\(120\), \(240\), and \(480\) retained transitions. The corresponding mean
error decreases from \(4.982\) to \(3.031\), but the trajectory remains in a
sparse-coverage regime. In particular, the average fraction of missing
state--action coordinates decreases from \(0.840\) to \(0.248\) over the
sweep and remains non-negligible even at the largest tested physical budget.

\paragraph{Summary.}
Figures~\ref{fig:markovian_main_ablation}
and~\ref{fig:markovian_budget_ablation} establish two complementary empirical
observations. First, once the retained information is held fixed and is
sufficiently large, mixing-time-based thinning makes the performance under
Markovian data comparable to that under IID sampling. Second, achieving this
decorrelation can require a substantially longer physical trajectory. Reporting only \(T_{\mathrm{used}}\) would hide the physical cost of thinning,
whereas reporting only \(T_{\mathrm{raw}}\) would obscure the fact that the
degradation is driven by the small number of retained and sufficiently
separated observations. Both budgets are therefore necessary for a faithful
assessment of learning from Markovian data.
\section{Additional Experiments on Gymnasium Environments}
\label{sec:experiments}

We evaluate  \textcolor{winered}{\texttt{Robust Async-Fed-Q}} to determine whether the
collaborative gain predicted by our analysis persists when the agents learn
from corrupted observations in standard gymnasium environments~\cite{towers2024gymnasium}.  Our experiments are
designed to answer three questions.  First, does increasing the number of
agents reduce the estimation error in the presence of corruption?  Second,
does the empirical dependence on the network size remain consistent with the
$\mathcal{O}(1/\sqrt{N})$ statistical scaling predicted under full coverage ?
Third, how much additional error is induced as the corruption level increases,
and can this degradation be offset by increasing the number of collaborating
agents?

\subsection{Environments and Evaluation Protocol}
\label{subsec:benchmark_setup}

We consider four standard tabular benchmarks~\cite{towers2024gymnasium} from Gymnasium:
\textcolor{mygreen}{\texttt{FrozenLake-v1}},
\textcolor{mygreen}{\texttt{CliffWalking-v0}},
\textcolor{mygreen}{\texttt{Taxi-v3}}, and
\textcolor{mygreen}{\texttt{Blackjack-v1}}. These environments provide
complementary tests of the proposed method. The default \(4\times4\) instance
of \textcolor{mygreen}{\texttt{FrozenLake-v1}} contains \(16\) states and \(4\)
actions and uses slippery stochastic transitions, making accurate value
estimation difficult even without corruption. The
\textcolor{mygreen}{\texttt{CliffWalking-v0}} environment contains \(48\)
states and \(4\) actions and assigns a large negative reward upon entering the
cliff, making its action values particularly sensitive to inaccurate Bellman
estimates. The \textcolor{mygreen}{\texttt{Taxi-v3}} environment provides a
larger structured navigation problem with \(500\) encoded states and \(6\)
actions. The agent must navigate the taxi, pick up a passenger, and complete a
legal drop-off while avoiding penalties for illegal pickup and drop-off
actions. Finally, \textcolor{mygreen}{\texttt{Blackjack-v1}} has an observation
space containing \(32\times11\times2=704\) encodings and \(2\) actions,
corresponding to sticking and hitting. We use its default configuration with
\texttt{natural=False} and \texttt{sab=False}. This environment introduces
stochastic card draws, state-dependent transition distributions, and terminal
rewards determined by the random evolution of both the player's and dealer's
hands. We use the full-coverage setting
\textup{(\textcolor{winered}{\(\texttt{S}_1\)})} in all four environments.
Specifically, every active nonterminal state--action coordinate has positive
sampling probability at every agent. Under the sampling
protocol used in the experiments, each agent obtains one independent sample
from every active state--action coordinate during each epoch. This gives
\(44\), \(148\), \(2400\), and \(560\) active state--action coordinates for
\textcolor{mygreen}{\texttt{FrozenLake-v1}},
\textcolor{mygreen}{\texttt{CliffWalking-v0}},
\textcolor{mygreen}{\texttt{Taxi-v3}}, and
\textcolor{mygreen}{\texttt{Blackjack-v1}}, respectively. Unless otherwise stated, we use \(K=120\), \(\gamma=0.9\), and a unit update
stepsize. The reward observations are perturbed by bounded, symmetric noise
with unit variance. The declared robustness budget is fixed at \(0.01\), while
the realized corruption fraction \(\varepsilon\) is varied. Corrupted agents
use a coordinatewise adaptive attack and transmit either \(-10^6\) or \(10^6\),
choosing the sign that maximizes the instantaneous Bellman-operator error.
All remaining learning, sampling, and aggregation parameters are held fixed
within each sweep.

We consider the following parameter grid
\[N \in \{1600,2400,3200,4800\}, ~\text{and}~ \varepsilon \in \{0,0.001,0.0025,0.005,0.01\}.\]
Here, \(\varepsilon=0\) represents the clean baseline, whereas
\(\varepsilon>0\) specifies the realized fraction of agents whose messages are
corrupted. Performance is measured using the terminal server sup-norm
error \(E_K\). Each configuration is evaluated over \(20\) independent repetitions. The solid curves report the average terminal error across all repetitions. To
construct the shaded regions, we partition the \(20\) runs into five batches
of four runs, average within each batch, and report the pointwise minimum and
maximum of the five batch averages. The black dashed curve is a vertically
rescaled \(\mathcal{O}(1/\sqrt{N})\) reference line. It is included only as a
visual slope guide and is not obtained by fitting the experimental data. For every environment, the collaboration-benefit plot fixes
\(\varepsilon=0.0025\) and varies \(N\). The complete
\((N,\varepsilon)\) grid is then presented from two complementary viewpoints:
one plot reports \(E_K\) as a function of \(N\), with one curve for each fixed
corruption level, while the other reports \(E_K\) as a function of
\(\varepsilon\), with one curve for each fixed network size. For
\textcolor{mygreen}{\texttt{FrozenLake-v1}} and
\textcolor{mygreen}{\texttt{CliffWalking-v0}}, we additionally include an
isolated corruption-sensitivity plot that fixes \(N=3200\) and varies
\(\varepsilon\). For \textcolor{mygreen}{\texttt{Taxi-v3}} and
\textcolor{mygreen}{\texttt{Blackjack-v1}}, the three retained plots are
arranged in a single row: the left panel shows the collaboration benefit, the
middle panel plots \(E_K\) against \(N\) for each corruption level, and the
right panel plots \(E_K\) against \(\varepsilon\) for each network size.
\begin{figure}[h]
\begin{center}
\begin{tabular}{cc}
   \hspace{-8mm}
   \includegraphics[scale=0.35]{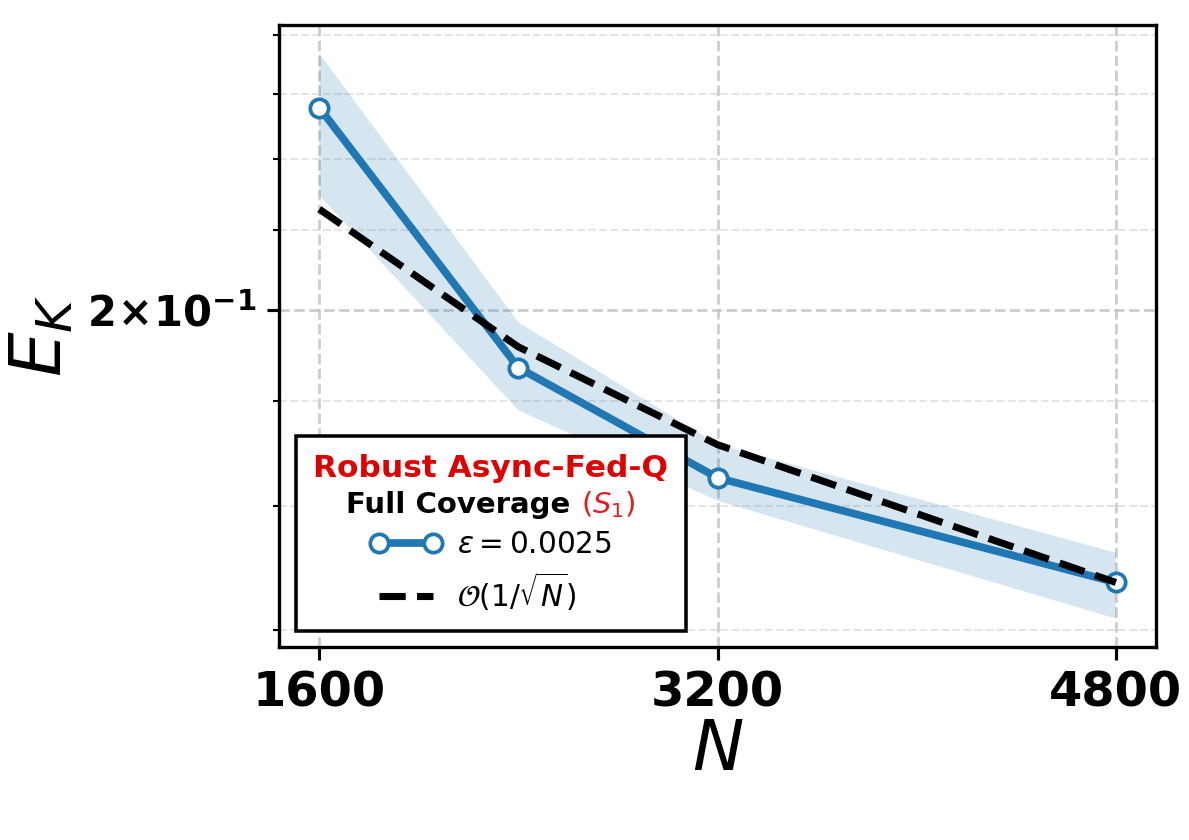}
   &
   \hspace{-4mm}
   \includegraphics[scale=0.35]{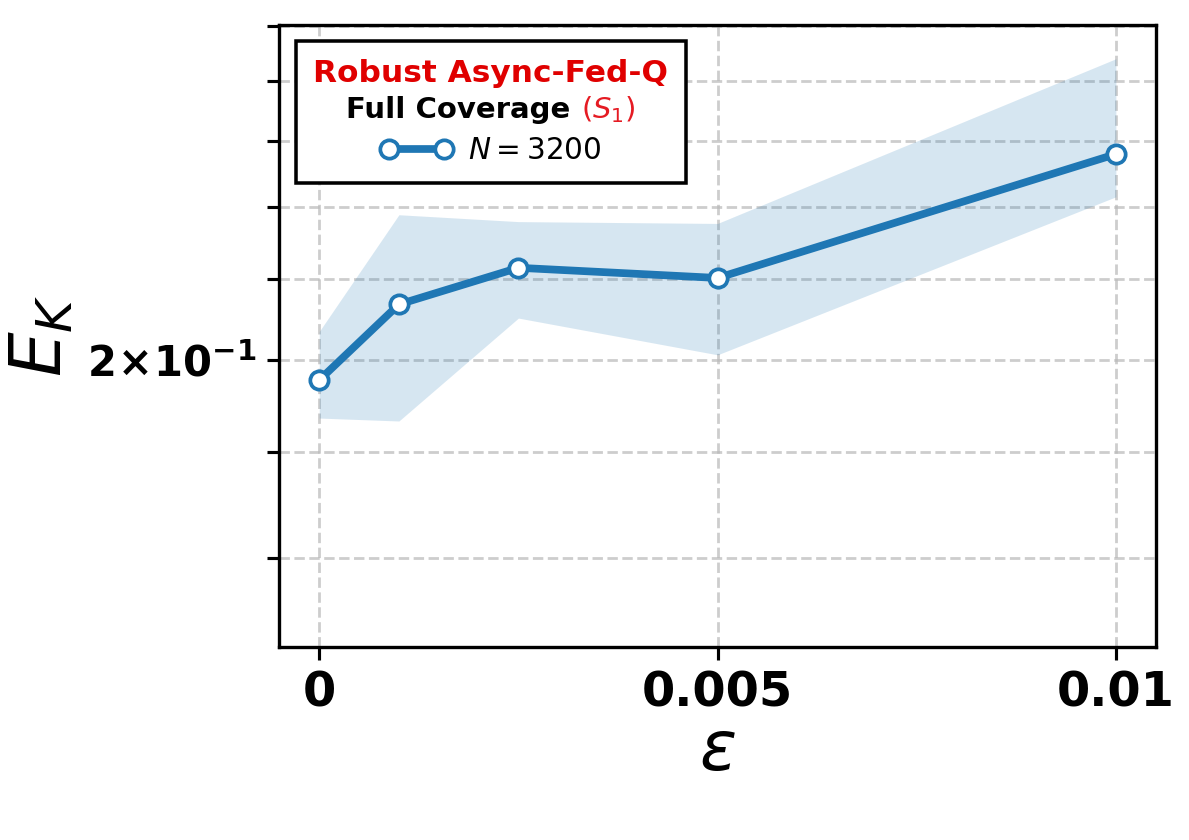}
\end{tabular}

\begin{tabular}{cc}
   \hspace{-8mm}
   \includegraphics[scale=0.35]{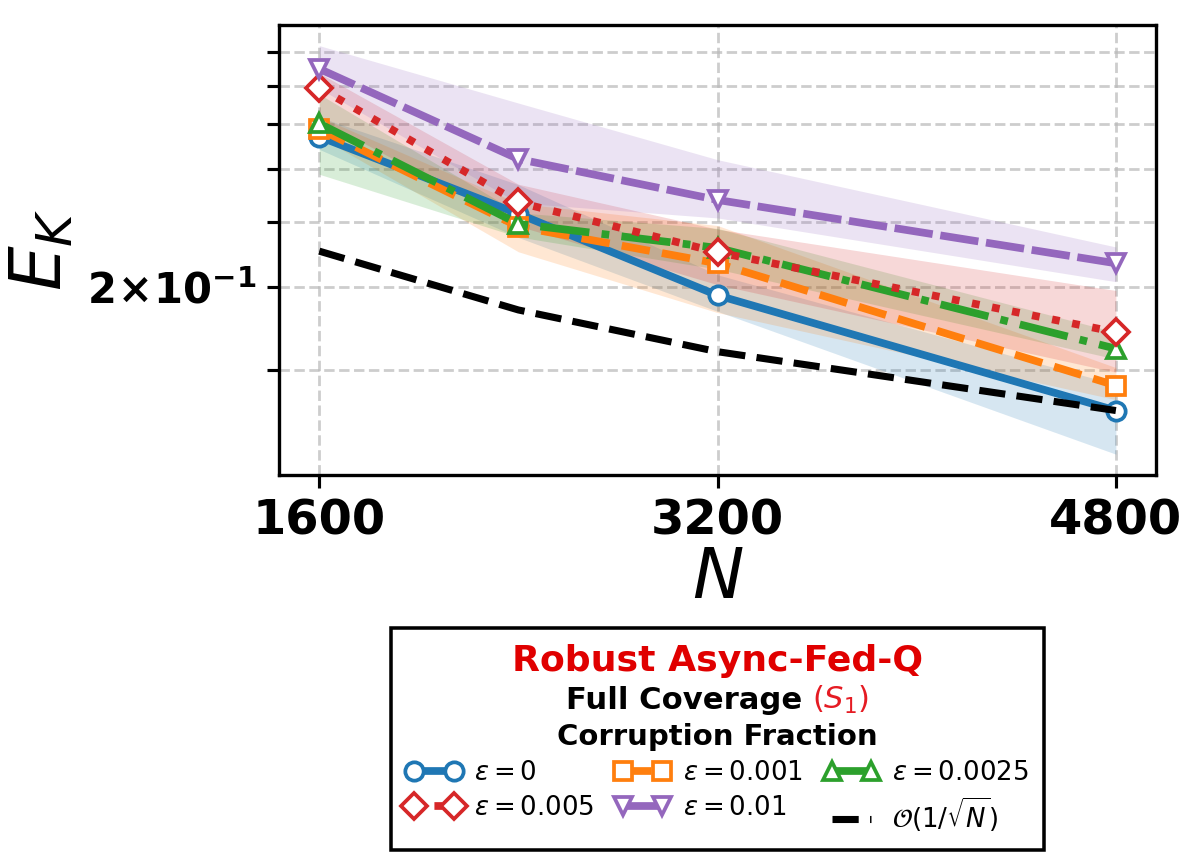}
   &
   \hspace{-5mm}
   \includegraphics[scale=0.35]{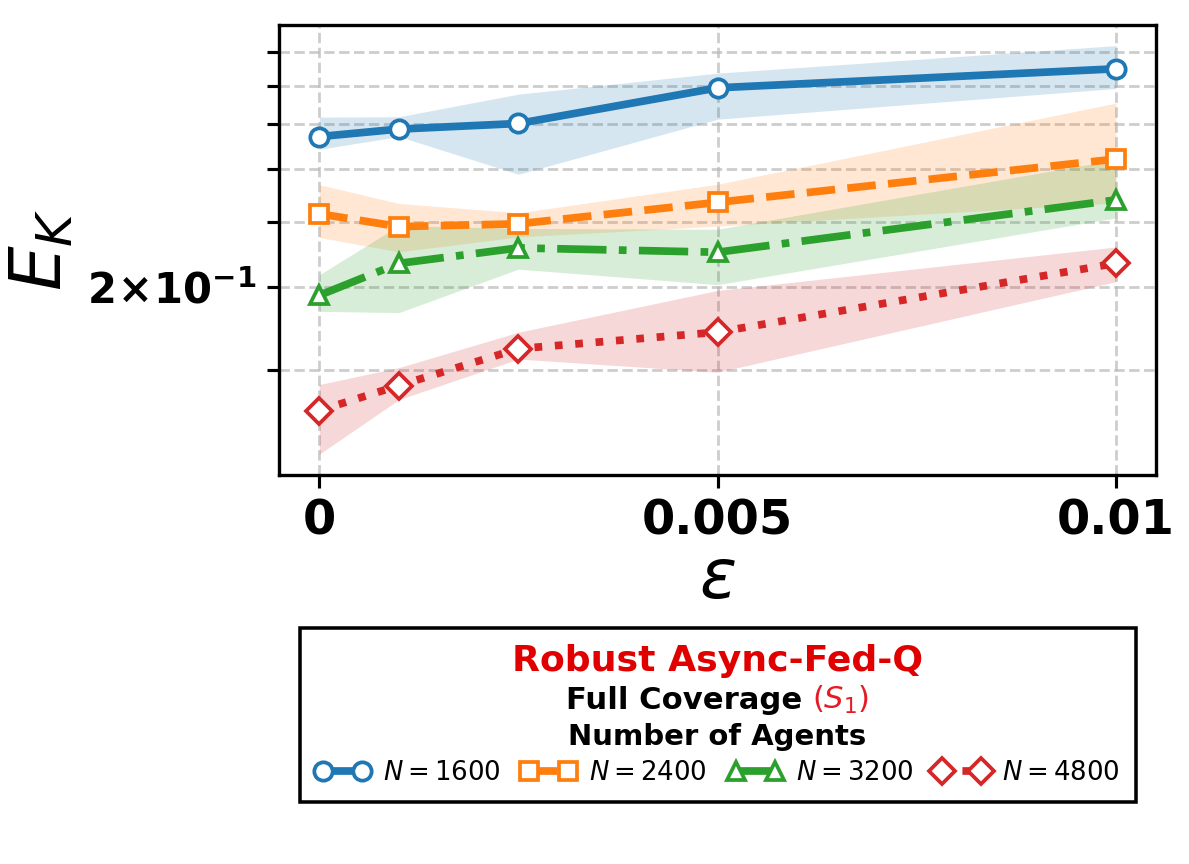}
\end{tabular}
\end{center}

\caption{Results on \textcolor{mygreen}{\texttt{CliffWalking-v0}} under full coverage \textup{(\textcolor{winered}{\(\texttt{S}_1\)})}.
\textbf{(Upper Left)} Collaboration benefit at the fixed corruption level
\(\varepsilon=0.0025\) as the number of agents varies over
\(N\in\{1600,2400,3200,4800\}\). The dashed black curve represents the
vertically rescaled \(\mathcal{O}(1/\sqrt{N})\) reference slope. \textbf{(Upper Right)} Corruption sensitivity at the fixed network size
\(N=3200\) as the corruption level varies over
\(\varepsilon\in\{0,0.001,0.0025,0.005,0.01\}\). \textbf{(Lower Left)} Terminal worst-agent error as a function of \(N\), with
one curve for each corruption level. \textbf{(Lower Right)} Terminal worst-agent error as a function of \(\varepsilon\), with one curve for each network size. Increasing \(N\) consistently reduces the error, while increasing \(\varepsilon\) raises the error envelope. The collaboration-induced decay remains visible at every tested corruption level.
}
\label{fig:cliffwalking_corruption_grid}
\end{figure}


\begin{figure}[h]
\begin{center}
\begin{tabular}{cc}
   \hspace{-14mm}
   \includegraphics[scale=0.35]{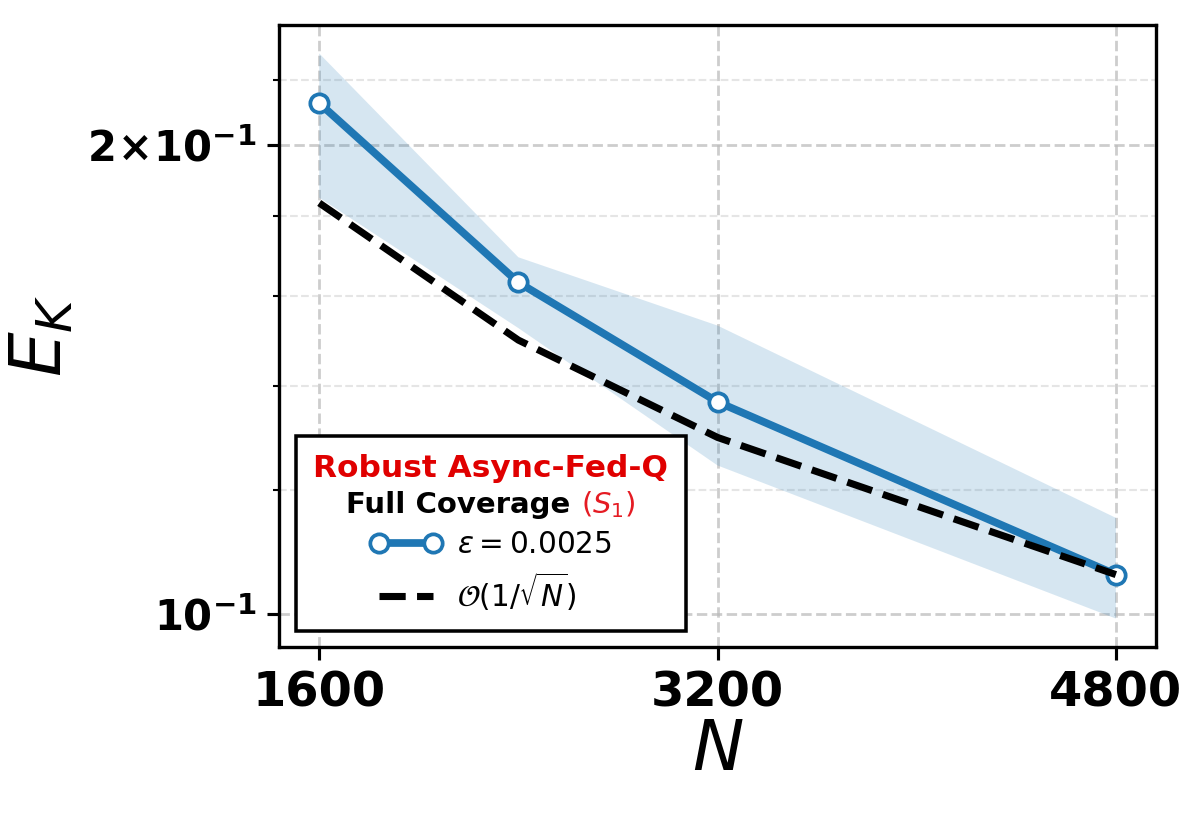}
   &
   \hspace{-6mm}
   \includegraphics[scale=0.35]{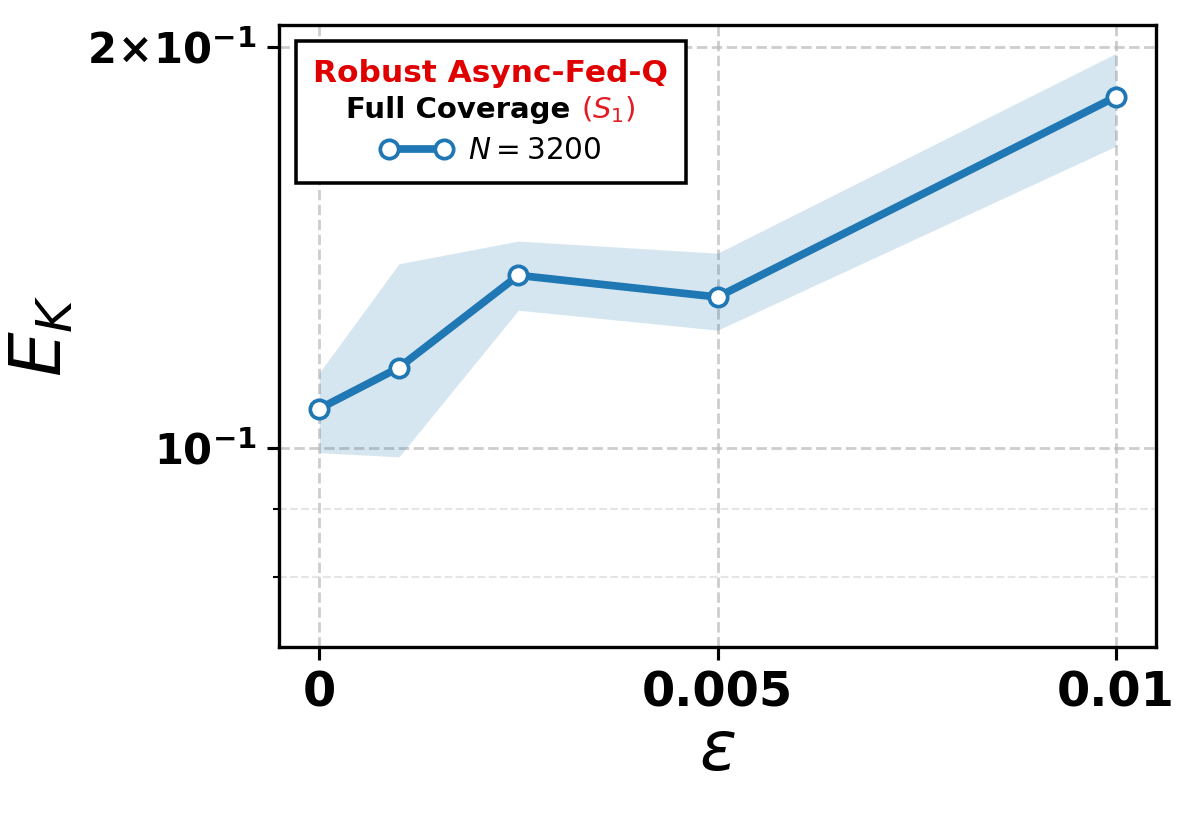}
\end{tabular}

\begin{tabular}{cc}
   \hspace{-14mm}
   \includegraphics[scale=0.35]{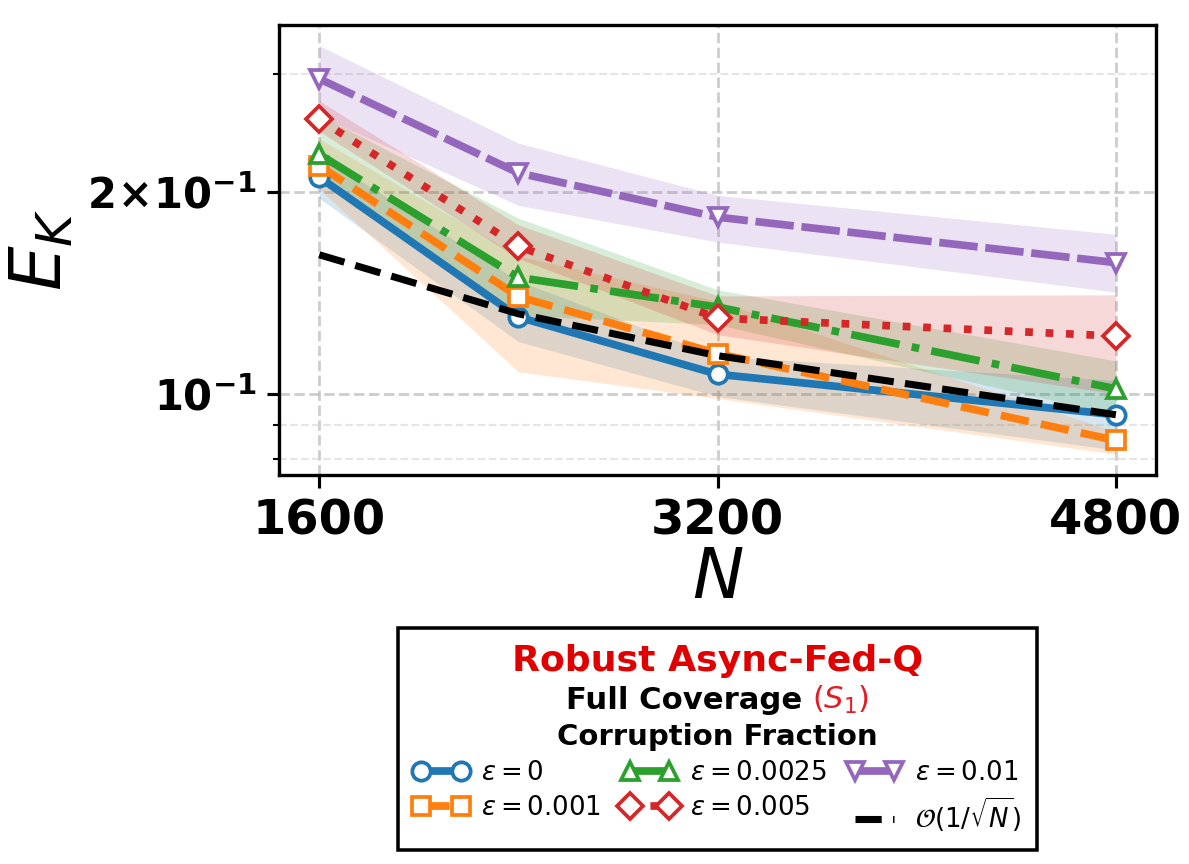}
   &
   \hspace{-6mm}
   \includegraphics[scale=0.35]{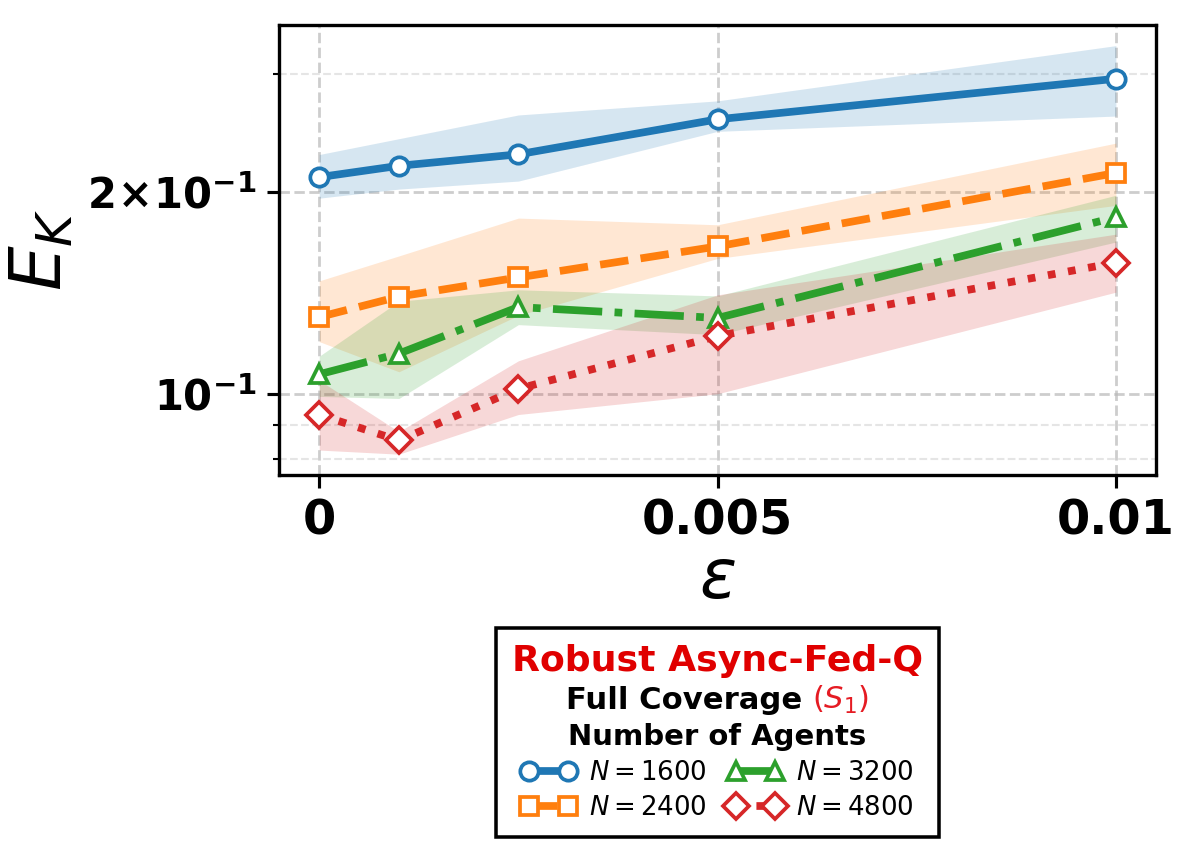}
\end{tabular}
\vspace{-6mm}
\end{center}

\caption{
Results on \textcolor{mygreen}{\texttt{FrozenLake-v1}} under full coverage
\textup{(\textcolor{winered}{\(\texttt{S}_1\)})}.
\textbf{(Upper Left)} Collaboration benefit at the fixed corruption level
\(\varepsilon=0.0025\) as the number of agents varies over
\(N\in\{1600,2400,3200,4800\}\), together with the vertically rescaled
\(\mathcal{O}(1/\sqrt{N})\) reference slope.
\textbf{(Upper Right)} Corruption sensitivity at the fixed network size
\(N=3200\) as
\(\varepsilon\in\{0,0.001,0.0025,0.005,0.01\}\).
\textbf{(Lower Left)} Terminal worst-agent error as a function of \(N\), with
one curve for each tested corruption level.
\textbf{(Lower Right)} Terminal worst-agent error as a function of
\(\varepsilon\), with one curve for each \(N\). The error decreases
with \(N\) throughout the corruption sweep, showing that additional
collaboration offsets a substantial portion of the corruption-induced
degradation.
}
\label{fig:frozenlake_corruption_grid}
\end{figure}


\begin{figure}[h]
\begin{center}
\begin{tabular}{cc}
   \hspace{-13mm}
   \includegraphics[scale=0.30]{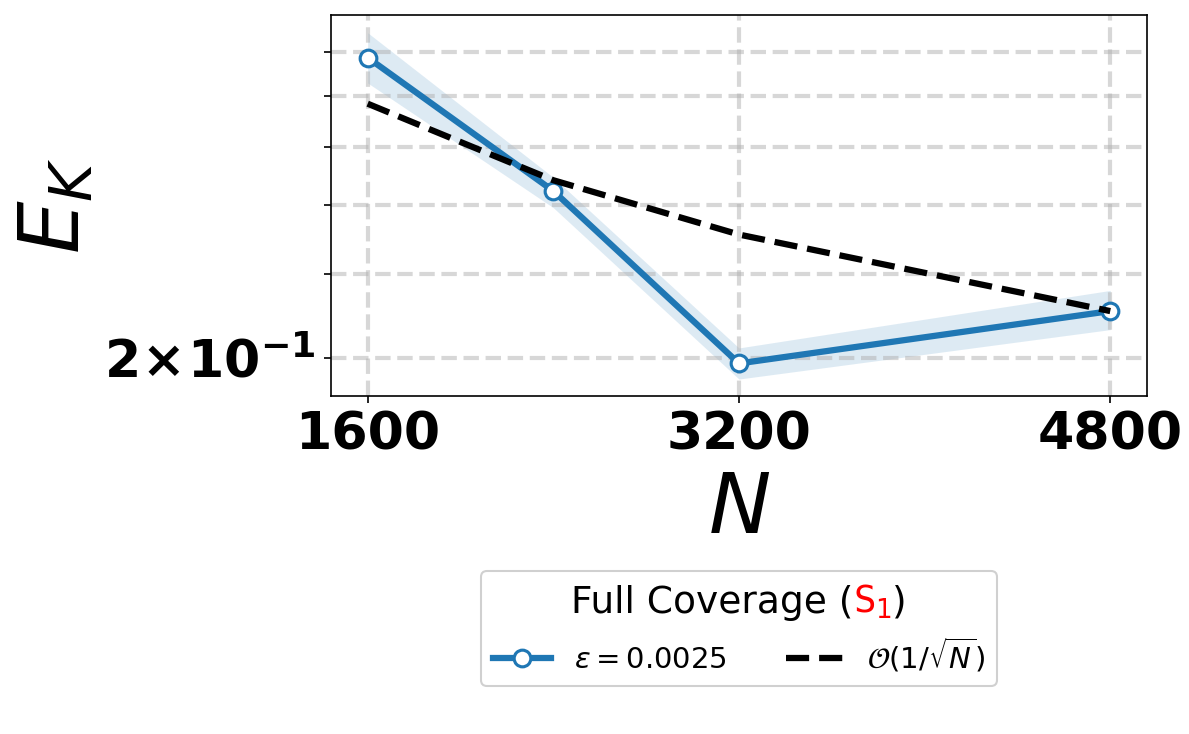}
   \hspace{-3mm}
   \includegraphics[scale=0.30]{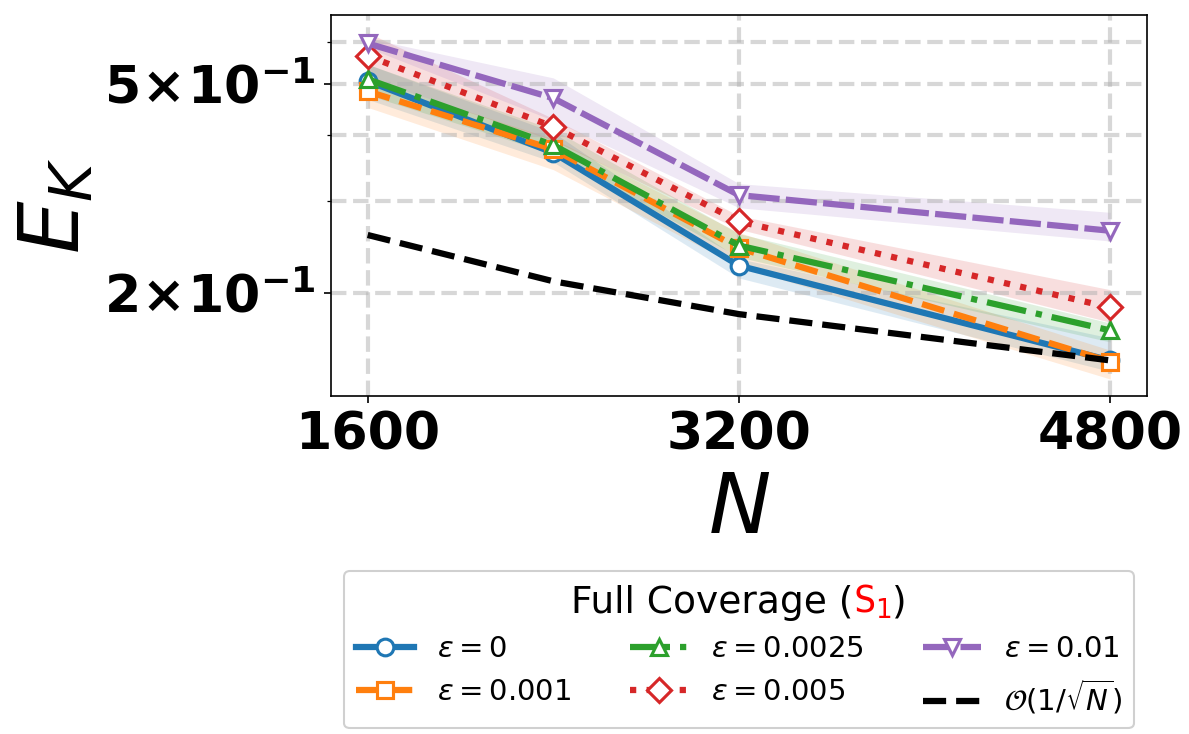}
   \hspace{-3mm}
   \includegraphics[scale=0.30]{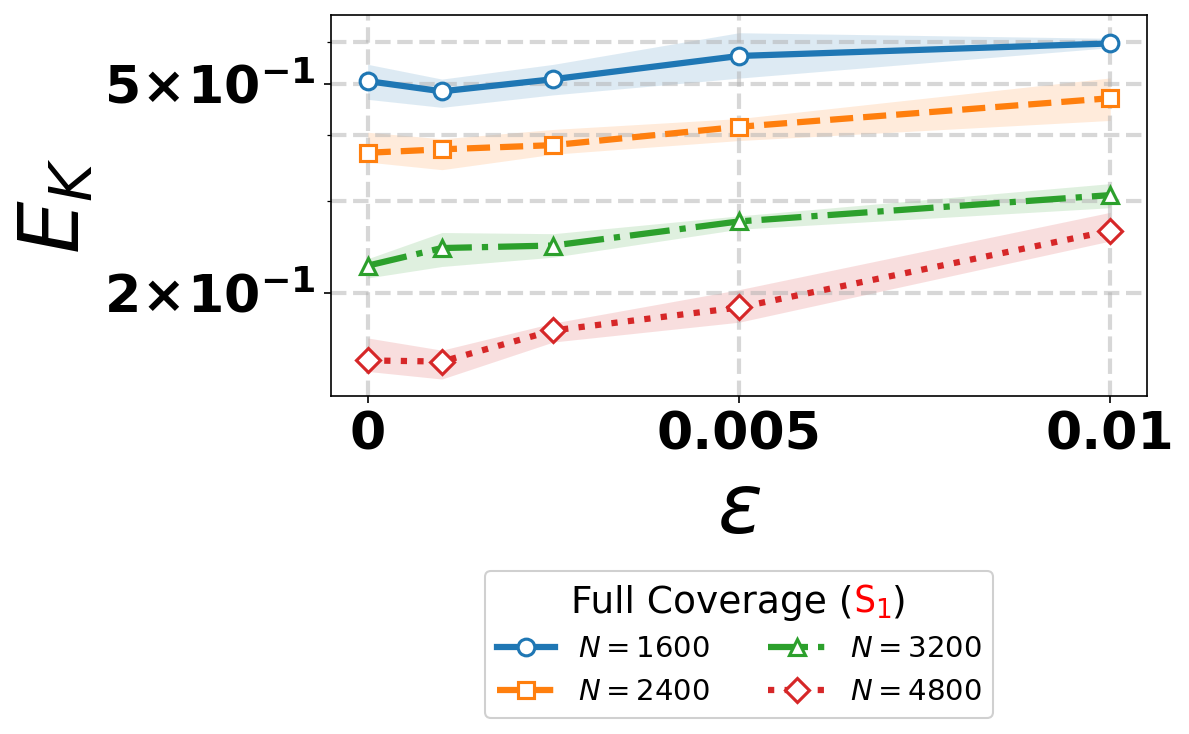}
\end{tabular}
\vspace{-6mm}
\end{center}

\caption{
Results on \textcolor{mygreen}{\texttt{Taxi-v3}} under full coverage
\textup{(\textcolor{winered}{\(\texttt{S}_1\)})}.
\textbf{(Left)} Collaboration benefit at the fixed corruption level
\(\varepsilon=0.0025\) as the number of agents varies over
\(N\in\{1600,2400,3200,4800\}\), together with the vertically rescaled
\(\mathcal{O}(1/\sqrt{N})\) reference slope.
\textbf{(Middle)} Terminal worst-agent error as a function of \(N\), with
one curve for each corruption level. Every curve decreases as the number of
agents increases, demonstrating that collaboration remains beneficial under
corruption.
\textbf{(Right)} Terminal worst-agent error as a function of
\(\varepsilon\), with one curve for each network size.
}
\label{fig:taxi_corruption_grid}
\end{figure}
\begin{figure}[h]
\begin{center}
\begin{tabular}{cc}
   \hspace{-13mm}
   \includegraphics[scale=0.30]{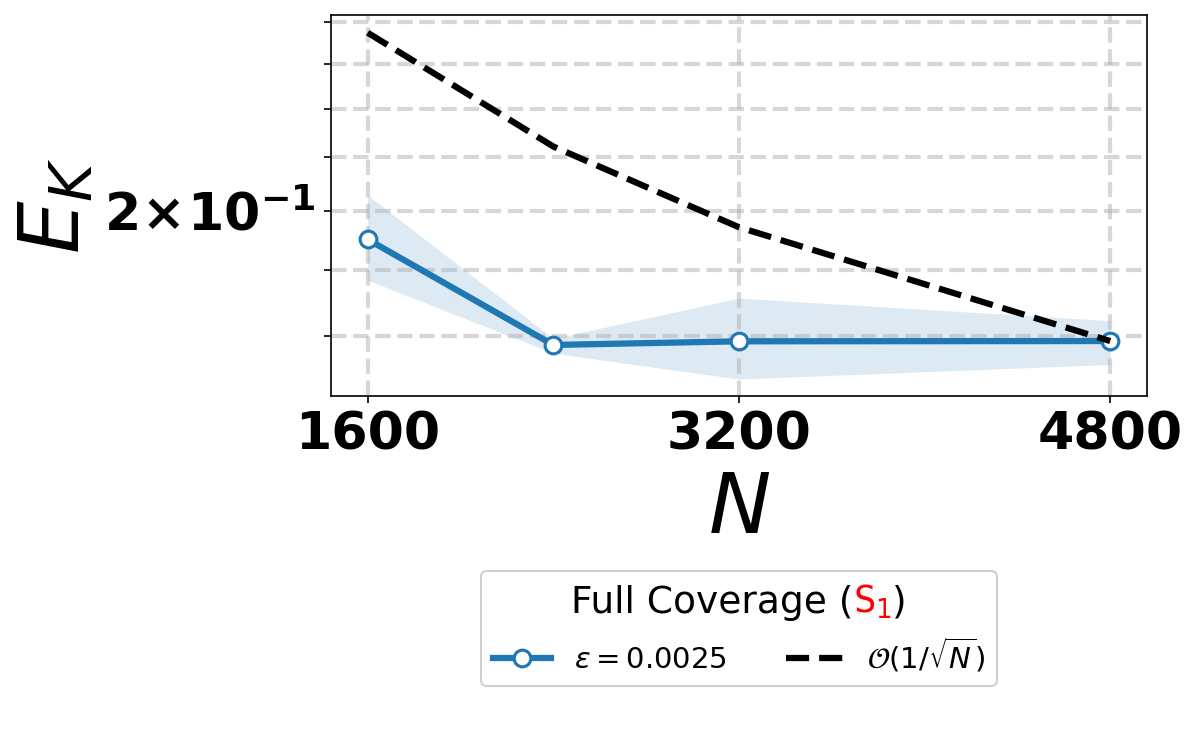}
   \hspace{-3mm}
   \includegraphics[scale=0.30]{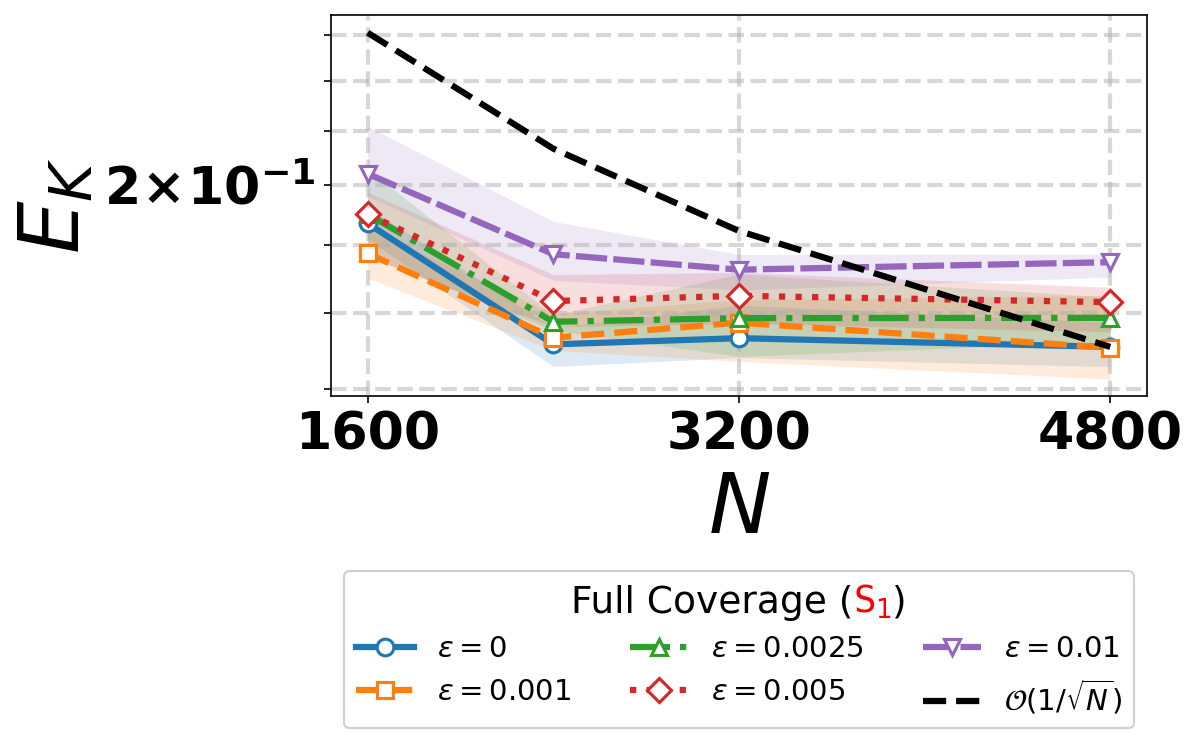}
   \hspace{-3mm}
   \includegraphics[scale=0.30]{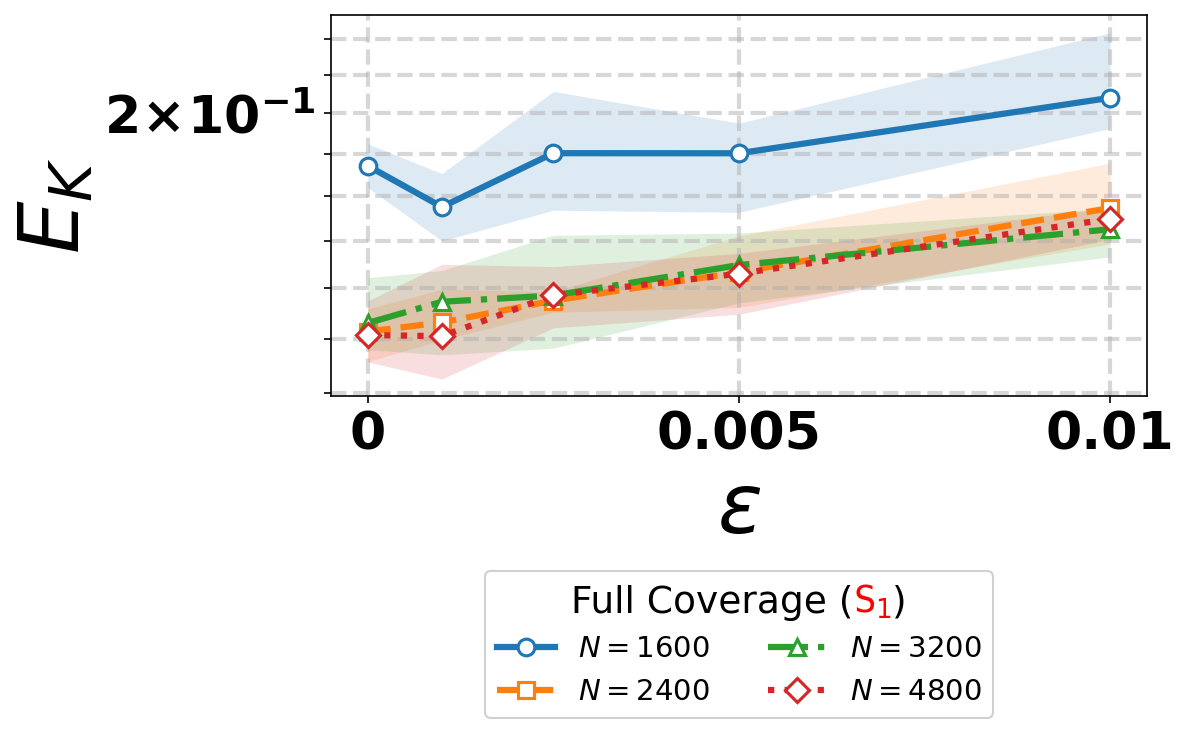}
\end{tabular}
\vspace{-6mm}
\end{center}
\caption{Results on \textcolor{mygreen}{\texttt{Blackjack-v1}} under full coverage \textup{(\textcolor{winered}{\(\texttt{S}_1\)})}.~\textbf{(Left)} Terminal error as a function of the number of agents
\(N\in\{1600,2400,3200,4800\}\), with one curve for each corruption level
\(\varepsilon\in\{0,0.001,0.0025,0.005,0.01\}\). \textbf{(Middle)} Collaboration benefit at the fixed corruption level
\(\varepsilon=0.0025\), together with the \(\mathcal{O}(1/\sqrt{N})\) reference curve. \textbf{(Right)} Terminal error as a function of \(\varepsilon\), with one curve for each network size \(N\). The results show that increasing the number of agents initially reduces the terminal error, after which the error approaches an environment-dependent floor. Increasing the corruption level consistently raises the terminal error across the tested network sizes.}
\label{fig:blackjack_corruption_grid}
\end{figure}
\subsection{Collaboration Benefit under Corruption}
\label{subsec:collaboration_corruption}

The upper-left panels of
Figures~\ref{fig:cliffwalking_corruption_grid}
and~\ref{fig:frozenlake_corruption_grid}, together with the left panels of
Figures~\ref{fig:taxi_corruption_grid}
and~\ref{fig:blackjack_corruption_grid}, fix the corruption level at
\(\varepsilon=0.0025\) and vary the number of agents \(N\). Across all four
environments, increasing the network size produces a clear overall reduction
in the terminal error \(E_K\), although the strength of this reduction depends
on the environment.

In \textcolor{mygreen}{\texttt{CliffWalking-v0}}, increasing \(N\) from
\(1600\) to \(4800\) reduces the terminal error from approximately
\(2.6\times10^{-1}\) to \(1.3\times10^{-1}\). In
\textcolor{mygreen}{\texttt{FrozenLake-v1}}, the corresponding error decreases
from approximately \(2.1\times10^{-1}\) to \(1.0\times10^{-1}\). Thus,
tripling the number of agents yields an approximately two-fold reduction in
both environments. Moreover, the two empirical curves closely follow the
slope of the displayed \(\mathcal{O}(1/\sqrt{N})\) reference line. The larger \textcolor{mygreen}{\texttt{Taxi-v3}} environment exhibits the same
overall collaboration benefit. Its terminal error decreases from approximately
\(3.0\times10^{-1}\) at \(N=1600\) to approximately
\(2.2\times10^{-1}\) at \(N=4800\). The smallest observed error occurs near
\(N=3200\), followed by a modest increase at \(N=4800\). Hence, the Taxi curve
is not strictly monotone, but the larger networks remain substantially more
accurate than the \(N=1600\) configuration. For \textcolor{mygreen}{\texttt{Blackjack-v1}}, the error decreases from
approximately \(1.90\times10^{-1}\) at \(N=1600\) to
\(1.58\times10^{-1}\) at \(N=2400\), after which it remains close to this level
through \(N=4800\). Blackjack therefore displays a clear initial collaboration
gain followed by an environment-dependent error floor. Unlike
\textcolor{mygreen}{\texttt{CliffWalking-v0}} and
\textcolor{mygreen}{\texttt{FrozenLake-v1}}, its complete curve does not follow
the \(1/\sqrt{N}\) reference slope over the entire tested range. These results show that moderate corruption changes the finite-sample error
level without eliminating the benefit of collaboration. The
\(\mathcal{O}(1/\sqrt{N})\) curve should be interpreted as a visual slope guide:
it closely describes the observed scaling on
\textcolor{mygreen}{\texttt{CliffWalking-v0}} and
\textcolor{mygreen}{\texttt{FrozenLake-v1}}, while the Taxi and Blackjack
experiments exhibit finite-sample saturation at the larger \(N\).
It demonstrates that collaboration improves the final accuracy
network-wide.
\subsection{Sensitivity to the Corruption Level}
\label{subsec:corruption_sensitivity}

The upper-right panels of
Figures~\ref{fig:cliffwalking_corruption_grid}
and~\ref{fig:frozenlake_corruption_grid} fix \(N=3200\) and vary the corruption
fraction \(\varepsilon\). The same dependence is shown by the \(N=3200\)
curves in the right panels of Figures~\ref{fig:taxi_corruption_grid}
and~\ref{fig:blackjack_corruption_grid}. Across all four environments, the
terminal error generally increases with the corruption level, although the
magnitude of this increase is environment dependent.

On \textcolor{mygreen}{\texttt{CliffWalking-v0}}, the terminal error rises from
approximately \(1.9\times10^{-1}\) in the clean case to roughly
\(2.5\times10^{-1}\) at \(\varepsilon=0.01\). On
\textcolor{mygreen}{\texttt{FrozenLake-v1}}, the corresponding increase is
from approximately \(1.1\times10^{-1}\) to \(1.8\times10^{-1}\). The larger
\textcolor{mygreen}{\texttt{Taxi-v3}} environment exhibits a similar but more
pronounced degradation, with the error increasing from approximately
\(2.3\times10^{-1}\) at \(\varepsilon=0\) to
\(3.1\times10^{-1}\) at \(\varepsilon=0.01\). Finally, on
\textcolor{mygreen}{\texttt{Blackjack-v1}}, the error increases from
approximately \(1.53\times10^{-1}\) to \(1.73\times10^{-1}\) over the same
corruption range. The degradation remains gradual rather than catastrophic. In every environment, the algorithm remains stable throughout the complete tested
interval, including at the largest corruption level. This contrasts sharply
with standard mean aggregation, which becomes unstable even under much smaller
corruption fractions, as shown later in Figure~\ref{fig:ablation_robustness}.

The \textcolor{mygreen}{\texttt{CliffWalking-v0}} and
\textcolor{mygreen}{\texttt{FrozenLake-v1}} curves exhibit small local
non-monotonicities at intermediate values of \(\varepsilon\). These
fluctuations lie within the displayed run-to-run variation and should not be
interpreted as improvements caused by corruption. In contrast, the
\(N=3200\) curves for \textcolor{mygreen}{\texttt{Taxi-v3}} and
\textcolor{mygreen}{\texttt{Blackjack-v1}} increase monotonically over the
tested grid. The reliable conclusion across all four environments is the
global one: increasing corruption raises the terminal-error envelope, with the
clearest separation occurring at \(\varepsilon=0.01\).
\subsection{Joint Effect of Agent Population and Corruption Fraction}
\label{subsec:joint_N_epsilon}

We next report the complete two-parameter sweep. The lower-left panels of
Figures~\ref{fig:cliffwalking_corruption_grid}
and~\ref{fig:frozenlake_corruption_grid}, together with the middle panels of
Figures~\ref{fig:taxi_corruption_grid}
and~\ref{fig:blackjack_corruption_grid}, plot \(E_K\) against \(N\), with one
curve for each corruption level. Across all four environments and every tested
value of \(\varepsilon\), increasing \(N\) produces a clear overall reduction
in the terminal error.

On \textcolor{mygreen}{\texttt{CliffWalking-v0}} and
\textcolor{mygreen}{\texttt{FrozenLake-v1}}, the curves for
\(\varepsilon\leq0.005\) decay broadly parallel to the
\(\mathcal{O}(1/\sqrt{N})\) reference guide. At \(\varepsilon=0.01\), the
curves are shifted upward but continue to decrease with \(N\). The
\textcolor{mygreen}{\texttt{Taxi-v3}} curves exhibit an even stronger empirical
decay over the tested range. For example, its clean error decreases from
approximately \(5.1\times10^{-1}\) at \(N=1600\) to
\(1.5\times10^{-1}\) at \(N=4800\). At \(\varepsilon=0.01\), the corresponding
decrease is from approximately \(6.0\times10^{-1}\) to
\(2.6\times10^{-1}\). The behavior of \textcolor{mygreen}{\texttt{Blackjack-v1}} is slightly different. For each corruption level, the largest reduction occurs between \(N=1600\) and \(N=2400\). The curves then approach an
environment-dependent floor, and the results for
\(N\in\{2400,3200,4800\}\) partially overlap within their displayed
variability bands. Thus, \textcolor{mygreen}{\texttt{Blackjack-v1}} exhibits a clear collaboration benefit over
the complete network-size range, but its empirical dependence on \(N\) does
not follow the \(1/\sqrt{N}\) guide after \(N=2400\). Nevertheless, even the
largest tested corruption fraction does not destroy the improvement obtained
by increasing the number of agents. The clean and lightly corrupted curves partially overlap, particularly for \(\varepsilon\in\{0.001,0.0025\}\). This is consistent with a regime in which the statistical error remains comparable to, or larger than, the corruption-induced penalty. A clearer separation generally appears as \(\varepsilon\) reaches \(0.005\) and \(0.01\). This behavior also explains why a single fixed-\(\varepsilon\) experiment is insufficient. The complete sweep distinguishes the statistical regime, where the \(N\)-dependent term dominates, from the more strongly corrupted regime, where the entire error curve is visibly elevated. The lower-right panels of
Figures~\ref{fig:cliffwalking_corruption_grid} and~\ref{fig:frozenlake_corruption_grid}, together with the right panels of
Figures~\ref{fig:taxi_corruption_grid}
and~\ref{fig:blackjack_corruption_grid}, provide the complementary view by
plotting \(E_K\) against \(\varepsilon\) for each fixed network size. The broad
ordering by \(N\) is preserved throughout the corruption sweep: the
\(N=1600\) curve lies above the curves corresponding to the larger networks.
The curves for the larger network sizes can overlap or exhibit small local
crossings, especially on \textcolor{mygreen}{\texttt{Blackjack-v1}}, but these
differences are comparable to the displayed run-to-run variability.

Most importantly, the \(N=4800\) configuration remains substantially more
accurate than the \(N=1600\) configuration at the largest corruption level.
Indeed, in all four environments, the error at
\((N,\varepsilon)=(4800,0.01)\) is lower than the clean error obtained at
\(N=1600\). Thus, within the tested parameter range, additional collaboration
can more than compensate for the degradation induced by corruption. The four environments exhibit the same qualitative structure despite their
different state-space sizes, transition mechanisms, reward scales, and
stochasticity. We do not compare their absolute error magnitudes directly,
since these depend on the value scale and dynamics of each environment. The
relevant observation is the consistency of the within-environment trends:
increasing \(\varepsilon\) raises the terminal-error envelope, increasing \(N\)
produces an overall reduction in \(E_K\), and a collaboration-dependent
separation remains visible throughout the complete corruption sweep.
\subsection{Summary of Empirical Findings}
\label{subsec:experiment_takeaways}

The experiments yield three consistent conclusions.  \textbf{First}, collaboration reduces the worst-agent terminal error under every tested corruption level. \textbf{Second}, the decay with $N$ remains close to the predicted $\mathcal{O}(1/\sqrt{N})$ statistical scaling in the benchmark environments. \textbf{Third}, increasing corruption raises the error floor, but the effect is controlled and can be offset by increasing the number of agents.  Together, these results support the central theoretical message: under full coverage, corruption and finite-sample uncertainty affect the error level, while collaboration continues to provide a network-wide statistical gain.

\section{Conclusion}
We introduced \texttt{Robust Async-Fed-Q}, which preserves collaborative gains under adversarial corruption while requiring only \(\widetilde{\mc{O}}(1)\) communication rounds. Our finite-time guarantees cover full and partial coverage, extend to Markovian sampling, and are supported by a nearly matching lower bound and numerical results. Future work will consider function approximation and heterogeneous MDPs.

\appendix
\section{Standard Results}
\label{app:add_literature_review}
In this section, we compile a few useful results that will be used by us throughout the proofs.  
\begin{lemma}(\textbf{Bernstein's Inequality})\label{lemma:bernstein}
If $X_1, X_2, \ldots, X_N$ are independent random variables with $\mathbb{P}(|X_i| \leq c) = 1$ and common mean $\mu$, then for any $\varepsilon > 0$:
\begin{equation}
\mathbb{P}(|\bar{X}_N - \mu| > \varepsilon) \leq 2 \exp\left\{ -\frac{N\varepsilon^2}{2\sigma^2 + \frac{2c\varepsilon}{3}} \right\},
\end{equation}
where $\bar{X}_N = \frac{1}{N} \sum_{i=1}^{N} X_i$ and $\sigma^2 = \frac{1}{N}\sum_{i=1}^{N} \text{Var}(X_i)$.
\end{lemma}
\begin{lemma}
 (\textbf{Hoeffding's Inequality}) 
\label{lemma:hoeffding}
If $X_1, X_2, \ldots, X_N$ are independent random variables with $\mathbb{P}(a \leq X_i \leq b) = 1$ and common mean $\mu$, then for any $\varepsilon > 0$:
\begin{equation}
\mathbb{P}(|\bar{X}_N - \mu| > \varepsilon) \leq 2 \exp\left\{\frac{-2N\varepsilon^2}{(b-a)^2}\right\},
\end{equation}
where $\bar{X}_N = \frac{1}{N} \sum_{i=1}^{N} X_i$.
\end{lemma}
\begin{lemma}\textbf{(Pinsker's inequality)}
\label{lem:pinsker}
 Let \(P\) and \(Q\) be two probability measures on a common measurable space \((\Omega,\mathcal F)\), and suppose that every event with zero probability under \(Q\) also has zero probability under \(P\). Then,
\begin{equation}
\texttt{TV}(P,Q)
\le
\sqrt{\frac{1}{2}\KL(P\|Q)}.
\end{equation}
Here, \(\operatorname{TV}(P,Q):=\sup_{A\in\mathcal F}|P(A)-Q(A)|\) denotes the total variation distance, and \(\KL(P\|Q):=\int \log\left(\frac{dP}{dQ}\right)dP\) denotes the Kullback--Leibler divergence.
\end{lemma}
\section{Additional Technical Background}
\label{app:add_prob_background}
\textbf{Asynchronous i.i.d. single-agent \(Q\)-learning.}
We next recall the classical asynchronous i.i.d. \(Q\)-learning scheme, which serves as the baseline for our federated formulation. Unlike the synchronous generative-model setting, where each iteration provides an independent transition sample for every state-action pair \((s,a)\in\mc S\times\mc A\), asynchronous i.i.d. \(Q\)-learning observes only one state-action pair at a time, sampled independently from a fixed distribution over \(\mc S\times\mc A\). At time \(t\), the learner observes a single transition tuple \((s_t,a_t,r_t,s_{t+1})\), where \((s_t,a_t)\sim \lambda\) independently across time, the reward satisfies \(\mathbb E[r_t\mid (s_t,a_t)=(s,a)]=R(s,a)\), and the next state is drawn as \(s_{t+1}\sim \mc P(\cdot\mid s_t,a_t)\). Consequently, only the entry corresponding to the sampled pair \((s_t,a_t)\) is updated at time \(t\), while all other entries of the \(Q\)-function remain unchanged. The update is given by
\begin{equation}
\label{eqn:asyncQ}
Q_{t+1}(s,a)
=
\begin{cases}
(1-\alpha_t)Q_t(s,a)
+
\alpha_t\left(r_t +\gamma\max_{a'\in\mc A}Q_t(s_{t+1},a')\right),
& \text{if }(s,a)=(s_t,a_t),\\[1mm]
Q_t(s,a),
& \text{otherwise},
\end{cases}
\end{equation}
where \(\{\alpha_t\}\) is a suitable step-size sequence. 
The update in \eqref{eqn:asyncQ} is a stochastic approximation scheme. Conditioned on the sampled pair \((s_t,a_t)\), the temporal-difference target \(r_t+\gamma\max_{a'\in\mc A}Q_t(s_{t+1},a')\) has mean \((\mc T^\star Q_t)(s_t,a_t)\). Since the Bellman optimality operator \(\mc T^\star\) is a \(\gamma\)-contraction in the sup-norm and has \(Q^\star\) as its unique fixed point, classical results show that \(Q\)-learning converges almost surely to \(Q^\star\) under standard coverage and step-size conditions~\cite{watkins1992q,tsitsiklis94,jaakkola}. Recent finite-time analyses further establish high-probability rates of order \(\widetilde O(1/\sqrt{T})\), up to problem-dependent factors induced by the sampling distribution~\cite{wainwright,Qu,li2024q}. The qualifier \emph{asynchronous} reflects the fact that only the coordinate \(Q_t(s_t,a_t)\) is updated at time \(t\), while all other coordinates remain fixed; the qualifier \emph{i.i.d.} reflects that the sampled pairs \((s_t,a_t)\) are independent draws from \(\lambda\). Consequently, the statistical accuracy is governed not only by the total horizon \(T\), but also by how the sampling distribution allocates samples across state-action pairs. Writing \(\lambda(s,a):=\mathbb P\{(s_t,a_t)=(s,a)\}\) and \(\lambda_{\min}:=\min_{(s,a)\in\mc S\times\mc A}\lambda(s,a)\), uniform finite-time guarantees require \(\lambda_{\min}>0\); after \(T\) samples, the pair \((s,a)\) receives only about \(\lambda(s,a)T\) updates. Thus, the least frequently sampled pairs determine the effective sample size and form the intrinsic statistical bottleneck of asynchronous i.i.d. \(Q\)-learning. This motivates our federated setting, where multiple agents collect data in parallel, potentially over different regions of the state-action space, and a central server aggregates their information while remaining robust to adversarially corrupted messages.
\subsection{Graph-Theoretic Interpretation of Assumptions~\ref{ass:ergodic} and~\ref{ass:committee_coverage}}
\label{app:Assumptions}
In this subsection, we provide graph-theoretic interpretations of Assumptions~\ref{ass:ergodic} and~\ref{ass:committee_coverage}, and use this viewpoint to state simple sufficient conditions under which these assumptions hold. The two assumptions require different graph constructions because they capture different notions of coverage. For Assumption~\ref{ass:ergodic}, which corresponds to the full-coverage setting, we use directed state-transition graphs to explain when a behavior policy induces an irreducible and aperiodic Markov chain, and hence assigns positive stationary mass to every state. For Assumption~\ref{ass:committee_coverage}, which corresponds to the partial-coverage setting, we use a weighted bipartite coverage graph whose left vertices are agents and whose right vertices are state-action pairs, with each state-action pair connected only to the committee of agents responsible for covering it. This perspective also highlights the key distinction between the two assumptions: Assumption~\ref{ass:ergodic} is a global per-agent coverage condition requiring each agent to cover the entire state-action space, whereas Assumption~\ref{ass:committee_coverage} is a local collective coverage condition requiring each state-action pair to be covered by at least one sufficiently informative committee. Finally, this interpretation clarifies how behavior policies and committee assignments can be designed in practice so that the required coverage conditions hold.

\paragraph{Sufficient conditions for Assumption~\ref{ass:ergodic}.}We now explain why Assumption~\ref{ass:ergodic} holds under mild structural conditions on the MDP. In the full-coverage setting \((\textcolor{winered}{\texttt{S}_1})\), all agents use the same stochastic behavior policy \(\mu\) with full action support, i.e., \(\mu(a|s)>0\) for every \((s,a)\in\mc S\times\mc A\). Under this policy, the induced state-transition kernel is \(P^\mu(s'|s):=\sum_{a\in\mc A}\mu(a|s)P(s'|s,a)\). Associate \(P^\mu\) with a directed graph \(\mc G^\mu\) on vertex set \(\mc S\), where \(s\to s'\) exists if and only if \(P^\mu(s'|s)>0\). Similarly, for each action \(a\), define the action graph \(\mc G_a\) by placing an edge \(s\to s'\) whenever \(P(s'|s,a)>0\). Since \(\mu(a|s)>0\) for all \((s,a)\), the graph \(\mc G^\mu\) contains the union \(\bigcup_{a\in\mc A}\mc G_a\). Hence, a sufficient condition for irreducibility is that this union graph is strongly connected; in particular, it is enough that there exists a single action \(a\) for which \(\mc G_a\) is strongly connected. For aperiodicity, a simple sufficient condition is that every state admits a positive-probability self-loop under some action, i.e., for every \(s\in\mc S\), there exists \(a\in\mc A\) such that \(P(s|s,a)>0\). Since \(\mu(a|s)>0\), this implies \(P^\mu(s|s)>0\) for every \(s\), and hence the induced chain is aperiodic. Therefore, if the union of the action graphs is strongly connected and every state has a self-loop under some action, then the common full-support behavior policy \(\mu\) induces an ergodic Markov chain with stationary distribution \(\pi\) satisfying \(\pi(s)>0\) for all \(s\). Since \(\mu(a|s)>0\) by construction, every agent has the same positive visitation probability \(\lambda(s,a):=\pi(s)\mu(a|s)>0\) for every \((s,a)\), which is exactly the full-coverage condition required in Assumption~\ref{ass:ergodic}. The same argument would also apply if agents used different full-support behavior policies \(\{\mu_i\}_{i=1}^N\): for each agent \(i\), the induced kernel \(P_i^{\mu}(s'|s):=\sum_{a\in\mc A}\mu_i(a|s)P(s'|s,a)\) would satisfy the same irreducibility and aperiodicity conditions, and hence would induce a stationary distribution \(\pi_i\) with \(\pi_i(s)>0\) for all \(s\), yielding \(\lambda_i(s,a)=\pi_i(s)\mu_i(a|s)>0\) for every \((s,a)\).
\paragraph{Sufficient conditions for Assumption~\ref{ass:committee_coverage}.}
Assumption~\ref{ass:committee_coverage} is strictly weaker than Assumption~\ref{ass:ergodic} from the perspective of state-action coverage. Indeed, Assumption~\ref{ass:ergodic} requires $\lambda_i(s,a)>0$ for every agent $i\in[N]$ and every $(s,a)\in\mc S\times\mc A$, whereas Assumption~\ref{ass:committee_coverage} requires only that, for each $(s,a)$, at least one agent satisfy $\lambda_i(s,a)>0$. Thus, under full coverage, every agent can provide information about every state-action pair, while under collective coverage, different agents may specialize in different portions of $\mc S\times\mc A$ as long as their individual coverage regions jointly cover the entire state-action space. This distinction can be understood through a bipartite assignment graph $\mc G=([N],\mc S\times\mc A,\mc E)$, whose left vertices are the agents and whose right vertices are the state-action pairs. An edge $(i,(s,a))\in\mc E$ indicates that agent $i$ is assigned to provide information about $(s,a)$, and $\mc N_{\mc G}(s,a):=\{i\in[N]:(i,(s,a))\in\mc E\}$ denotes the set of agents assigned to that pair. From this viewpoint, Assumption~\ref{ass:ergodic} corresponds to a complete bipartite coverage graph, whereas Assumption~\ref{ass:committee_coverage} only requires that no state-action vertex be isolated. More precisely, it is sufficient that two conditions hold: first, $\mc N_{\mc G}(s,a)\neq\emptyset$ for every $(s,a)$, so that each state-action pair is assigned to at least one agent; and second, every assigned edge is realizable under the corresponding behavior policy, meaning that $\lambda_i(s,a):=\pi_i(s)\mu_i(a\mid s)>0$ whenever $(i,(s,a))\in\mc E$. The second condition requires both that agent $i$ visit state $s$ with positive stationary probability, so that $\pi_i(s)>0$, and that it select action $a$ at that state with positive probability, so that $\mu_i(a\mid s)>0$. If these two conditions hold, then the assigned neighborhood satisfies $\mc N_{\mc G}(s,a)\subseteq\mc I(s,a):=\{i\in[N]:\lambda_i(s,a)>0\}$ for every $(s,a)$. Since every assigned neighborhood is nonempty, every actual source set $\mc I(s,a)$ is also nonempty, which establishes Assumption~\ref{ass:committee_coverage}. The same condition has a useful weighted interpretation. Assign to every actual source edge $(i,(s,a))$ the weight $\lambda_i(s,a)$, which represents the long-run fraction of agent $i$'s samples collected at $(s,a)$. The total stationary visitation mass available for $(s,a)$ across all agents is then $\Lambda_{\Sigma}(s,a):=\sum_{i=1}^N\lambda_i(s,a)=\sum_{i\in\mc I(s,a)}\lambda_i(s,a)$. Because all visitation probabilities are nonnegative, $\Lambda_{\Sigma}(s,a)>0$ holds if and only if at least one agent visits $(s,a)$ with positive stationary probability. Consequently, Assumption~\ref{ass:committee_coverage} is equivalently characterized by $\Lambda_{\Sigma,\min}:=\min_{(s,a)}\Lambda_{\Sigma}(s,a)>0$. To quantify the strength of this coverage, let $N_{s,a}:=|\mc I(s,a)|$ be the number of source agents for $(s,a)$, define $\underline N:=\min_{(s,a)}N_{s,a}$ as the minimum committee size, and define $\Lambda_{\min}:=\min_{(s,a)}\min_{i\in\mc I(s,a)}\lambda_i(s,a)$ as the smallest positive visitation probability among all source edges. Since the state-action space and the number of agents are finite, Assumption~\ref{ass:committee_coverage} implies $\underline N\geq1$ and $\Lambda_{\min}>0$. Moreover, for every $(s,a)$, the aggregate visitation mass satisfies $\Lambda_{\Sigma}(s,a)\geq N_{s,a}\Lambda_{\min}\geq\underline N\Lambda_{\min}$, because the committee contains at least $\underline N$ agents and each source agent contributes visitation mass at least $\Lambda_{\min}$. Under stationary sampling, after $T$ samples per agent, the expected aggregate number of observations of $(s,a)$ is $T\Lambda_{\Sigma}(s,a)$ and is therefore at least $T\underline N\Lambda_{\min}$. Thus, the product $\underline N\Lambda_{\min}T$ has a transparent interpretation: $\underline N$ measures the smallest number of agents covering any state-action pair, $\Lambda_{\min}$ measures the weakest visitation frequency among the source agents, and $T$ is the sampling horizon available to each agent. Hence, Assumption~\ref{ass:committee_coverage} is a weighted graph-cover condition requiring every state-action pair to be covered by a nonempty committee with strictly positive aggregate visitation mass.
\paragraph{Choosing behavior policies under $(\textcolor{winered}{\texttt{S}_2})$.}
To realize the preceding sufficient conditions, first choose an assignment graph with no isolated state-action vertices and define the region assigned to agent $i$ as $\mc C_i:=\{(s,a)\in\mc S\times\mc A:(i,(s,a))\in\mc E\}$. The behavior policy $\mu_i$ must then be chosen so that $\lambda_i(s,a)=\pi_i(s)\mu_i(a\mid s)>0$ for every $(s,a)\in\mc C_i$. This requires two separate properties. First, every assigned state $s$ must belong to the recurrent support of the Markov chain induced by $\mu_i$, ensuring that $\pi_i(s)>0$. More precisely, if $P_{\mu_i}(s'\mid s):=\sum_{b\in\mc A}\mu_i(b\mid s)P(s'\mid s,b)$ denotes the induced transition kernel, it is sufficient that all states assigned to agent $i$ belong to a closed communicating class of $P_{\mu_i}$ that is reached by the agent's sampling process. A stronger and more convenient condition is that $P_{\mu_i}$ have a unique closed communicating class containing all assigned states; global irreducibility of $P_{\mu_i}$ is an even stronger sufficient condition. Second, the policy must satisfy $\mu_i(a\mid s)>0$ for every assigned pair. This can be enforced using the exploratory mixture $\mu_i(a\mid s)=(1-\rho_i)\mu_i^{\mathrm{task}}(a\mid s)+\rho_i\mu_i^{\mathrm{exp}}(a\mid s)$, where $\rho_i\in(0,1)$ and $\mu_i^{\mathrm{exp}}(a\mid s)>0$ on $\mc C_i$. Positive action probabilities alone are insufficient, since the corresponding states must also be reachable and recurrent under the induced dynamics. For a quantitative guarantee consistent with our analysis, it is sufficient that there exist $p_{\min},u_{\min}>0$ such that $\pi_i(s)\geq p_{\min}$ and $\mu_i(a\mid s)\geq u_{\min}$ for every actual source edge $i\in\mc I(s,a)$. Then $\lambda_i(s,a)\geq p_{\min}u_{\min}$, $\Lambda_{\min}\geq p_{\min}u_{\min}$, and the effective sample size is at least $\underline Np_{\min}u_{\min}T$. Finally, this coverage requirement is separate from the local corruption constraint $|\mc B\cap\mc I(s,a)|\leq\varepsilon|\mc I(s,a)|$, which must also be enforced when constructing the source committees.
\section{Robust Mean Estimation under Independent Sub-Gaussian Inliers}
\label{app:rob_mean}
In this section, we present the robust mean-estimation subroutine (\texttt{trimSC}) used by the server in Algorithm~\ref{alg:async-fedq}. The server repeatedly needs to aggregate \(N\) scalar messages, one from each agent, into a single reliable estimate. Since even a small fraction of corrupted agents can arbitrarily bias the ordinary average, this aggregation step must tolerate an \(\varepsilon\)-fraction of malicious messages while retaining the concentration behavior of the clean messages. We therefore use \texttt{trimSC}, stated in Algorithm~\ref{alg:trimmed_sc}, as the robust aggregation rule. We first describe \texttt{trimSC} in a problem-independent form, and then state its statistical guarantee in Lemma~\ref{lem:trim-subg}.

Consider a contaminated dataset \(\mc D'=\{X_1,\ldots,X_M\}\) with at most $\varepsilon M$ arbitrary corruptions. Let $\mc G\subseteq[M]$ denote the clean indices, and assume that $\{X_i:i\in\mc G\}$ are independent with common mean $\mu$, where each $X_i-\mu$ is $\sigma_i$-sub-Gaussian. Define $\sigma^\star:=\max_{i\in\mc G}\sigma_i$. Algorithm~\ref{alg:trimmed_sc} partitions the data uniformly at random into two equal halves using fresh randomness independent of the data. Following~\cite{lugosi}, it computes clipping thresholds from the first half using $\zeta=8\varepsilon+24\log(4/\delta)/M$ and averages the clipped observations in the second half. In \texttt{Robust Async-Fed-Q}, $\mc D'$ consists of the agent-level Bellman messages for a fixed epoch and state-action pair; see Line~8 of Algorithm~\ref{alg:async-fedq}.
\begin{algorithm}[t]
\caption{Trimmed-Mean Estimator \cite{lugosi} (\texttt{trimSC})}
\label{alg:trimmed_sc}
\begin{algorithmic}[1]
\Require Contaminated dataset  \(\mc{D}'=\{X_1,X_2,\dots,X_M\}=\mc{D}_1\oplus\mc{D}_2\), where \(|\mc{D}_1|=|\mc{D}_2|=M/2\); contamination level \(\varepsilon\); confidence level \(\delta\). Define $\zeta = 8\varepsilon + 24\log(4/\delta)/M$.
\State Let $X^*_1 \leq X^*_2 \leq \cdots \leq  X^*_{M/2}$ represent a non-decreasing arrangement of $\mc{D}_1$. 
\State \texttt{Compute} quantiles: \(\alpha = X^*_{\lceil\zeta M/2\rceil}, \beta = X^*_{\lfloor(1 - \zeta) M/2\rfloor}.\) 
\State \texttt{Define} \(\phi_{\alpha,\beta}(x):={\mathbf{min}}\{{\mathbf{max}}\{x,\alpha\},\,\beta\}.\)
\State Compute  \(\tilde{\mu}= (2/M) \sum_{X_i \in \mc{D}_2}^{} \phi_{\alpha,\beta}(X_i)
\).
\end{algorithmic}
\end{algorithm}
The following lemma shows that \texttt{trimSC} achieves error of order $\varepsilon\sqrt{\log(1/\varepsilon)}+M^{-1/2}$ for independent, possibly non-identically distributed, sub-Gaussian clean observations with a common mean.
\begin{lemma}
\textbf{(Robust mean estimation; adapted from~\cite{lugosi})}
\label{lem:trim-subg}
Fix \(\delta \in (0,1)\). Let $\mc D'=\{X_1,\ldots,X_M\}$ contain at most $\varepsilon M$ arbitrary corruptions, and let $\mc G\subseteq[M]$ denote the clean indices. Suppose that $\{X_i:i\in\mc G\}$ are independent random variables  with common mean $\mu$, and that $X_i-\mu$ is $\sigma_i$-sub-Gaussian. Define $\sigma^*:=\max_{i\in\mc G}\sigma_i$. For any $\delta\in(0,1)$ satisfying $\delta\geq 8e^{-M/2}$, the estimator $\widetilde{\mu}:= \textup{\texttt{trimSC}}(\mc{D}',\varepsilon, \delta)$ in Algorithm~\ref{alg:trimmed_sc} satisfies, with probability at least $1-\delta$,
\begin{equation}
\label{eqn:trim-subg}
\left|\widetilde{\mu}-\mu\right|
\leq
\mc C\sigma^*
\left(
\varepsilon\sqrt{\log(1/\varepsilon)}
+
\sqrt{\frac{\log(4/\delta)}{M}}
\right),
\end{equation}
where $\mc C\geq1$ is a universal constant.
\end{lemma}
\section{Proof of Theorem~\ref{theorem:theoremmainresult}: Full Coverage Setting}
\label{app:proof_of_upper_bound_S1}
In this section, we prove Theorem~\ref{theorem:theoremmainresult} for the full-coverage setting $(\textcolor{winered}{\texttt{S}_1})$ by combining a Bellman error recursion with uniform visitation, boundedness, and robust aggregation guarantees. We first recall two basic properties~\cite{suttonRL} of the Bellman optimality operator \(\mathcal{T}^{*}:\mathbb{R}^{|\mathcal{S}|\times|\mathcal{A}|}\to\mathbb{R}^{|\mathcal{S}|\times|\mathcal{A}|}\): for each $Q  \in \mathbb{R}^{|\mathcal{S}|\times|\mathcal{A}|}$, we have 
\begin{equation}\label{eqn:Bellman}
(\mathcal{T}^*Q)(s,a) \coloneqq R(s,a) + \gamma\, \mathbb{E}_{s'\sim \mathcal{P}(\cdot \mid s,a)}\!\left[\max_{a'\in\mathcal{A}} Q(s',a')\right].
\end{equation}
(i) \(Q^*\) is the fixed point of $\mathcal{T}^*$, i.e., \(\mathcal{T}^*Q^*=Q^*\).
(ii) \(\mathcal{T}^*\) is a \(\gamma\)-contraction in the infinity norm: 
\begin{equation}\label{eqn:Bellmancontraction}
\|\mathcal{T}^*Q_1-\mathcal{T}^*Q_2\|_\infty \le \gamma \|Q_1-Q_2\|_\infty,\quad \forall\, Q_1,Q_2 \in \mathbb{R}^{|\mc{S}|\times|\mc{A}|}.
\end{equation}
Throughout the proof, let $e_k:=\|Q_k-Q^\star\|_\infty$. The server update in Algorithm~\ref{alg:async-fedq} is \(Q_{k+1}=(1-\eta)Q_k+\eta \tilde d_k\), as defined in equation \eqref{eqn:server_Q_update} of Section~\ref{sec:algorithm}, where \(\eta\in(0,1)\) is the step-size. Since \(Q^\star\) is the unique fixed point of the Bellman optimality operator \(\mc T^*\), we have \(Q^\star=\mc T^*Q^\star\). Combining these facts, we decompose the one-step error as
\begin{equation}\label{eq:decomposition}
\begin{aligned}
Q_{k+1}-Q^\star
&=
(1-\eta)Q_k+\eta\tilde d_k-Q^\star \\
&=
(1-\eta)(Q_k-Q^\star)+\eta(\tilde d_k-\mc T^*Q^\star) \\
&=
(1-\eta)(Q_k-Q^\star)
+\eta(\mc T^*Q_k-\mc T^*Q^\star)
+\eta(\tilde d_k-\mc T^*Q_k).
\end{aligned}
\end{equation}
Taking the \(\infty\)-norm on both sides of \eqref{eq:decomposition} and then applying the triangle inequality gives the following decomposition: 
\begin{equation}
\begin{aligned}
\|Q_{k+1}-Q^\star\|_\infty
\le
(1-\eta)\|Q_k-Q^\star\|_\infty &+\eta\|\mc T^*Q_k-\mc T^*Q^\star\|_\infty\\
&+\eta\|\tilde d_k-\mc T^*Q_k\|_\infty .
\end{aligned}
\end{equation}
Moreover, by the contraction property in \eqref{eqn:Bellmancontraction}, \(\mc T^*\) is a \(\gamma\)-contraction in the \(\infty\)-norm. Hence,
\begin{equation}
\label{eq:main_error_recursion_prelim}
e_{k+1}
\le
\left(1-\eta(1-\gamma)\right)e_k
+
\eta\left\|\tilde d_k-\mc T^*Q_k\right\|_\infty.
\end{equation}
Thus, the decomposition in \eqref{eq:main_error_recursion_prelim}
separates the one-step error into the \emph{Bellman contraction term}, which propagates and contracts the error from the preceding epoch, and the \textbf{robust aggregation error} $\left\|\tilde d_k-\mc T^*Q_k\right\|_\infty$, which captures both finite-sample fluctuations in the honest empirical Bellman updates and the effect of corrupted messages. Since the final bound follows by recursively unrolling \eqref{eq:main_error_recursion_prelim}, we must control the {robust aggregation error} simultaneously over all $k\in[K]$.

We next establish auxiliary lemmas that, together with a union bound, provide this uniform high-probability control.
\begin{lemma}\textbf{(Uniform boundedness of iterates)}
\label{lem:s1_uniform_boundedness}
The iterates generated by Algorithm~\ref{alg:async-fedq} satisfy
\begin{equation}
\label{eq:s1_q_uniform_bound}
\|Q_k\|_\infty
\le B:=
\frac{\bar R}{1-\gamma},
\qquad
\forall k\ge 0.
\end{equation}
\end{lemma}
\begin{proof}
We prove the claim by induction for all \(k \ge 0\).

\textbf{Base case:}
Since Algorithm~\ref{alg:async-fedq} initializes $Q_0=0$, we have $\|Q_0\|_\infty=0\leq \bar R/(1-\gamma)$, establishing the base case.

\textbf{Induction hypothesis:}
Assume that $\|Q_k\|_\infty\leq B:=\bar R/(1-\gamma)$, and fix an arbitrary $(s,a)\in\mc S\times\mc A$. We first show that every uncorrupted message $d_{i,k}(s,a)$ sent by an agent $i\in\mc G$ belongs to $[-B,B]$. The argument distinguishes whether agent $i$ visits $(s,a)$ during epoch $k$, equivalently, whether $N_{i,k}(s,a) > 0$, or $N_{i,k}(s,a) =0$.

\(\textcolor{winered}{\blacksquare}\) \textbf{Boundedness of honest agent messages.}
First, we show that \(d_{i,k}(s,a)\) is uniformly bounded by \(B\) for every \((s,a)\in\mc S\times\mc A\), every good agent \(i\in\mc G\), and every epoch \(k\ge 0\).

\(\bullet\) \textbf{Case 1 (\(N_{i,k}(s,a)>0\))~:}
In this case, the good agent \(i\) constructs \(d_{i,k}(s,a)\) using the samples collected from \((s,a)\) during epoch \(k\). This object is defined in \eqref{eqn:dik_update}. Taking absolute values on each side of \eqref{eqn:dik_update}, we obtain the following bound for a pair \((s,a) \in \mc{S} \times \mc{A}\) with nonzero visitation:

\begin{equation}
\begin{aligned}
|d_{i,k}(s,a)|
&=
\left|
\frac{1}{N_{i,k}(s,a)}
\sum_{j=1}^{N_{i,k}(s,a)}
\left(
r_j(s,a)
+
\gamma
\sum_{s'\in\mc S}
\mathbbm{1}^{i,j,k}(s'\mid s,a)
\max_{a'\in\mc A}Q_k(s',a')
\right)
\right| \\
&\overset{(\textcolor{winered}{\bullet})}{\le}
\frac{1}{N_{i,k}(s,a)}
\sum_{j=1}^{N_{i,k}(s,a)}
\left(
|r_j(s,a)|
+
\gamma
\sum_{s'\in\mc S}
\mathbbm{1}^{i,j,k}(s'\mid s,a)
\|Q_k\|_\infty
\right) \\
&\overset{(\textcolor{winered}{\bullet \bullet})}{\le}
\bar R
+
\frac{\gamma\|Q_k\|_\infty}{N_{i,k}(s,a)}
\sum_{j=1}^{N_{i,k}(s,a)}
\sum_{s'\in\mc S}
\mathbbm{1}^{i,j,k}(s'\mid s,a) =
\bar R+\gamma\|Q_k\|_\infty
\le
B.
\end{aligned}
\label{eq:good_agent_bound_trim_positive_count_async}
\end{equation}
Here, \((\textcolor{winered}{\bullet})\) follows from the triangle
inequality and \(|\max_{a'}Q_k(s',a')|\le\|Q_k\|_\infty\), while
\((\textcolor{winered}{\bullet\bullet})\) uses \(|r_j(s,a)|\le\bar R\).
The final line follows from
\(\sum_{s'\in\mc S}\mathbbm{1}^{i,j,k}(s'\mid s,a)=1\)
and the induction hypothesis \(\|Q_k\|_\infty\le B=\bar R/(1-\gamma)\).

\noindent \(\bullet\) \textbf{Case 2 (\(N_{i,k}(s,a)=0\))~:} In this case, by the convention of Algorithm~\ref{alg:async-fedq}, the good agent sends \(d_{i,k}(s,a)=0\). Hence, \(|d_{i,k}(s,a)|=0\le B\) trivially. Combining the two cases, every uncorrupted message in \(\{d_{i,k}(s,a)\}_{i\in\mc G}\) lies in \([-B,B]\).

\(\textcolor{winered}{\blacksquare}\) \textbf{Uniform boundedness of the server aggregate.} Next, we show that the server aggregate \(\tilde d_k(s,a)\) inherits the same uniform bound. We prove the result for even \(N\); the odd case follows by the same argument with an unequal split. Fix an epoch \(k\ge 0\) and a state-action pair \((s,a)\in\mc S\times\mc A\). The server applies Algorithm~\ref{alg:trimmed_sc} to the corrupted collection of agent messages
\begin{equation}
\widetilde{\mc D}_{k}(s,a)
:=
\{d_{1,k}(s,a),d_{2,k}(s,a),\ldots,d_{N,k}(s,a)\}.
\end{equation}
This collection is split into two halves according to Algorithm~\ref{alg:trimmed_sc} (defined in Appendix~\ref{app:rob_mean}) as follows:
\begin{equation}
\widetilde{\mc D}_{k}(s,a)
=
\mc D_{1,k}(s,a)\oplus \mc D_{2,k}(s,a),
\quad
|\mc D_{1,k}(s,a)|=|\mc D_{2,k}(s,a)|=\frac{N}{2}.
\end{equation}
Next, we arrange the samples in \(\mc D_{1,k}(s,a)\) in non-decreasing order as follows:
\begin{equation}
d^{*,1}_{1,k}(s,a)
\le
d^{*,1}_{2,k}(s,a)
\le
\cdots
\le
d^{*,1}_{N/2,k}(s,a).
\end{equation}
Following Line~2 of Algorithm~\ref{alg:trimmed_sc}, recall the lower and upper order-statistic indices as \(r_\zeta:=\lceil\zeta N/2\rceil\) and \(u_\zeta:=\lfloor(1-\zeta)N/2\rfloor\), and set the corresponding clipping thresholds to \(\alpha_k(s,a):=d^{*,1}_{r_\zeta,k}(s,a)\) and \(\beta_k(s,a):=d^{*,1}_{u_\zeta,k}(s,a)\). Moreover, since at most an \(\varepsilon\)-fraction of the \(N\) agent messages are corrupted, at most \(\varepsilon N\) messages are corrupted in total.

We next show that the clipping thresholds lie in \([-B,B]\). Since \(\varepsilon\geq0\) and \(\delta\in(0,1)\), we have \(r_\zeta\geq\zeta N/2=4\varepsilon N+12\log(4/\delta)>\varepsilon N\). As at most \(\varepsilon N\) messages are corrupted in the entire collection, the first split \(\mc D_{1,k}(s,a)\) also contains at most \(\varepsilon N<r_\zeta\) corrupted messages, regardless of how corruptions are distributed between the two splits. Recall that every uncorrupted message lies in \([-B,B]\). If \(\alpha_k(s,a)=d^{*,1}_{r_\zeta,k}(s,a)<-B\), then the first \(r_\zeta\) ordered messages must all lie below \(-B\) and hence be corrupted, contradicting the corruption budget. Thus, \(\alpha_k(s,a)\geq-B\). Similarly, if \(\beta_k(s,a)=d^{*,1}_{u_\zeta,k}(s,a)>B\), then all messages occupying positions \(u_\zeta,\ldots,N/2\) must exceed \(B\) and hence be corrupted. The number of such messages is \(N/2-u_\zeta+1=r_\zeta+1>\varepsilon N\), yielding the same contradiction. Therefore, \(\beta_k(s,a)\leq B\), and the ordering \(r_\zeta\leq u_\zeta\) gives \(-B\leq\alpha_k(s,a)\leq\beta_k(s,a)\leq B\).

Consequently, the clipping operation in Line~3 of Algorithm~\ref{alg:trimmed_sc} maps
every message in \(\mc D_{2,k}(s,a)\), including corrupted messages,
into \([\alpha_k(s,a),\beta_k(s,a)]\subseteq[-B,B]\).
Since Line~4 averages these clipped messages, the resulting estimate
also lies in this interval, yielding
\(|\widetilde d_k(s,a)|\le B=\bar R/(1-\gamma)\).

\(\textcolor{winered}{\blacksquare}\) \textbf{Uniform boundedness of the next iterate.} Finally, since \(\eta\in[0,1]\), the update \(Q_{k+1}(s,a)=(1-\eta)Q_k(s,a)+\eta\widetilde d_k(s,a)\) is a convex combination of two elements of \([-B,B]\), and therefore also belongs to \([-B,B]\). Since \((s,a)\in\mc S\times\mc A\) was arbitrary, \(\|Q_{k+1}\|_\infty\leq B\), establishing the induction step and completing the proof.
\end{proof}

\textbf{High-Probability Uniform Visitation Bound over State-Action Pairs.} We next establish the uniform visitation event used to control \(d_{i,k}(s,a)\). Under \((\textcolor{winered}{\texttt{S}_1})\), Assumption~\ref{ass:ergodic} ensures that every state-action pair has positive sampling probability, and the following lemma shows that each honest agent visits every pair \(\Omega(\lambda_{\min}H)\) times per epoch with high probability.
\begin{lemma}\textbf{(Uniform visitation event in} \(\textcolor{winered}{\texttt{S}_1}\)\textbf{)}
\label{lem:s1_uniform_visitation}
Fix a confidence parameter \(\delta\in(0,1)\). Suppose the epoch length satisfies \(
H \ge
(104/3\lambda_{\min})
\log\left(2NT|\mc S||\mc A|/\delta\right).\) Then the following event holds with probability at least \(1-\delta/2\):
\begin{equation}
\label{eq:s1_good_visit_event}
\mc V
:=
\left\{
N_{i,k}(s,a)
\ge
\frac{3}{4}\lambda_{\min}H,
\forall i\in[N],\ \forall k\in[K],\ \forall(s,a)\in\mc S\times\mc A
\right\}.
\end{equation}
\end{lemma}
\begin{proof}
Fix an agent \(i\in[N]\), an epoch \(k\in[K]\), and a state-action pair \((s,a)\in\mc S\times\mc A\). By definition, the number of visits of agent \(i\) to \((s,a)\) during epoch \(k\) is
\(
N_{i,k}(s,a)
=
\sum_{t \in \mc{I}_k}
\mathbbm{1}\{(s_{i,t},a_{i,t})=(s,a)\},
\)
where recall that $\mc{I}_k$ contains the time-steps in the $k$-th epoch. 
Let
\begin{equation}
X_t^{i,k}(s,a)
:=
\mathbbm{1}\{(s_{i,t},a_{i,t})=(s,a)\},
\quad t \in \mc{I}_k.
\end{equation}
Under the full-coverage i.i.d. sampling model, \(X_t^{i,k}(s,a)\) is a Bernoulli random variable with mean \(\lambda_i(s,a)\), where \(\lambda_i(s,a)\) is the sampling probability of \((s,a)\) at stationarity under the behavior policy \(\mu_i\). Hence,
\(
\mathbb E[N_{i,k}(s,a)]
=
\sum_{t \in \mc{I}_k} \mathbb E[X_t^{i,k}(s,a)]
=
\lambda_i(s,a)H. 
\)
We now apply Bernstein's inequality to lower bound \(N_{i,k}(s,a)\). Since \(X_t^{i,k}(s,a)\in \{0,1\}\), we have
\(
\left|X_t^{i,k}(s,a)-\lambda_i(s,a)\right|\le 1,
\)
and
\(
\sum_{t \in \mc{I}_k}
\mathrm{Var}\left(X_t^{i,k}(s,a)\right)
=
H\lambda_i(s,a)(1-\lambda_i(s,a))
\le
H\lambda_i(s,a).
\)
Applying Bernstein's bound from \cite{chungconc} to the centered independent variables \(\{X_t^{i,k}(s,a)-\lambda_i(s,a)\}_{t \in \mc{I}_k}\), with deviation \(\lambda_i(s,a)H/4\), gives
\begin{equation}
\begin{aligned}
&\mathbb P\left(
N_{i,k}(s,a)
<
\frac{3}{4}\lambda_i(s,a)H
\right)
=
\mathbb P\left(
\sum_{t \in \mc{I}_k}
\left(X_t^{i,k}(s,a)-\lambda_i(s,a)\right)
<
-\frac{1}{4}\lambda_i(s,a)H
\right) \\
&\le
\exp\left(
-
\frac{\left(\lambda_i(s,a)H/4\right)^2}
{2H\lambda_i(s,a)+(2/3)\lambda_i(s,a)H/4}
\right) =
\exp\left(
-\frac{3}{104}\lambda_i(s,a)H
\right)
\le
\exp\left(
-\frac{3}{104}\lambda_{\min}H
\right).
\end{aligned}
\end{equation}
Since \(\lambda_i(s,a)\ge \lambda_{\min}\), the event \(N_{i,k}(s,a)\ge (3/4)\lambda_i(s,a)H\) immediately implies \(N_{i,k}(s,a)\ge (3/4)\lambda_{\min}H\). Consequently, the probability that \(N_{i,k}(s,a)\) falls below \((3/4)\lambda_{\min}H\) is bounded from above by the probability that it falls below \((3/4)\lambda_i(s,a)H\), and hence
\begin{equation}
\mathbb P\left(
N_{i,k}(s,a)
<
\frac{3}{4}\lambda_{\min}H
\right)
\le
\exp\left(
-\frac{3}{104}\lambda_{\min}H
\right).
\end{equation}
Now define the bad event \[\mc V^c:=\left\{\min_{i\in[N],\,k\in[K],\,(s,a)\in\mc S\times\mc A}N_{i,k}(s,a)<3\lambda_{\min}H/4\right\}.\]
Taking a union bound over all agents, epochs, and state-action pairs gives
\[
\mathbb P(\mc V^c)
\le
\sum_{i=1}^N
\sum_{k=1}^K
\sum_{(s,a)\in\mc S\times\mc A}
\mathbb P\left(
N_{i,k}(s,a)
<
3\lambda_{\min}H/4
\right)\le
NK|\mc S||\mc A|
\exp\left(
-3\lambda_{\min}H/104
\right).
\]
We further obtain \(\mathbb P(\mc V^c)\leq NT|\mc S||\mc A|\exp\left(-3\lambda_{\min}H/104\right)\), where we used \(K\leq T\). The prescribed lower bound on \(H\) makes the right-hand side at most \(\delta/2\), completing the proof.
\end{proof}
\textbf{Uniform Robust Aggregation Error in Lemma~\ref{lem:robust_aggregation_error}.} We next bound the robust aggregation error uniformly over epochs and state-action pairs. On \(\mc V\), each honest message \(d_{i,k}(s,a)\) concentrates around \((\mc T^*Q_k)(s,a)\), and Algorithm~\ref{alg:trimmed_sc} robustly aggregates these messages. The main subtlety is that the agent-wise visitation counts are random and heterogeneous, which the following analysis handles explicitly.

\textbf{Proof of Lemma~\ref{lem:robust_aggregation_error}.} Fix an epoch \(k\in[K]\) and a state-action pair \((s,a)\in\mc S\times\mc A\). For a deterministic tolerance \(t_1>0\) to be specified shortly, define the deviation event
\begin{equation}
\mc E_k(s,a)
:=
\left\{
\left|
\tilde d_k(s,a)-(\mc T^*Q_k)(s,a)
\right|
>t_1
\right\}.
\end{equation}
We first control this event on the visitation event \(\mc V\). Since \(\mc V\) guarantees that every honest agent visits every state-action pair sufficiently often during every epoch, on the event $\mc{V}$, we have
\(
N_{i,k}(s,a)
\ge
\frac{3}{4}\lambda_{\min}H,
\qquad
\forall i\in\mc G.
\)
For fixed visit counts \(p_1,\ldots,p_N\in\{0,\ldots,H\}\), define the following event: 
\(
\mc J_{k,s,a}(p_1,\ldots,p_N)
:=
\left\{
N_{1,k}(s,a)=p_1,\ldots,
N_{N,k}(s,a)=p_N
\right\}.
\)
The events \(\{\mc J_{k,s,a}(p_1,\ldots,p_N)\}_{p_1,\ldots,p_N=0}^{H}\) form a partition of the sample space according to the number of visits made by each agent to the fixed pair \((s,a)\) during epoch \(k\). Hence, the law of total probability yields the following expansion of \(\mathbb P\left(\mc E_k(s,a)\cap\mc V\right)\):
\begin{equation}
\begin{aligned}
\mathbb P\left(\mc E_k(s,a)\cap \mc V\right)
&=
\sum_{p_1=0}^{H}\cdots\sum_{p_N=0}^{H}
\mathbb P\left(
\mc E_k(s,a)\cap \mc V\cap
\mc J_{k,s,a}(p_1,\ldots,p_N)
\right).
\end{aligned}
\end{equation}
By Lemma~\ref{lem:s1_uniform_visitation}, on the event \(\mc V\), all agents \(i\in[N]\) satisfy \(N_{i,k}(s,a)\ge (3/4)\lambda_{\min}H\). Thus, 
\begin{equation}
\begin{aligned}
&\mathbb P\left(\mc E_k(s,a)\cap \mc V\right)
=
\sum_{p_1=\frac{3}{4}\lambda_{\min}H}^{H}\cdots\sum_{p_N=\frac{3}{4}\lambda_{\min}H}^{H}
\mathbb P\left(
\mc E_k(s,a)\cap \mc V\cap
\mc J_{k,s,a}(p_1,\ldots,p_N)
\right) \\
&\le
\sum_{p_1=\frac{3}{4}\lambda_{\min}H}^{H}\hspace{-2mm}\cdots\hspace{-2mm}\sum_{p_N=\frac{3}{4}\lambda_{\min}H}^{H}
\mathbb P\left(
\mc E_k(s,a)\cap
\mc J_{k,s,a}(p_1,\ldots,p_N)
\right) \\
&=
\sum_{p_1=\frac{3}{4}\lambda_{\min}H}^{H}\hspace{-2mm}\cdots\hspace{-2mm}\sum_{p_N=\frac{3}{4}\lambda_{\min}H}^{H}
\underbrace{\mathbb P\left(
\mc E_k(s,a)
\mid
\mc J_{k,s,a}(p_1,\ldots,p_N)
\right)}_{(\textcolor{winered}{\bullet})}
\mathbb P\left(
\mc J_{k,s,a}(p_1,\ldots,p_N)
\right).
\end{aligned}
\label{eq:agg_condition_on_counts}
\end{equation}
In order to bound the right-hand side of \eqref{eq:agg_condition_on_counts}, we first derive an upper bound on the conditional probability marked \((\textcolor{winered}{\bullet})\) that holds uniformly over all count vectors \((p_1,\ldots,p_N)\) satisfying \(p_i\geq3\lambda_{\min}H/4\) for every \(i\in[N]\). To that end, fix any such count vector \((p_1,\ldots,p_N)\) and condition on \(\mc J_{k,s,a}(p_1,\ldots,p_N)\). We further condition on \(\mc F_{k-1}\), the \(\sigma\)-field generated by the history before epoch \(k\), which fixes \(Q_k\). Under this joint conditioning, each honest message \(d_{i,k}(s,a)\) is an average of exactly \(p_i\) independent Bellman samples with common conditional mean \((\mc T^*Q_k)(s,a)\). To express this average explicitly, for agent \(i \in [N]\), let us index its visits to \((s,a)\) during epoch \(k\) by
\(j\). Let \(r_{j}(s,a)\) and \(s'_{i,j}(s,a)\) denote the reward
and next state observed on the \(j\)-th such visit. We then define the one-sample Bellman target and rewrite \(d_{i,k}(s,a)\) from \eqref{eqn:dik_update} in terms of this target:
\begin{equation}
\begin{aligned}
Y_{i,j,k}(s,a)
&:=
r_j(s,a)
+
\gamma\max_{a'\in\mc A}Q_k(s'_{i,j}(s,a),a'), \\
d_{i,k}(s,a)
&=
\frac{1}{p_i}
\sum_{j=1}^{p_i}
Y_{i,j,k}(s,a).
\end{aligned}
\label{eq:bellman_sample_average}
\end{equation}
 Under the i.i.d. asynchronous sampling model in Section~\ref{sec:backgrounD}, each next-state sample \(s'_{i,j}(s,a)\) is drawn from \(\mc P(\cdot|s,a)\), and
\(\mathbb{E}[r_{j}(s,a)]=R(s,a)\). Therefore,
\begin{equation}
\begin{aligned}
&\mathbb E\left[
Y_{i,j,k}(s,a)
\mid
\mc F_{k-1}
\right] =
\mathbb E\left[
r_{j}(s,a)
+
\gamma \max_{a'\in\mc A} Q_k(s'_{i,j}(s,a),a')
\mid
\mc F_{k-1}
\right] \\
&=
R(s,a)
+
\gamma
\mathbb E_{s'\sim \mc P(\cdot\mid s,a)}
\left[
\max_{a'\in\mc A} Q_k(s',a')
\right] =
(\mc T^*Q_k)(s,a).
\end{aligned}
\end{equation}
Under the i.i.d. asynchronous model, conditioning on \(\mc J_{k,s,a}(p_1,\ldots,p_N)\) fixes only the visitation counts and does not alter the law of the corresponding reward-next-state pairs. Thus, conditional on \(\mc F_{k-1}\) and this event, the \(p_i\) Bellman samples of each honest agent \(i\) are independent with common mean \((\mc T^*Q_k)(s,a)\). Consequently, for \(p_i>0\),
\begin{equation}
\nonumber
\begin{aligned}
\mathbb E\left[
d_{i,k}(s,a)
\mid
\mc F_{k-1},\, \mc J_{k,s,a}(p_1,\ldots,p_N)
\right]
=
(\mc T^*Q_k)(s,a).
\end{aligned}
\end{equation}
By Lemma~\ref{lem:s1_uniform_boundedness}, \(\|Q_k\|_\infty\leq B:=\bar R/(1-\gamma)\), so every Bellman sample has absolute value at most \(B\), and its centered version has absolute value at most \(2B\). Hence, Hoeffding's inequality~\cite{chungconc} implies that, conditional on \(\mc F_{k-1}\) and \(\mc J_{k,s,a}(p_1,\ldots,p_N)\), each honest message is \(2B/\sqrt{p_i}\)-sub-Gaussian. Since \(p_i\geq3\lambda_{\min}H/4\), the clean messages have common mean \((\mc T^*Q_k)(s,a)\) and uniform sub-Gaussian proxy \(\sigma^*\leq4B/\sqrt{\lambda_{\min}H}\). As at most an \(\varepsilon\)-fraction of the messages are corrupted, Lemma~\ref{lem:trim-subg} gives
\begin{equation}\medmath{
\mathbb P\left(
\left.
\left|
\tilde d_k(s,a)-(\mc T^*Q_k)(s,a)
\right|
>
\mc C\sigma^*
\left(
\varepsilon\sqrt{\log(1/\varepsilon)}
+
\sqrt{\frac{\log(4/\delta_1)}{N}}
\right)
\right|
(\textcolor{myblue}{\bullet})
\right)
\le
\delta_1},
\label{eq:preceding}
\end{equation}
where $(\textcolor{myblue}{\bullet}):=\{\mc F_{k-1},\mc J_{k,s,a}(p_1,\ldots,p_N)\}$, and $\mc C$ is the universal constant appearing in Lemma~\ref{lem:trim-subg}. Since $\sigma^*\leq 4B/\sqrt{\lambda_{\min}H}$, we choose the tolerance as follows:
\begin{equation}
\label{eqn:thresh}
t_1
:=
\frac{4\mc C B}{\sqrt{\lambda_{\min}H}}
\left(
\varepsilon\sqrt{\log(1/\varepsilon)}
+
\sqrt{\frac{\log(4/\delta_1)}{N}}
\right) \triangleq \Delta.
\end{equation}
With this choice of tolerance \(t_1\), the preceding conditional bound in \eqref{eq:preceding} implies the following:
\begin{equation}
\mathbb P\left(
\left.
\mc E_k(s,a)
\right|
\mc F_{k-1},\mc J_{k,s,a}(p_1,\ldots,p_N)
\right)
\le
\delta_1
\end{equation}
for every admissible count vector. Therefore, using the tower property of conditional expectation, 
\begin{equation}\label{eqn:final-ad-count}
\begin{aligned}
&\mathbb P\left(\mc E_k(s,a)|\mc J_{k,s,a}(p_1,\ldots,p_N)\right)
=
\mathbb E\left[
\mathbbm{1}_{\mc E_{k}(s,a)}|\mc J_{k,s,a}(p_1,\ldots,p_N)
\right] \\
&= \mathbb E \left[\mathbb E\left[
\mathbbm{1}_{\mc E_{k}(s,a)}|\mc J_{k,s,a}(p_1,\ldots,p_N),\mc{F}_{k-1}
\right]\right] \\
&=\mathbb E \left[\mathbb P\left(
\left.
\mc E_k(s,a)
\right|
\mc J_{k,s,a}(p_1,\ldots,p_N), \mc F_{k-1}
\right)\right] \le \delta_1. 
\end{aligned}
\end{equation}
Combining \eqref{eqn:final-ad-count} with the decomposition in
\eqref{eq:agg_condition_on_counts}, we obtain
the following for every epoch and
state-action pair, as follows:
\begin{equation}
\begin{aligned}
    \mathbb P\left(\mc E_k(s,a)\cap \mc V\right) &\le \delta_1 \hspace{-2mm}\sum_{p_1=\frac{3}{4}\lambda_{\min}H}^{H}\hspace{-2mm}\cdots\hspace{-2mm}\sum_{p_N=\frac{3}{4}\lambda_{\min}H}^{H}\hspace{-2mm}
\mathbb P\left(
\mc J_{k,s,a}(p_1,\ldots,p_N)
\right) \\
& \le \delta_1 \sum_{p_1=0}^{H}\cdots\sum_{p_N=0}^{H}
\mathbb P\left(
\mc J_{k,s,a}(p_1,\ldots,p_N)
\right) \overset{(\bullet)}{=} \delta_1.
\end{aligned}
\end{equation}
Here, \((\bullet)\) follows because
\(\mathbb{P}\bigl(\mc J_{k,s,a}(p_1,\ldots,p_N)\bigr)\) is the \emph{joint
probability mass function} of the visitation-count vector, and its
probabilities sum to one over all possible count vectors. Now define
\(\mc E:=\bigcup_{k\in[K],(s,a)\in\mc S\times\mc A}
\mc E_k(s,a).\) Applying a union bound over these \(K|\mc S||\mc A|\) events and using
\(K\leq T\), we obtain:
\begin{equation}
\label{eqn:final_prob}
\begin{aligned}
\mathbb P\left(
\mc V\bigcap \mc{E}
\right)
&\le
K|\mc S||\mc A|\delta_1 \le
T|\mc S||\mc A|
\cdot
\frac{\delta}{2T|\mc S||\mc A|}
=
\frac{\delta}{2},
\end{aligned}
\end{equation}
with \(\delta_1=\delta/(2T|\mc S||\mc A|)\). Finally, for \(t_1\) chosen as in \eqref{eqn:thresh}, \(\mc E\) is the event
that the aggregation error exceeds \(t_1\) for at least one epoch and
state-action pair. Splitting according to whether the visitation event \(\mc V\) holds, we obtain \(\mathbb P(\mc E)=\mathbb P(\mc E\cap\mc V)+\mathbb P(\mc E\cap\mc V^c)\leq\mathbb P(\mc E\cap\mc V)+\mathbb P(\mc V^c)\).
From the estimate in~\eqref{eqn:final_prob}, we have
\(
\mathbb P(\mc E\cap \mc V)
\le
\frac{\delta}{2}.
\)
Moreover, Lemma~\ref{lem:s1_uniform_visitation} gives \(\mathbb P(\mc V^c)\le \delta/2\). Therefore, \(\mathbb P(\mc E)\le \delta\). Equivalently, with probability at least \(1-\delta\), the desired robust aggregation bound holds simultaneously for all \(k\in[K]\) and all \((s,a)\in\mc S\times\mc A\). 

\textbf{Completion of the proof of Theorem~\ref{theorem:theoremmainresult}.} Recall \(e_k:=\|Q_k-Q^\star\|_\infty\). Combining the one-step error decomposition \eqref{eq:main_error_recursion_prelim} in Appendix~\ref{app:proof_of_upper_bound_S1} with the \(\gamma\)-contraction property \eqref{eqn:Bellmancontraction} and the uniform aggregation bound in Lemma~\ref{lem:robust_aggregation_error}, we obtain the following inequality, which holds simultaneously for all \(k\in[K]\) with probability at least \(1-\delta\):
\begin{equation}
\begin{aligned}
e_{k+1}
&\le
(1-\eta)e_k
+
\eta\|\mc T^*Q_k-\mc T^*Q^\star\|_\infty
+
\eta\|\tilde d_k-\mc T^*Q_k\|_\infty \\
&\le
(1-\eta (1-\gamma))e_k
+
\eta\Delta,
\end{aligned}
\end{equation}
where \(\Delta\) and \(B\) are as follows:
\begin{equation}
\Delta
:=
\frac{4\mc C B}{\sqrt{\lambda_{\min}H}}
\left(
\varepsilon\sqrt{\log(1/\varepsilon)}
+
\sqrt{\frac{\log(4/\delta_1)}{N}}
\right),
B:=\frac{\bar R}{1-\gamma}.
\end{equation}
Unrolling this recursion from \(k=0\) to \(K-1\) gives
\begin{equation}
\begin{aligned}
e_K
&\le
\left(1-\eta(1-\gamma)\right)^K e_0
+
\eta\Delta
\sum_{\ell=0}^{K-1}
\left(1-\eta(1-\gamma)\right)^\ell \\
&\le
\exp\left(-\eta(1-\gamma)K\right)e_0
+
\frac{\Delta}{1-\gamma}.
\end{aligned}
\end{equation}
Since \(Q_0=0\), we have \(e_0=\|Q_0-Q^\star\|_\infty=\|Q^\star\|_\infty\le B\). With the choice \(\eta=\log(NT)/((1-\gamma)K)\), we get \(\exp(-\eta(1-\gamma)K)=\exp(-\log(NT))=1/(NT)\). Therefore, \(e_K\le e_0/(NT)+\Delta/(1-\gamma)\).
Substituting the definition of \(\Delta\), we conclude that
\begin{equation}
\left\|Q_K-Q^\star\right\|_\infty
\le
\frac{e_0}{NT}
+
\frac{4\mc C\bar R}{(1-\gamma)^2\sqrt{\lambda_{\min}H}}
\left(
\varepsilon\sqrt{\log(1/\varepsilon)}
+
\sqrt{\frac{\log(4/\delta_1)}{N}}
\right).
\end{equation}
Using \(H = T/K\), \(K=\left\lceil c_1\log(NT)/(1-\gamma)\right\rceil\), and $\delta_1= \frac{\delta}{2T|\mc S||\mc A|}$, followed by straightforward algebra leads to our final desired bound:
\begin{equation}
e_K
\le
\frac{e_0}{NT}
+
\mathcal{O}\left(
\frac{
\bar R
\sqrt{\log(NT)\log\left(|\mathcal S||\mathcal A|T/\delta\right)}
}{
(1-\gamma)^{5/2}\sqrt{\lambda_{\min}NT}
}
\right) +
\mathcal{O}\left(
\frac{
\bar R\varepsilon\sqrt{\log(1/\varepsilon)}\sqrt{\log(NT)}
}{
(1-\gamma)^{5/2}\sqrt{\lambda_{\min}T}
}
\right).
\end{equation}
This completes the proof of Theorem~\ref{theorem:theoremmainresult}.

\section{Proof of Theorem~\ref{theorem:theoremmainresult_S2}: Partial Coverage Setting}
\label{app:proof_of_upper_bound_S2}
We now prove Theorem~\ref{theorem:theoremmainresult_S2}, which corresponds to the partial coverage setting described as \((\textcolor{winered}{\tt S_2})\). Unlike the full-coverage setting, an individual agent need not visit every state-action pair with positive probability. Instead, for each pair \((s,a) \in \mc{S}\times\mc{A}\), only agents in \(\mc I(s,a)\) contribute samples, and the relevant coverage parameter is \(\Lambda_{\min}:=\min_{(s,a)\in\mc S\times\mc A}\min_{i\in\mc I(s,a)}\lambda_i(s,a)\), the worst-case visitation probability among admissible agents.
\begin{algorithm}
\caption{Robust Asynchronous Federated \(Q\)-Learning under Partial Coverage \((\textcolor{winered}{\tt S_2})\)}
\label{alg:async-fedq-S2}
\begin{algorithmic}[1]
\Require Total samples \(T\), number of epochs \(K\), stepsize \(\eta\), discount factor $\gamma$, corruption fraction \(\varepsilon\), confidence parameter \(\delta\), behavior policies \(\{\mu_i\}_{i \in [N]}\), partial coverage sets \(\{\mc I(s,a)\}_{(s,a)\in\mc S\times\mc A}\).
\State Initialize \(Q_0(s,a)\gets 0,~d_{i,0}(s,a) \gets 0\) for all \(i \in [N]\), and \((s,a)\in\mc S\times\mc A\).
\For{epoch \(k=0,\dots,K-1\)}
    \State Server broadcasts \(Q_k\) to all agents.
    \For{each good agent \(i\in[N] \setminus \mc{B} \) in parallel}
        \State Agent \(i\) follows \(\mu_i\) for \(H\) steps and collects data tuples of the form $(s_{i,t}, a_{i,t}, {r}_{i,t}, s_{i,t+1})$.
        \State For each $(s,a) \in \mc{S} \times \mc{A}$ with \(i\in\mc I(s,a)\), if $N_{i,k,H}(s,a)=0$, agent $i$ sets $d_{i,k}(s,a)=0$; else, it computes $d_{i,k}(s,a)$ via~\eqref{eqn:dik_update}. It then sends \(\{d_{i,k}(s,a)\}_{(s,a):\, i\in\mc I(s,a)}\) to the server.
    \EndFor
    \State For each \((s,a)\in\mc S\times\mc A\), the server computes \(\widetilde d_k(s,a)\) via \eqref{eqn:server_update} using the messages from agents \(i\in\mc I(s,a)\), and \(Q_{k+1}(s,a)\) via  \eqref{eqn:server_Q_update}.
\EndFor
\State \Return \(Q_K\).
\end{algorithmic}
\end{algorithm}
The proof therefore follows the same structure as that of Theorem~\ref{theorem:theoremmainresult}, with \(\lambda_{\min}\) replaced by \(\Lambda_{\min}\). The only additional step is to show that, on the good-visit event, each state-action pair receives sufficiently many honest samples in aggregate from the agents that can observe it, after which the robust aggregation and Bellman perturbation arguments proceed unchanged.

\begin{lemma}\textbf{(Uniform boundedness of iterates under partial coverage)}
\label{lem:s2_uniform_boundedness}
Suppose \(|R(s,a)|\le \bar R\) for every \((s,a)\in\mc S\times\mc A\), and suppose the stepsize satisfies \(\eta\in[0,1]\). Then the iterates generated by Algorithm~\ref{alg:async-fedq-S2} satisfy
\begin{equation}
\label{eq:s2_q_uniform_bound}
\|Q_k\|_\infty
\le B:=
\frac{\bar R}{1-\gamma},
\qquad
\forall k\ge 0.
\end{equation}
\end{lemma}
\begin{proof}
The proof is identical to the full-coverage case in \((\textcolor{winered}{\tt S_1})\), with the only difference being that the aggregation for each \((s,a)\) is performed over the agents in \(\mc I(s,a)\), rather than over all agents. We prove the claim by induction. The base case holds since \(Q_0(s,a)=0\) for all \((s,a)\), and hence \(\|Q_0\|_\infty\le B\). Now assume \(\|Q_k\|_\infty\le B\). Fix any \((s,a)\in\mc S\times\mc A\) and any honest agent \(i\in\mc I(s,a)\). If \(N_{i,k,H}(s,a)>0\), then, by the definition of \(d_{i,k}(s,a)\) in \eqref{eqn:dik_update}, and following the boundedness argument established in Lemma~\ref{lem:s1_uniform_boundedness}, we have \(|d_{i,k}(s,a)|\le B\).

On a similar note, when \(N_{i,k,H}(s,a)=0\), then the algorithm sets \(d_{i,k}(s,a)=0\), so \(|d_{i,k}(s,a)|\le B\) trivially. Therefore, every honest message used for the aggregation of \((s,a)\) lies in \([-B,B]\). Since the server uses the same trimmed/clipped aggregation rule as in Lemma~\ref{lem:s1_uniform_boundedness}, applied here only over the relevant set \(\mc I(s,a)\), the same argument gives \(|\widetilde d_k(s,a)|\le B\). Thus, since \(Q_{k+1}(s,a)\) is a convex combination of two quantities bounded by \(B\), it is also bounded by \(B\), completing the induction.
\end{proof}
\begin{lemma}\textbf{(Uniform visitation event under partial coverage)}
\label{lem:s2_uniform_visitation}
Fix a confidence parameter \(\delta\in(0,1)\). Suppose the epoch length satisfies \(
H \ge
(104/3\Lambda_{\min})
\log\left(2NT|\mc S||\mc A|/\delta\right).\) Then the following event holds with probability at least \(1-\delta/2\):
\begin{equation}
\label{eq:s2_good_visit_event}
\mc V
:=
\left\{
N_{i,k,H}(s,a)
\ge
\frac{3}{4}\Lambda_{\min}H,
\forall i\in\mc I(s,a),\ \forall k\in[K],\ \forall(s,a)\in\mc S\times\mc A
\right\}.
\end{equation}
\end{lemma}
\begin{proof}
Fix a state-action pair \((s,a)\in\mc S\times\mc A\), an agent \(i\in\mc I(s,a)\), and an epoch \(k\in[K]\). Similar to the proof of Lemma~\ref{lem:s1_uniform_visitation}, the number of visits of agent \(i\) to \((s,a)\) during epoch \(k\) is given by
\begin{equation}
N_{i,k,H}(s,a)
=
\sum_{t \in \mc{I}_k}
X_t^{i,k}(s,a), \hspace{1mm}
X_t^{i,k}(s,a)
:=
\mathbbm{1}\{(s_{i,t},a_{i,t})=(s,a)\},  \hspace{1mm} t \in \mc{I}_k.
\end{equation}
Under the partial coverage sampling model in \((\textcolor{winered}{\tt S_2})\), for every \(i\in\mc I(s,a)\), \(X_t^{i,k}(s,a)\) is a Bernoulli random variable with mean \(\lambda_i(s,a)\), where \(\lambda_i(s,a)\) is the sampling probability of \((s,a)\) under the behavior policy \(\mu_i\) of agent \(i\). Hence,
\begin{equation}
\mathbb E[N_{i,k,H}(s,a)]
=
\sum_{t \in \mc{I}_k} \mathbb E[X_t^{i,k}(s,a)]
=
\lambda_i(s,a)H
\ge
\Lambda_{\min}H,
\end{equation}
where the last inequality follows from the definition
\(
\Lambda_{\min}
:=
\min_{(s,a)\in\mc S\times\mc A}
\min_{i\in\mc I(s,a)}
\lambda_i(s,a)
\). We now apply Bernstein's inequality to lower bound \(N_{i,k,H}(s,a)\). Since \(X_t^{i,k}(s,a)\in \{0,1\}\), we have
\begin{equation}
\left|X_t^{i,k}(s,a)-\lambda_i(s,a)\right|\le 1,
\end{equation}
and
\begin{equation}
\sum_{t \in \mc{I}_k}
\mathrm{Var}\left(X_t^{i,k}(s,a)\right)
=
H\lambda_i(s,a)(1-\lambda_i(s,a))
\le
H\lambda_i(s,a).
\end{equation}
Applying Bernstein's inequality from Lemma~\ref{lemma:bernstein} to the centered variables \(X_t^{i,k}(s,a)-\lambda_i(s,a)\), with deviation \(\lambda_i(s,a)H/4\), gives the following, exactly as in Lemma~\ref{lem:s1_uniform_visitation},
\begin{equation}
\begin{aligned}
\mathbb P\left(
N_{i,k,H}(s,a)
<
\frac{3}{4}\lambda_i(s,a)H
\right)
&=
\mathbb P\left(
\sum_{t=1}^H
\left(X_t^{i,k}(s,a)-\lambda_i(s,a)\right)
<
-\frac{1}{4}\lambda_i(s,a)H
\right) \\
&\le
\exp\left(
-
\frac{\left(\lambda_i(s,a)H/4\right)^2}
{2H\lambda_i(s,a)+(2/3)\lambda_i(s,a)H/4}
\right) \\
&=
\exp\left(
-\frac{3}{104}\lambda_i(s,a)H
\right)
\le
\exp\left(
-\frac{3}{104}\Lambda_{\min}H
\right).
\end{aligned}
\end{equation}
Since \(\lambda_i(s,a)\ge \Lambda_{\min}\), the event \(N_{i,k,H}(s,a)\ge (3/4)\lambda_i(s,a)H\) immediately implies \(N_{i,k,H}(s,a)\ge (3/4)\Lambda_{\min}H\). Consequently, the probability that \(N_{i,k,H}(s,a)\) falls below \((3/4)\Lambda_{\min}H\) is bounded by the probability that it falls below \((3/4)\lambda_i(s,a)H\), and hence
\begin{equation}
\mathbb P\left(
N_{i,k,H}(s,a)
<
\frac{3}{4}\Lambda_{\min}H
\right)
\le
\exp\left(
-\frac{3}{104}\Lambda_{\min}H
\right).
\end{equation}
Now define the bad event
\begin{equation}
\mc V^c
:=
\left\{
\exists (s,a)\in\mc S\times\mc A,\ \exists i\in\mc I(s,a),\ \exists k\in[K]
\text{ such that }
N_{i,k,H}(s,a)
<
\frac{3}{4}\Lambda_{\min}H
\right\}.
\end{equation}
The rest of the proof follows the proof of Lemma~\ref{lem:s1_uniform_visitation} for the full-coverage setting \((\textcolor{winered}{\tt S_1})\).
\end{proof}
\newpage
\begin{lemma}\textbf{(Robust Aggregation Error in \((\textcolor{winered}{\texttt{S}_2})\) setting)}
\label{lem:s2_robust_aggregation_error}
Suppose the conditions of Theorem~\ref{theorem:theoremmainresult_S2} hold, and let
\(\underline{N}:=\min_{(s,a)\in\mc S\times\mc A}|\mc I(s,a)|\). Then, with probability at least \(1-\delta\), the following bound holds simultaneously for all epochs $k \in [K]$:
\begin{equation}
\label{eqn:s2_robust_aggregation_error}
\left\lVert \widetilde d_k-\mc T^*Q_k \right\rVert_{\infty}
\le
\mathcal{O}\left(
\frac{\bar R \sqrt{\log\left(|\mc S||\mc A|T/\delta\right)}}{(1-\gamma)\sqrt{\Lambda_{\min}\underline{N}H}}
\right)
+
\mathcal{O}\left(
\frac{\bar R\,\varepsilon\sqrt{\log(1/\varepsilon)}}{(1-\gamma)\sqrt{\Lambda_{\min}H}}
\right).
\end{equation}
\end{lemma}
\begin{proof}
The proof follows the same argument as Lemma~\ref{lem:robust_aggregation_error}; we only indicate the changes needed for the partial coverage setting. Fix an epoch \(k\in[K]\) and a state-action pair \((s,a)\in\mc S\times\mc A\). Let
\begin{equation}
N_{s,a}:=|\mc I(s,a)|,
\qquad
\underline N:=\min_{(s,a)\in\mc S\times\mc A}|\mc I(s,a)|.
\end{equation}
For a deterministic tolerance \(t_2>0\), define
\begin{equation}
\mc E_k(s,a)
:=
\left\{
\left|
\widetilde d_k(s,a)-(\mc T^*Q_k)(s,a)
\right|
>t_2
\right\}.
\end{equation}
We first control this event on the visitation event \(\mc V\). By Lemma~\ref{lem:s2_uniform_visitation}, on \(\mc V\), for every \(i\in\mc I(s,a)\),
\begin{equation}
N_{i,k,H}(s,a)
\ge
\frac{3}{4}\Lambda_{\min}H.
\end{equation}
Condition on the visit counts \(\{N_{i,k,H}(s,a)=p_i\}_{i\in\mc I(s,a)}\), where each admissible \(p_i\) satisfies \(p_i\ge (3/4)\Lambda_{\min}H\). Conditional on these counts and on the sigma-field \(\mc F_{k-1}\), the iterate \(Q_k\) is fixed, and each honest message \(d_{i,k}(s,a)\), \(i\in\mc I(s,a)\setminus\mc B\), is an empirical average of \(p_i\) Bellman samples with common mean \((\mc T^*Q_k)(s,a)\). By Lemma~\ref{lem:s2_uniform_boundedness}, \(\|Q_k\|_\infty\le B=\bar R/(1-\gamma)\), so each Bellman sample is bounded in absolute value by \(B\). Hence, by Hoeffding's lemma, each centered honest message \(d_{i,k}(s,a)-(\mc T^*Q_k)(s,a)\) is sub-Gaussian with proxy at most \(2B/\sqrt{p_i}\le 2B/\sqrt{(3/4)\Lambda_{\min}H}\le 4B/\sqrt{\Lambda_{\min}H}\). Therefore, conditional on the admissible count vector, the honest messages in the pair-specific collection \(\{d_{i,k}(s,a):i\in\mc I(s,a)\}\) have common mean \((\mc T^*Q_k)(s,a)\) and are sub-Gaussian with uniform proxy \(\sigma^*\le 4B/\sqrt{\Lambda_{\min}H}\). Since the corrupted fraction inside \(\mc I(s,a)\) is at most \(\varepsilon\), applying Lemma~\ref{lem:trim-subg} to this collection gives, conditionally on the admissible count vector,
\begin{equation}\medmath{
\mathbb P\left(
\left.
\left|
\widetilde d_k(s,a)-(\mc T^*Q_k)(s,a)
\right|
>
\mc C\sigma^*
\left(
\varepsilon\sqrt{\log(1/\varepsilon)}
+
\sqrt{\frac{\log(4/\delta_1)}{N_{s,a}}}
\right)
\right|
\mc F_{k-1}, \mc J_{k,s,a}\bigl(\left[p_i\right]_{i\in\mc I(s,a)}\bigr)
\right)
\le
\delta_1.}
\end{equation}
Since \(N_{s,a}\ge \underline N\), it is enough to choose
\begin{equation}
t_2
:=
\frac{4\mc C B}{\sqrt{\Lambda_{\min}H}}
\left(
\varepsilon\sqrt{\log(1/\varepsilon)}
+
\sqrt{\frac{\log(4/\delta_1)}{\underline N}}
\right).
\end{equation}
By the same conditioning and tower-property argument used in proving Lemma~\ref{lem:robust_aggregation_error}, the following holds true:
\begin{equation}
\mathbb P\left(\mc E_k(s,a)\cap\mc V\right)
\le
\delta_1.
\end{equation}
Taking a union bound over all \(k\in[K]\) and all \((s,a)\in\mc S\times\mc A\), and choosing
\[
\delta_1:=\frac{\delta}{2T|\mc S||\mc A|},
\]
we get, using \(K\le T\),
\begin{equation}
\mathbb P\left(
\exists k\in[K],\ \exists(s,a)\in\mc S\times\mc A
\text{ such that }
\mc E_k(s,a)\cap\mc V
\right)
\le
\frac{\delta}{2}.
\end{equation}
Moreover, Lemma~\ref{lem:s2_uniform_visitation} gives \(\mathbb P(\mc V^c)\le \delta/2\). Therefore, with probability at least \(1-\delta\), for all \(k\in[K]\) and all \((s,a)\in\mc S\times\mc A\),
\begin{equation}
\left|
\widetilde d_k(s,a)-(\mc T^*Q_k)(s,a)
\right|
\le
\frac{4\mc C B}{\sqrt{\Lambda_{\min}H}}
\left(
\varepsilon\sqrt{\log(1/\varepsilon)}
+
\sqrt{\frac{\log(8T|\mc S||\mc A|/\delta)}{\underline N}}
\right).
\end{equation}
Since \(B=\bar R/(1-\gamma)\), taking the maximum over \((s,a)\) yields
\begin{equation}\label{eq:final_S2}
\left\lVert \widetilde d_k-\mc T^*Q_k \right\rVert_{\infty}
\le
\mathcal{O}\left(
\frac{\bar R \sqrt{\log\left(|\mc S||\mc A|T/\delta\right)}}{(1-\gamma)\sqrt{\Lambda_{\min}\underline N H}}
\right)
+
\mathcal{O}\left(
\frac{\bar R\,\varepsilon\sqrt{\log(1/\varepsilon)}}{(1-\gamma)\sqrt{\Lambda_{\min}H}}
\right),
\end{equation}
simultaneously for all \(k\in[K]\). This proves the lemma.
\end{proof}
\paragraph{Proof of Theorem \ref{theorem:theoremmainresult_S2}.} Equation~\eqref{eq:final_S2} establishes the high-probability robust aggregation error bound stated in Lemma~\ref{lem:s2_robust_aggregation_error}. The rest of the proof, which derives the final finite-time bound for the \((\textcolor{winered}{\texttt{S}_2})\) setting stated in Theorem~\ref{theorem:theoremmainresult_S2}, follows the proof for the \((\textcolor{winered}{\texttt{S}_1})\) setting in Appendix~\ref{app:proof_of_upper_bound_S1}.

\section{Proof of Theorem \ref{thm:async_fedq_lower_bound}: Fundamental Lower Bound}
\label{app:proof_of_lower_bound}
We prove Theorem~\ref{thm:async_fedq_lower_bound} via a two-point indistinguishability argument. We construct two one-state, one-action MDPs with identical self-loop transitions but different reward means. The reward distributions are chosen so that adversarial corruption makes the server observations identical under both instances while their optimal \(Q\)-values remain separated. Consequently, any estimator must incur a nontrivial error on at least one instance. Before presenting the construction, we introduce an inequality central to the proof of Theorem~\ref{thm:async_fedq_lower_bound}.

\begin{lemma}(\textbf{Bretagnolle--Huber Inequality}~\cite[Theorem~14.2]{lattimore2020bandit})
\label{lem:BH}
Let $P$ and $Q$ be two probability measures on the same measurable space $(\Omega,\mathcal{F})$,
and let $\mc{A} \in \mathcal{F}$ be any arbitrary event. Then,
\begin{equation}
P(\mc{A}) + Q(\mc{A}^c) \ge \frac{1}{2}\exp\!\big(-\texttt{KL}(P\,\|\,Q)\big),
\end{equation}
where $\mc{A}^c$ is the complement of the event $\mc{A}$, and \texttt{KL} is the Kullback--Leibler distance between $P$ and $Q$.
\end{lemma}
\textcolor{winered}{\(\blacksquare\)}~{\textbf{Construction of two instances.}}
Consider the Gaussian reward distributions
\begin{equation}
\mc P_1=\mathcal N(\mu_1,\bar R^2),
\quad
\mc P_2=\mathcal N(\mu_2,\bar R^2),
\end{equation}
where \(\mu_1\) and \(\mu_2\) will be chosen to yield separated optimal \(Q\)-values while remaining indistinguishable under an \(\varepsilon\)-fraction of adversarial agents. For \(j\in\{1,2\}\), let \(\mc P_j^{\otimes T}\) denote the joint law of \(T\) independent reward samples drawn from \(\mc P_j\), where \(j \in \{1,2\}\). Equivalently,
\begin{equation}
\mc P_1^{\otimes T}
=
\mathcal N(\mu_1\mathbf 1_T,\bar R^2 I_T),
\quad
\mc P_2^{\otimes T}
=
\mathcal N(\mu_2\mathbf 1_T,\bar R^2 I_T),
\end{equation}
where \(\mathbf 1_T\in\mathbb R^T\) denotes the all-ones vector of size $T$. Let \(p_1\) and \(p_2\) denote the densities of \(\mc P_1^{\otimes T}\) and \(\mc P_2^{\otimes T}\), respectively. Since \(\varepsilon\in(0,1/2)\), we have \(\varepsilon/(1-\varepsilon)<1\). We choose \(\mu_1\) and \(\mu_2\), with \(\mu_2>\mu_1\), such that
\begin{equation}
\TV(\mc P_1^{\otimes T},\mc P_2^{\otimes T})
=
\frac{1}{2}\int_{\mathbb R^T}|p_1(x)-p_2(x)|\,dx
=
\frac{\varepsilon}{1-\varepsilon}.
\label{eq:tv_choice_async_fedq_lower}
\end{equation}
Such a choice is possible because the total variation distance between the two Gaussian product measures is a continuous function of \(|\mu_2-\mu_1|\), equals \(0\) when \(\mu_1=\mu_2\), and tends to \(1\) as \(|\mu_2-\mu_1|\to\infty\). Thus, by the intermediate value theorem, we can tune the mean separation \(|\mu_2-\mu_1|\) so that \(\texttt{TV}(\mc{P}_1^{\otimes T},\mc{P}_2^{\otimes T})\) equals the value required by the contamination construction, where \(\mc{P}_1^{\otimes T}\) is the law of the \(T\)-dimensional reward vector observed by one honest agent under \(M_j\).

We now lower bound the separation between the two means. By Pinsker's inequality,
\(
\TV(\mc P_1^{\otimes T},\mc P_2^{\otimes T})
\le
\sqrt{\frac{1}{2}\KL(\mc P_1^{\otimes T}\|\mc P_2^{\otimes T})}.
\)
Since \(\mc P_1^{\otimes T}\) and \(\mc P_2^{\otimes T}\) are Gaussian product measures with common covariance matrix \(\bar R^2 I_T\), their KL divergence is
\(
\KL(\mc P_1^{\otimes T}\|\mc P_2^{\otimes T})
=
\frac{1}{2\bar R^2}
\left\|
\mu_1\mathbf 1_T-\mu_2\mathbf 1_T
\right\|_2^2
= T(\mu_1-\mu_2)^2/2\bar R^2.
\)

Combining Pinsker's inequality and the KL-divergence expression for
\(\mc P_1^{\otimes T}\) and \(\mc P_2^{\otimes T}\) with
\eqref{eq:tv_choice_async_fedq_lower}, we obtain
\begin{equation}\label{eq:mean_gap_async_fedq_lower-1}
\varepsilon/(1-\varepsilon)
=
\TV(\mc P_1^{\otimes T},\mc P_2^{\otimes T})
\le
\frac{\sqrt T|\mu_2-\mu_1|}{2\bar R}.
\end{equation}
Rearranging \eqref{eq:mean_gap_async_fedq_lower-1} yields the following:
\begin{equation}
|\mu_2-\mu_1|
\ge
\frac{2\varepsilon \bar R}{(1-\varepsilon)\sqrt T}
\ge
\frac{2\varepsilon \bar R}{\sqrt T}.
\label{eq:mean_gap_async_fedq_lower}
\end{equation}
Thus, the reward means are separated by order \(\varepsilon\bar R/\sqrt T\), which will induce a corresponding \(Q^\star\)-gap after embedding the reward laws into discounted MDPs.

Let \(\mc M\) and \(\mc M'\) be two MDPs with the same single state,
single action, and deterministic self-transition, but different reward
distributions. We define them precisely as follows:
\begin{equation}
\mc S=\{s^\star\},
\quad
\mc A=\{a^\star\},
\quad
\mathsf P(s^\star\mid s^\star,a^\star)=1.
\end{equation}
The reward distributions associated with \(\mc M\) and \(\mc M'\) are:
\begin{equation}\label{eq:reward_dif}
\mathcal M:\ r(s^\star,a^\star)\sim\mc P_1=\mathcal N(\mu_1,\bar R^2),
\quad
\mathcal M':\ r(s^\star,a^\star)\sim\mc P_2=\mathcal N(\mu_2,\bar R^2).
\end{equation}
In this construction, each agent acquires $T$ reward samples over a time horizon of length $T$. The resulting reward vector constitutes an agent's message which the server is allowed to observe directly.  Under \(\mathcal M\), an honest agent's reward vector
\((r_{i,1},\ldots,r_{i,T})\) consists of \(T\) independent samples from
\(\mc P_1\), and therefore has joint distribution \(\mc P_1^{\otimes T}\). Under \(\mathcal M'\), the same holds with \(\mc P_1\) replaced by \(\mc P_2\). The separation between the optimal \(Q\)-functions of \(\mathcal M\) and \(\mathcal M'\) follows directly from the Bellman equation: under \(\mathcal M\),
\(
Q_{\mathcal M}^\star(s^\star,a^\star)
=
\mathbb E_{\mathcal M}[r(s^\star,a^\star)]
+
\gamma Q_{\mathcal M}^\star(s^\star,a^\star)
=
\mu_1+\gamma Q_{\mathcal M}^\star(s^\star,a^\star).
\)
The same argument applies under \(\mathcal M'\). Solving the scalar Bellman fixed-point equations for \(\mathcal M\) and \(\mathcal M'\) yields the following:
\begin{equation}\label{eq:fp_identity}
Q_{\mathcal M}^\star(s^\star,a^\star)
=
\frac{\mu_1}{1-\gamma},
\quad
Q_{\mathcal M'}^\star(s^\star,a^\star)
=
\frac{\mu_2}{1-\gamma}.
\end{equation}
Combining the two fixed-point identities from \eqref{eq:fp_identity}, we obtain the following exact separation between the optimal \(Q\)-functions:
\begin{equation}
\left\|
Q_{\mathcal M}^\star-Q_{\mathcal M'}^\star
\right\|_\infty
=
\left|
Q_{\mathcal M}^\star(s^\star,a^\star)
-
Q_{\mathcal M'}^\star(s^\star,a^\star)
\right|
=
\frac{|\mu_2-\mu_1|}{1-\gamma}.
\end{equation}
Using \eqref{eq:mean_gap_async_fedq_lower}, we obtain
\begin{equation}
\left\|
Q_{\mathcal M}^\star-Q_{\mathcal M'}^\star
\right\|_\infty
\ge
\frac{2\varepsilon \bar R}{(1-\gamma)\sqrt T}.
\label{eq:q_gap_async_fedq_lower}
\end{equation}
Hence, the two MDP instances have optimal \(Q\)-functions separated by order \(\varepsilon \bar R/((1-\gamma)\sqrt T)\). The next step is to show that, despite this separation, the adversaries can choose corrupted reward distributions such that the observations at the server have identical distributions under \(\mathcal M\) and \(\mathcal M'\). 

\textcolor{winered}{\(\blacksquare\)}~{\textbf{Attack strategy and indistinguishability of corrupted transcripts.}} Recall that the server's overall observation consists of length-$T$ reward vectors received from each of the $N$ agents. 
The idea now is for the adversary to make the server observation identically distributed under \(\mathcal M\) and \(\mathcal M'\). Since the instances differ only in their rewards, it suffices to hide the distinction between \(\mc P_1^{\otimes T}\) and \(\mc P_2^{\otimes T}\). Let \(p_1\) and \(p_2\) denote their respective densities, and define the adversarial densities \(q_1,q_2\) on \(\mathbb R^T\) by
\begin{equation}
q_1(x)
=
\frac{1-\varepsilon}{\varepsilon}
\bigl(p_2(x)-p_1(x)\bigr)\mathbf 1\{p_2(x)\ge p_1(x)\},
\end{equation}
and
\begin{equation}
q_2(x)
=
\frac{1-\varepsilon}{\varepsilon}
\bigl(p_1(x)-p_2(x)\bigr)\mathbf 1\{p_1(x)\ge p_2(x)\}.
\end{equation}
These are valid probability densities. Indeed, by the definition of total variation distance,
\[
\int_{\mathbb R^T}
\bigl(p_2(x)-p_1(x)\bigr)\mathbf 1\{p_2(x)\ge p_1(x)\}\,dx
=
\TV(\mc P_1^{\otimes T},\mc P_2^{\otimes T}),
\]
and similarly,
\[
\int_{\mathbb R^T}
\bigl(p_1(x)-p_2(x)\bigr)\mathbf 1\{p_1(x)\ge p_2(x)\}\,dx
=
\TV(\mc P_1^{\otimes T},\mc P_2^{\otimes T}).
\]
Since \(\mu_1,\mu_2\) were chosen so that
\[
\TV(\mc P_1^{\otimes T},\mc P_2^{\otimes T})
=
\frac{\varepsilon}{1-\varepsilon},
\]
we obtain
\begin{equation}
\int_{\mathbb R^T}q_1(x)\,dx
=
\frac{1-\varepsilon}{\varepsilon}
\TV(\mc P_1^{\otimes T},\mc P_2^{\otimes T})
= \int_{\mathbb R^T}q_2(x)\,dx
=
1.
\end{equation}
Let \(\mc Q_1\) and \(\mc Q_2\) denote the probability distributions with densities \(q_1\) and \(q_2\), respectively.

The adversary acts independently across agents, corrupting each with probability \(\varepsilon\). Under \(\mathcal M\), an honest agent reports its \(T\) reward observations drawn independently from \(\mc P_1\) and hence jointly distributed as \(\mc P_1^{\otimes T}\). A corrupted agent instead reports a fabricated sequence of \(T\) reward values drawn jointly from \(\mc Q_1\). Under \(\mathcal M'\), the corresponding distributions are \(\mc P_2^{\otimes T}\) and \(\mc Q_2\), respectively. Consequently, the distribution of a single agent's message is \(\mc D_{\mathcal M}=(1-\varepsilon)\mc P_1^{\otimes T}+\varepsilon\mc Q_1\) under \(\mathcal M\), whereas it is \(\mc D_{\mathcal M'}=(1-\varepsilon)\mc P_2^{\otimes T}+\varepsilon\mc Q_2\) under \(\mathcal M'\).

We claim that, by construction, the two one-agent message distributions \(\mc D_{\mathcal M}\) and \(\mc D_{\mathcal M'}\) are identical. To see this, compare their densities pointwise. If \(p_2(x)\ge p_1(x)\), then \(q_2(x)=0\), and hence
\[
(1-\varepsilon)p_1(x)+\varepsilon q_1(x)
=
(1-\varepsilon)p_1(x)
+
(1-\varepsilon)(p_2(x)-p_1(x))
=
(1-\varepsilon)p_2(x)
=
(1-\varepsilon)p_2(x)+\varepsilon q_2(x).
\]

On the other hand, if \(p_1(x)\ge p_2(x)\), then \(q_1(x)=0\), and hence
\[
(1-\varepsilon)p_1(x)+\varepsilon q_1(x)
=
(1-\varepsilon)p_1(x)
=
(1-\varepsilon)p_2(x)
+
(1-\varepsilon)(p_1(x)-p_2(x))
=
(1-\varepsilon)p_2(x)+\varepsilon q_2(x).
\]
Therefore,
\[
(1-\varepsilon)\mc P_1^{\otimes T}
+
\varepsilon \mc Q_1
=
(1-\varepsilon)\mc P_2^{\otimes T}
+
\varepsilon \mc Q_2,
\]
or equivalently,
\[
\mc D_{\mathcal M}
=
\mc D_{\mathcal M'}.
\]

\textcolor{winered}{\(\blacksquare\)}~{\textbf{Reduction from indistinguishability to estimation error.}}
Let \(Z=(Z_1,\ldots,Z_N)\), where \(Z_i\in\mathbb R^T\) denotes the sequence of \(T\) reward values reported by agent \(i\). If agent \(i\) is honest, then \(Z_i=(r_{i,1},\ldots,r_{i,T})\); otherwise, \(Z_i\) is a fabricated reward sequence drawn from \(\mc Q_1\) under \(\mathcal M\) or \(\mc Q_2\) under \(\mathcal M'\). Since the attacks act independently across agents, the \(Z_i\)'s are independent, and the law of the observation at the server is a product measure: we call this law \(
\mathbb P_{\mathcal M}
=
\mc D_{\mathcal M}^{\otimes N} \) for instance $\mc{M}$ and  \(
\mathbb P_{\mathcal M'}
=
\mc D_{\mathcal M'}^{\otimes N}
\) for instance $\mc{M}'$. Using this, the fact that the one-agent laws satisfy
\(
\mc D_{\mathcal M}
=
(1-\varepsilon)\mc P_1^{\otimes T}
+
\varepsilon \mc Q_1
=
(1-\varepsilon)\mc P_2^{\otimes T}
+
\varepsilon \mc Q_2
=
\mc D_{\mathcal M'}  
\), and the chain rule for relative entropies, we obtain:
\begin{equation}
\label{eqn:zero_KL}
\KL(\mathbb P_{\mathcal M}\|\mathbb P_{\mathcal M'})
=
N \cdot \KL(\mc D_{\mathcal M}\|\mc D_{\mathcal M'})
=
0.
\end{equation}

We now use the above indistinguishability of the observed message laws to prove the estimation lower bound. To that end, let \(\widehat Q\) be any estimator constructed by the server from the corrupted observed message \(Z\). Define the separation
\(
\Delta
=
\left\|
Q_{\mathcal M}^\star
-
Q_{\mathcal M'}^\star
\right\|_\infty.
\)
Now, we already have
\(
\Delta
\ge
\frac{2\varepsilon \bar R}{(1-\gamma)\sqrt T}.
\label{eq:delta_lower_bound}
\)
Consider the event
\(
A
=
\left\{
\left\|
\widehat Q
-
Q_{\mathcal M}^\star
\right\|_\infty
<
\frac{\Delta}{2}
\right\}.
\)
Thus, 
\begin{equation}
\mathbb P_{\mathcal M}
\left(
\left\|
\widehat Q
-
Q_{\mathcal M}^\star
\right\|_\infty
\ge
\frac{\Delta}{2}
\right)
= 
\mathbb P_{\mathcal M}(A^c).
\label{eq:error_m}
\end{equation}
On the event \(A\), the estimator is within distance \(\Delta/2\) of \(Q_{\mathcal M}^\star\). Since \(Q_{\mathcal M}^\star\) and \(Q_{\mathcal M'}^\star\) are separated by \(\Delta\), the triangle inequality gives
\(
\left\|
\widehat Q
-
Q_{\mathcal M'}^\star
\right\|_\infty
\ge
\left\|
Q_{\mathcal M}^\star
-
Q_{\mathcal M'}^\star
\right\|_\infty
-
\left\|
\widehat Q
-
Q_{\mathcal M}^\star
\right\|_\infty
>
\frac{\Delta}{2}.
\)
Therefore,
\begin{equation}
\mathbb P_{\mathcal M'}
\left(
\left\|
\widehat Q
-
Q_{\mathcal M'}^\star
\right\|_\infty
\ge
\frac{\Delta}{2}
\right)
\ge
\mathbb P_{\mathcal M'}(A).
\label{eq:error_mprime}
\end{equation}
Combining \eqref{eq:error_m} and \eqref{eq:error_mprime}, we obtain
\begin{equation}
\begin{aligned}
&\max\left\{
\mathbb P_{\mathcal M}
\left(
\left\|
\widehat Q
-
Q_{\mathcal M}^\star
\right\|_\infty
\ge
\frac{\Delta}{2}
\right),
\mathbb P_{\mathcal M'}
\left(
\left\|
\widehat Q
-
Q_{\mathcal M'}^\star
\right\|_\infty
\ge
\frac{\Delta}{2}
\right)
\right\}\\
&\ge
\frac{1}{2}
\left(
\mathbb P_{\mathcal M}(A^c)
+
\mathbb P_{\mathcal M'}(A)
\right) \overset{(*)}{\ge} \frac{1}{4},
\end{aligned}
\end{equation}
where for \((*)\), we applied the Bretagnolle--Huber inequality (Lemma~\ref{lem:BH}). Specifically, applied to the two observation laws \(\mathbb P_{\mathcal M}\) and \(\mathbb P_{\mathcal M'}\), and to the event \(A^c\), it gives
\(
\mathbb P_{\mathcal M}(A^c)
+
\mathbb P_{\mathcal M'}(A)
\ge
\frac{1}{2}
\exp\left(
-
\KL(\mathbb P_{\mathcal M}\|\mathbb P_{\mathcal M'})
\right) = 1/2,
\)
where we used \eqref{eqn:zero_KL} for the last equality. This completes the proof. 

\subsection{Lower bound with the asynchronous full coverage factor \texorpdfstring{\(\lambda_{\min}\)}{lambda-min}}
\label{subsec:lambda_min_lower_bound}

We now spell out the analogous construction in the asynchronous sampling setting, where the hard state-action pair is not observed at every time step. This is the case in which the coverage factor \(\lambda_{\min}\) appears. The preceding proof corresponds to the special case \(\lambda_{\min}=1\), since the unique state-action pair is observed at every time step. In the asynchronous setting, however, the behavior policy visits the informative pair only with probability \(\lambda_{\min}\). Hence, over a length-\(T\) transcript, the effective number of informative samples is of order \(\lambda_{\min}T\), and this is what changes the lower bound from
\begin{equation}
\frac{\varepsilon \bar R}{(1-\gamma)\sqrt T} \rightarrow
\frac{\varepsilon \bar R}{(1-\gamma)\sqrt{\lambda_{\min}T}}.
\end{equation}

\paragraph{Step 1. Construction of two asynchronous instances.}
We construct two MDPs that differ only in the reward mean of one hard state-action pair. Let the state space and action space be
\begin{equation}
\mc S=\{s^\star\},
\qquad
\mc A=\{a^\star,a^0\}.
\end{equation}
The state is always \(s^\star\), and both actions lead deterministically back to \(s^\star\):
\begin{equation}
\mathsf P(s^\star\mid s^\star,a^\star)=1,
\qquad
\mathsf P(s^\star\mid s^\star,a^0)=1.
\end{equation}
The behavior policy used to generate asynchronous data is fixed and satisfies
\begin{equation}
\mu_b(a^\star\mid s^\star)=\lambda_{\min},
\quad
\mu_b(a^0\mid s^\star)=1-\lambda_{\min}, \quad \lambda_{\min} \le 1/2.
\end{equation}
Since the state is always \(s^\star\), the visitation probability of the hard pair is
\begin{equation}
\lambda(s^\star,a^\star)
=
\lambda_{\min} \le 1/2.
\end{equation}
Thus, the pair \((s^\star,a^\star)\) is sampled only with probability \(\lambda_{\min}\) at each time step.

The two MDPs \(\mathcal M\) and \(\mathcal M'\) differ only in the reward distribution at the hard pair \((s^\star,a^\star)\). At this pair, define
\begin{equation}
\mathcal M:\quad r(s^\star,a^\star)\sim \mc P_1=\mathcal N(\mu_1,\bar R^2),
\qquad
\mathcal M':\quad r(s^\star,a^\star)\sim \mc P_2=\mathcal N(\mu_2,\bar R^2),
\end{equation}
where \(\mu_2>\mu_1>0\) will be chosen below. At the dummy pair \((s^\star,a^0)\), the reward law is identical under both MDPs:
\begin{equation}
\mathcal M:\quad r(s^\star,a^0)\sim \mc P_0,
\qquad
\mathcal M':\quad r(s^\star,a^0)\sim \mc P_0,
\end{equation}
where, for concreteness, we may take
\begin{equation}
\mc P_0=\mathcal N(0,\bar R^2).
\end{equation}
Therefore, all observations generated from the dummy action \(a^0\) have exactly the same distribution under \(\mathcal M\) and \(\mathcal M'\). The only informative observations are those generated when the behavior policy selects \(a^\star\).

We now describe the law of one honest agent's asynchronous transcript. At each time \(t\in\{1,\ldots,T\}\), the behavior policy chooses \(a^\star\) with probability \(\lambda_{\min}\) and \(a^0\) with probability \(1-\lambda_{\min}\). If \(a_t=a^\star\), then the reward is drawn from \(\mc P_1\) under \(\mathcal M\) and from \(\mc P_2\) under \(\mathcal M'\). If \(a_t=a^0\), then the reward is drawn from the common law \(\mc P_0\) under both MDPs. Let
\begin{equation}
X_t=(a_t,r_t)
\end{equation}
denote the observation at time \(t\), and let
\begin{equation}
X=(X_1,\ldots,X_T)
\end{equation}
denote the length-\(T\) transcript of one honest agent.

Let \(\mathfrak P_1^{(T)}\) and \(\mathfrak P_2^{(T)}\) denote the laws of \(X\) under \(\mathcal M\) and \(\mathcal M'\), respectively. Let \(\mathfrak p_1\) and \(\mathfrak p_2\) denote their densities with respect to the common product measure given by counting measure over actions and Lebesgue measure over rewards. For one sample \(x=(a,r)\), the corresponding one-step densities are
\begin{equation}
g_1(a,r)
=
\lambda_{\min}\mathbf 1\{a=a^\star\}p_1(r)
+
(1-\lambda_{\min})\mathbf 1\{a=a^0\}p_0(r),
\end{equation}
and
\begin{equation}
g_2(a,r)
=
\lambda_{\min}\mathbf 1\{a=a^\star\}p_2(r)
+
(1-\lambda_{\min})\mathbf 1\{a=a^0\}p_0(r),
\end{equation}
where \(p_0,p_1,p_2\) are the densities of \(\mc P_0,\mc P_1,\mc P_2\), respectively. Since samples are independent across time under the i.i.d. asynchronous sampling model, the honest transcript densities are
\begin{equation}
\mathfrak p_1(x_1,\ldots,x_T)
=
\prod_{t=1}^T g_1(x_t),
\qquad
\mathfrak p_2(x_1,\ldots,x_T)
=
\prod_{t=1}^T g_2(x_t).
\end{equation}

We next choose the separation between the two reward means. Let \(c_0>0\) be a sufficiently small universal constant, and set
\begin{equation}
\mu_2-\mu_1
=
c_0\frac{\varepsilon \bar R}{\sqrt{\lambda_{\min}T}}.
\label{eq:lambda_mean_choice}
\end{equation}
The common shift of the two means is chosen so that \(\mu_1>0\). This shift does not affect any total-variation or KL calculation, since these quantities depend only on \(\mu_2-\mu_1\).

We now verify that the two honest transcript laws are close enough to be hidden by an \(\varepsilon\)-fraction of adversarial agents. By Pinsker's inequality in Lemma~\ref{lem:pinsker}, we have the following:
\begin{equation}
\TV(\mathfrak P_1^{(T)},\mathfrak P_2^{(T)})
\le
\sqrt{
\frac{1}{2}
\KL(\mathfrak P_1^{(T)}\|\mathfrak P_2^{(T)})
}.
\end{equation}
The two one-step laws differ only when \(a_t=a^\star\), which happens with probability \(\lambda_{\min}\). Therefore, the KL divergence of one asynchronous sample is
\begin{equation}
\KL(g_1\|g_2)
=
\lambda_{\min}\KL(\mc P_1\|\mc P_2).
\end{equation}
Since the \(T\) samples are independent, the KL divergence between the length-\(T\) honest transcript laws is
\begin{equation}
\KL(\mathfrak P_1^{(T)}\|\mathfrak P_2^{(T)})
=
T\lambda_{\min}\KL(\mc P_1\|\mc P_2).
\end{equation}
For the Gaussian reward laws \(\mc P_1=\mathcal N(\mu_1,\bar R^2)\) and \(\mc P_2=\mathcal N(\mu_2,\bar R^2)\), we have
\begin{equation}
\KL(\mc P_1\|\mc P_2)
=
\frac{(\mu_1-\mu_2)^2}{2\bar R^2}.
\end{equation}
Thus,
\begin{equation}
\KL(\mathfrak P_1^{(T)}\|\mathfrak P_2^{(T)})
=
\frac{T\lambda_{\min}(\mu_1-\mu_2)^2}{2\bar R^2}.
\end{equation}
Using the choice \eqref{eq:lambda_mean_choice}, we obtain
\begin{equation}
\KL(\mathfrak P_1^{(T)}\|\mathfrak P_2^{(T)})
=
\frac{c_0^2\varepsilon^2}{2}.
\end{equation}
Therefore,
\begin{equation}
\TV(\mathfrak P_1^{(T)},\mathfrak P_2^{(T)})
\le
\frac{c_0\varepsilon}{2}.
\end{equation}
Choosing \(c_0\le 2\) gives
\begin{equation}
\TV(\mathfrak P_1^{(T)},\mathfrak P_2^{(T)})
\le
\varepsilon
\le
\frac{\varepsilon}{1-\varepsilon},
\label{eq:lambda_tv_upper}
\end{equation}
where the last inequality uses \(\varepsilon\in(0,1/2)\). Hence, the two honest asynchronous transcript laws are sufficiently close for the same contamination construction as before.

We now translate the reward-mean gap into a gap between the optimal \(Q\)-functions. Since \(\mu_1>0\) and \(\mu_2>\mu_1\), the hard action \(a^\star\) is optimal under both \(\mathcal M\) and \(\mathcal M'\). Indeed, under \(\mathcal M\), the Bellman optimality equations are
\begin{equation}
Q_{\mathcal M}^\star(s^\star,a^\star)
=
\mu_1
+
\gamma
\max\left\{
Q_{\mathcal M}^\star(s^\star,a^\star),
Q_{\mathcal M}^\star(s^\star,a^0)
\right\},
\end{equation}
and
\begin{equation}
Q_{\mathcal M}^\star(s^\star,a^0)
=
0
+
\gamma
\max\left\{
Q_{\mathcal M}^\star(s^\star,a^\star),
Q_{\mathcal M}^\star(s^\star,a^0)
\right\}.
\end{equation}
Since \(\mu_1>0\), the first equation shows that
\begin{equation}
Q_{\mathcal M}^\star(s^\star,a^\star)
>
Q_{\mathcal M}^\star(s^\star,a^0),
\end{equation}
and hence
\begin{equation}
Q_{\mathcal M}^\star(s^\star,a^\star)
=
\mu_1+\gamma Q_{\mathcal M}^\star(s^\star,a^\star).
\end{equation}
Solving gives
\begin{equation}
Q_{\mathcal M}^\star(s^\star,a^\star)
=
\frac{\mu_1}{1-\gamma}.
\end{equation}
The same argument under \(\mathcal M'\) gives
\begin{equation}
Q_{\mathcal M'}^\star(s^\star,a^\star)
=
\frac{\mu_2}{1-\gamma}.
\end{equation}
Consequently,
\begin{equation}
\left\|
Q_{\mathcal M}^\star
-
Q_{\mathcal M'}^\star
\right\|_\infty
\ge
\left|
Q_{\mathcal M}^\star(s^\star,a^\star)
-
Q_{\mathcal M'}^\star(s^\star,a^\star)
\right|
=
\frac{|\mu_2-\mu_1|}{1-\gamma}.
\end{equation}
Using \eqref{eq:lambda_mean_choice}, we obtain
\begin{equation}
\left\|
Q_{\mathcal M}^\star
-
Q_{\mathcal M'}^\star
\right\|_\infty
\ge
c_0
\frac{\varepsilon \bar R}{(1-\gamma)\sqrt{\lambda_{\min}T}}.
\label{eq:lambda_q_gap}
\end{equation}
Thus, in our asynchronous construction, the two optimal \(Q\)-functions are separated by order
\begin{equation}
\frac{\varepsilon \bar R}{(1-\gamma)\sqrt{\lambda_{\min}T}}.
\end{equation}

\paragraph{Step 2. Attack strategy and indistinguishability of corrupted asynchronous transcripts.}
We now describe the adversarial strategy. The goal of the adversary is to make the corrupted transcript observed by the server have the same distribution under \(\mathcal M\) and \(\mathcal M'\). In the asynchronous construction, the honest transcript laws are \(\mathfrak P_1^{(T)}\) and \(\mathfrak P_2^{(T)}\), with densities \(\mathfrak p_1\) and \(\mathfrak p_2\). By \eqref{eq:lambda_tv_upper}, these two laws satisfy
\begin{equation}
\TV(\mathfrak P_1^{(T)},\mathfrak P_2^{(T)})
\le
\frac{\varepsilon}{1-\varepsilon}.
\end{equation}

Define
\begin{equation}
\rho
=
\frac{1-\varepsilon}{\varepsilon}
\TV(\mathfrak P_1^{(T)},\mathfrak P_2^{(T)}).
\end{equation}
By the preceding display, \(\rho\le 1\). Let \(h\) be any probability density on the transcript space. Define two fake transcript densities \(\mathfrak q_1\) and \(\mathfrak q_2\) by
\begin{equation}
\mathfrak q_1(x)
=
\frac{1-\varepsilon}{\varepsilon}
\bigl(\mathfrak p_2(x)-\mathfrak p_1(x)\bigr)
\mathbf 1\{\mathfrak p_2(x)\ge \mathfrak p_1(x)\}
+
(1-\rho)h(x),
\end{equation}
and
\begin{equation}
\mathfrak q_2(x)
=
\frac{1-\varepsilon}{\varepsilon}
\bigl(\mathfrak p_1(x)-\mathfrak p_2(x)\bigr)
\mathbf 1\{\mathfrak p_1(x)\ge \mathfrak p_2(x)\}
+
(1-\rho)h(x).
\end{equation}
These are valid probability densities. Indeed, by the definition of total variation distance,
\begin{equation}
\int
\bigl(\mathfrak p_2(x)-\mathfrak p_1(x)\bigr)
\mathbf 1\{\mathfrak p_2(x)\ge \mathfrak p_1(x)\}\,dx
=
\TV(\mathfrak P_1^{(T)},\mathfrak P_2^{(T)}),
\end{equation}
and similarly,
\begin{equation}
\int
\bigl(\mathfrak p_1(x)-\mathfrak p_2(x)\bigr)
\mathbf 1\{\mathfrak p_1(x)\ge \mathfrak p_2(x)\}\,dx
=
\TV(\mathfrak P_1^{(T)},\mathfrak P_2^{(T)}).
\end{equation}
Therefore,
\begin{equation}
\int \mathfrak q_1(x)\,dx
=
\frac{1-\varepsilon}{\varepsilon}
\TV(\mathfrak P_1^{(T)},\mathfrak P_2^{(T)})
+
(1-\rho)
=
\rho+(1-\rho)
=
1,
\end{equation}
and likewise,
\begin{equation}
\int \mathfrak q_2(x)\,dx
=
1.
\end{equation}
Let \(\mathfrak Q_1\) and \(\mathfrak Q_2\) denote the probability distributions with densities \(\mathfrak q_1\) and \(\mathfrak q_2\), respectively.

We now specify the attack. Under \(\mathcal M\), an honest agent sends a length-\(T\) asynchronous transcript distributed according to \(\mathfrak P_1^{(T)}\), while an adversarial agent sends a fake transcript distributed according to \(\mathfrak Q_1\). Under \(\mathcal M'\), an honest agent sends a length-\(T\) asynchronous transcript distributed according to \(\mathfrak P_2^{(T)}\), while an adversarial agent sends a fake transcript distributed according to \(\mathfrak Q_2\). Each agent is adversarial independently with probability \(\varepsilon\). Therefore, the marginal distribution of one agent's corrupted message under \(\mathcal M\) is
\begin{equation}
\mathfrak D_{\mathcal M}
=
(1-\varepsilon)\mathfrak P_1^{(T)}
+
\varepsilon \mathfrak Q_1,
\end{equation}
whereas the marginal distribution of one agent's corrupted message under \(\mathcal M'\) is
\begin{equation}
\mathfrak D_{\mathcal M'}
=
(1-\varepsilon)\mathfrak P_2^{(T)}
+
\varepsilon \mathfrak Q_2.
\end{equation}

We claim that these two one-agent corrupted message distributions are exactly identical. To see this, compare their densities pointwise. If \(\mathfrak p_2(x)\ge \mathfrak p_1(x)\), then
\begin{equation}
\mathfrak q_1(x)
=
\frac{1-\varepsilon}{\varepsilon}
\bigl(\mathfrak p_2(x)-\mathfrak p_1(x)\bigr)
+
(1-\rho)h(x),
\end{equation}
and
\begin{equation}
\mathfrak q_2(x)
=
(1-\rho)h(x).
\end{equation}
Hence,
\begin{equation}
\begin{aligned}
(1-\varepsilon)\mathfrak p_1(x)+\varepsilon \mathfrak q_1(x)
&=
(1-\varepsilon)\mathfrak p_1(x)
+
(1-\varepsilon)
\bigl(\mathfrak p_2(x)-\mathfrak p_1(x)\bigr)
+
\varepsilon(1-\rho)h(x)
\\
&=
(1-\varepsilon)\mathfrak p_2(x)
+
\varepsilon(1-\rho)h(x)
\\
&=
(1-\varepsilon)\mathfrak p_2(x)
+
\varepsilon \mathfrak q_2(x).
\end{aligned}
\end{equation}
On the other hand, if \(\mathfrak p_1(x)\ge \mathfrak p_2(x)\), then
\begin{equation}
\mathfrak q_1(x)
=
(1-\rho)h(x),
\end{equation}
and
\begin{equation}
\mathfrak q_2(x)
=
\frac{1-\varepsilon}{\varepsilon}
\bigl(\mathfrak p_1(x)-\mathfrak p_2(x)\bigr)
+
(1-\rho)h(x).
\end{equation}
Hence,
\begin{equation}
\begin{aligned}
(1-\varepsilon)\mathfrak p_1(x)+\varepsilon \mathfrak q_1(x)
&=
(1-\varepsilon)\mathfrak p_1(x)
+
\varepsilon(1-\rho)h(x)
\\
&=
(1-\varepsilon)\mathfrak p_2(x)
+
(1-\varepsilon)
\bigl(\mathfrak p_1(x)-\mathfrak p_2(x)\bigr)
+
\varepsilon(1-\rho)h(x)
\\
&=
(1-\varepsilon)\mathfrak p_2(x)
+
\varepsilon \mathfrak q_2(x).
\end{aligned}
\end{equation}
Therefore,
\begin{equation}
(1-\varepsilon)\mathfrak P_1^{(T)}
+
\varepsilon \mathfrak Q_1
=
(1-\varepsilon)\mathfrak P_2^{(T)}
+
\varepsilon \mathfrak Q_2,
\end{equation}
or equivalently,
\begin{equation}
\mathfrak D_{\mathcal M}
=
\mathfrak D_{\mathcal M'}.
\label{eq:lambda_single_agent_identity}
\end{equation}

Now let \(Z=(Z_1,\ldots,Z_N)\) denote the full corrupted federated transcript observed by the server, where \(Z_i\) is the corrupted message sent by agent \(i\). Since agents are independent under the random-adversary model, the transcript laws under \(\mathcal M\) and \(\mathcal M'\) are
\begin{equation}
\mathbb P_{\mathcal M}
=
\mathfrak D_{\mathcal M}^{\otimes N},
\qquad
\mathbb P_{\mathcal M'}
=
\mathfrak D_{\mathcal M'}^{\otimes N}.
\end{equation}
Using \eqref{eq:lambda_single_agent_identity}, we conclude that
\begin{equation}
\mathbb P_{\mathcal M}
=
\mathbb P_{\mathcal M'}.
\label{eq:lambda_full_transcript_identity}
\end{equation}
Thus, under the above attack strategy, the server observes exactly the same corrupted transcript distribution under \(\mathcal M\) and \(\mathcal M'\). This is the desired indistinguishability property in the asynchronous setting.

\paragraph{Step 3. Reduction from asynchronous indistinguishability to estimation error.}
We now use the indistinguishability of the corrupted transcript laws to prove the lower bound with the \(\lambda_{\min}\) factor. Let \(\widehat Q\) be any estimator constructed from the corrupted federated transcript \(Z\). Define
\begin{equation}
\Delta_\lambda
=
\left\|
Q_{\mathcal M}^\star
-
Q_{\mathcal M'}^\star
\right\|_\infty.
\end{equation}
From \eqref{eq:lambda_q_gap}, we have
\begin{equation}
\Delta_\lambda
\ge
c_0
\frac{\varepsilon \bar R}{(1-\gamma)\sqrt{\lambda_{\min}T}}.
\label{eq:lambda_delta_lower_bound}
\end{equation}
Consider the event
\begin{equation}
A_\lambda
=
\left\{
\left\|
\widehat Q
-
Q_{\mathcal M}^\star
\right\|_\infty
<
\frac{\Delta_\lambda}{2}
\right\}.
\end{equation}
On the complement \(A_\lambda^c\), the estimator is at distance at least \(\Delta_\lambda/2\) from \(Q_{\mathcal M}^\star\). Hence,
\begin{equation}
\mathbb P_{\mathcal M}
\left(
\left\|
\widehat Q
-
Q_{\mathcal M}^\star
\right\|_\infty
\ge
\frac{\Delta_\lambda}{2}
\right)
\ge
\mathbb P_{\mathcal M}(A_\lambda^c).
\label{eq:lambda_error_m}
\end{equation}
On the event \(A_\lambda\), the estimator is within distance \(\Delta_\lambda/2\) of \(Q_{\mathcal M}^\star\). Since \(Q_{\mathcal M}^\star\) and \(Q_{\mathcal M'}^\star\) are separated by \(\Delta_\lambda\), the triangle inequality gives
\begin{equation}
\left\|
\widehat Q
-
Q_{\mathcal M'}^\star
\right\|_\infty
\ge
\left\|
Q_{\mathcal M}^\star
-
Q_{\mathcal M'}^\star
\right\|_\infty
-
\left\|
\widehat Q
-
Q_{\mathcal M}^\star
\right\|_\infty
>
\frac{\Delta_\lambda}{2}.
\end{equation}
Therefore,
\begin{equation}
\mathbb P_{\mathcal M'}
\left(
\left\|
\widehat Q
-
Q_{\mathcal M'}^\star
\right\|_\infty
\ge
\frac{\Delta_\lambda}{2}
\right)
\ge
\mathbb P_{\mathcal M'}(A_\lambda).
\label{eq:lambda_error_mprime}
\end{equation}
Combining \eqref{eq:lambda_error_m} and \eqref{eq:lambda_error_mprime}, we obtain
\begin{equation}
\max\left\{
\mathbb P_{\mathcal M}
\left(
\left\|
\widehat Q
-
Q_{\mathcal M}^\star
\right\|_\infty
\ge
\frac{\Delta_\lambda}{2}
\right),
\mathbb P_{\mathcal M'}
\left(
\left\|
\widehat Q
-
Q_{\mathcal M'}^\star
\right\|_\infty
\ge
\frac{\Delta_\lambda}{2}
\right)
\right\}
\ge
\frac{1}{2}
\left(
\mathbb P_{\mathcal M}(A_\lambda^c)
+
\mathbb P_{\mathcal M'}(A_\lambda)
\right).
\label{eq:lambda_testing_reduction}
\end{equation}

We now apply the Bretagnolle-Huber inequality to the two corrupted transcript laws \(\mathbb P_{\mathcal M}\) and \(\mathbb P_{\mathcal M'}\), and to the event \(A_\lambda^c\). This gives
\begin{equation}
\mathbb P_{\mathcal M}(A_\lambda^c)
+
\mathbb P_{\mathcal M'}(A_\lambda)
\ge
\frac{1}{2}
\exp\left(
-
\KL(\mathbb P_{\mathcal M}\|\mathbb P_{\mathcal M'})
\right).
\label{eq:lambda_bh_step}
\end{equation}
It remains to compute the KL divergence between the two corrupted transcript laws. By \eqref{eq:lambda_full_transcript_identity}, we have
\begin{equation}
\mathbb P_{\mathcal M}
=
\mathbb P_{\mathcal M'}.
\end{equation}
Therefore,
\begin{equation}
\KL(\mathbb P_{\mathcal M}\|\mathbb P_{\mathcal M'})
=
0.
\label{eq:lambda_kl_zero}
\end{equation}
Substituting \eqref{eq:lambda_kl_zero} into \eqref{eq:lambda_bh_step} gives
\begin{equation}
\mathbb P_{\mathcal M}(A_\lambda^c)
+
\mathbb P_{\mathcal M'}(A_\lambda)
\ge
\frac{1}{2}.
\end{equation}
Combining this with \eqref{eq:lambda_testing_reduction}, we obtain
\begin{equation}
\max\left\{
\mathbb P_{\mathcal M}
\left(
\left\|
\widehat Q
-
Q_{\mathcal M}^\star
\right\|_\infty
\ge
\frac{\Delta_\lambda}{2}
\right),
\mathbb P_{\mathcal M'}
\left(
\left\|
\widehat Q
-
Q_{\mathcal M'}^\star
\right\|_\infty
\ge
\frac{\Delta_\lambda}{2}
\right)
\right\}
\ge
\frac{1}{4}.
\label{eq:lambda_two_point_testing}
\end{equation}
Finally, by \eqref{eq:lambda_delta_lower_bound},
\begin{equation}
\frac{\Delta_\lambda}{2}
\ge
\frac{c_0}{2}
\frac{\varepsilon \bar R}{(1-\gamma)\sqrt{\lambda_{\min}T}}.
\end{equation}
Thus,
\begin{equation}
\max\left\{
\mathbb P_{\mathcal M}
\left(
\left\|
\widehat Q
-
Q_{\mathcal M}^\star
\right\|_\infty
\ge
\frac{c_0}{2}
\frac{\varepsilon \bar R}{(1-\gamma)\sqrt{\lambda_{\min}T}}
\right),
\mathbb P_{\mathcal M'}
\left(
\left\|
\widehat Q
-
Q_{\mathcal M'}^\star
\right\|_\infty
\ge
\frac{c_0}{2}
\frac{\varepsilon \bar R}{(1-\gamma)\sqrt{\lambda_{\min}T}}
\right)
\right\}
\ge
\frac{1}{4}.
\label{eq:lambda_two_point_final_rate}
\end{equation}
This proves the lower bound with the asynchronous coverage factor \(\lambda_{\min}\). The factor appears because the hard state-action pair \((s^\star,a^\star)\) is sampled only with probability \(\lambda_{\min}\), so the informative part of a length-\(T\) transcript has effective size \(\lambda_{\min}T\), rather than \(T\).
\section{Proof of Theorem \ref{theorem:markov}}
\label{app:Markovian}
We extend the analysis of \texttt{Robust Async-Fed-Q} from i.i.d. asynchronous sampling to Markovian data under \((\textcolor{winered}{\texttt{S}_1})\); the same argument applies to \((\textcolor{winered}{\texttt{S}_2})\). To control temporal dependence, each honest agent samples its trajectory by retaining every \(\tau\)-th transition, where \(\tau\) denotes the curated sub-sampling gap defined in Section~\ref{sec:lMarkov}. The resulting thinned samples can be coupled with i.i.d. samples from the stationary distribution, allowing the preceding analysis to be transferred to the Markovian setting.

$\bullet$ \textbf{Background.} Let \(\{Z_t\}_{t\ge 0}\) be an ergodic
time-homogeneous Markov chain over a finite state space \(\Omega\) with
stationary distribution \(\rho\). Define
\begin{equation}
    d_{\mathrm{mix}}(t)
    :=
    \sup_{z\in\Omega}
    D_{\mathrm{TV}}\left(
    \mathbb P(Z_t\in\cdot\mid Z_0=z),\rho
    \right).
\end{equation}
The quantity \(d_{\mathrm{mix}}(t)\) measures the worst-case total-variation
distance between the distribution of the chain at time \(t\) and its stationary
distribution. It is non-increasing in \(t\). We define the mixing time as
\(
    \bar\tau
    :=
    \inf\left\{
    t\ge 1:\ d_{\mathrm{mix}}(t)\le \frac14
    \right\}.
\)
Intuitively, \(\bar\tau\) measures how quickly the Markov chain approaches
stationarity. A standard consequence of this definition is the geometric decay
property~\cite{dorfman}:
\(
    d_{\mathrm{mix}}(\ell\bar\tau)
    \le
    2^{-\ell},
    \forall \ell\in\mathbb N .
\)
The key idea is to choose a sub-sampling interval large enough that the entire retained state-action-next-state sequence can be coupled to i.i.d. stationary samples with high-probability agreement. Since the rewards in our setting are i.i.d. auxiliary randomness, the Markovian coupling is needed only for the state-action-next-state process.

Fix an agent \(i\in[N]\). Under its behavior policy \(\mu_i\), let
\(
    Z_{i,t}:=(s_{i,t},a_{i,t},s_{i,t+1})
\)
denote the state-action-next-state process. By Assumption~\ref{ass:ergodic}, \(\{Z_{i,t}\}_{t\ge 0}\) is an ergodic Markov chain with stationary distribution \(\rho_i\). Suppose the chain is initialized from \(\rho_i\). Let \(\bar\tau_i\) denote the mixing time of this chain, and define
\(
    \bar\tau:=\max_{i\in[N]}\bar\tau_i .
\)
For a prescribed confidence level \(\delta\in(0,1)\), set
\(
    \ell:=\left\lceil \log_{2}(2NT/\delta) \right\rceil,
    \quad
    \tau:=\lceil \ell\bar\tau\rceil .
\)
In the sub-sampled version of \texttt{Robust Async-Fed-Q}, each honest
agent uses only every \(\tau\)-th transition from its observed Markov
trajectory. Thus, over the full horizon \(T=KH\), agent \(i\) uses
\[
    Z_{i,0}, Z_{i,\tau}, Z_{i,2\tau},\ldots,Z_{i,(\overline T-1)\tau},
    \quad
    \overline T:=\left\lfloor \frac{T}{\tau}\right\rfloor .
\]
Let
\(
    \widetilde Z_{i,0},\widetilde Z_{i,\tau},\ldots,
    \widetilde Z_{i,(\overline T-1)\tau}
    \sim \rho_i^{\otimes \overline T}
\)
be i.i.d. samples from the stationary distribution of agent \(i\)'s Markov chain.
We now use the following key result from \cite{dorfman}, which bounds
the probability that a sub-sampled Markov chain differs from its i.i.d.
counterpart generated from the chain's stationary distribution.  
\begin{theorem}[\textbf{Coupling};\cite{dorfman}]\label{thm:coupling}
    Let $Z_0,Z_1,\cdots$ be a stationary finite-state Markov chain with stationary distribution $\rho$, and let $K,n\in\mathbb{N}$. Then, we can couple $(Z_0,Z_{K},\cdots, Z_{(n-1)K})$ and $(\tilde Z_0,\tilde Z_{K},\cdots, \tilde Z_{(n-1)K})\sim\rho^{\otimes n}$, such that
    \begin{equation}
        \P\paren{\{Z_0,Z_{K},\cdots, Z_{(n-1)K}\}\neq \{\tilde Z_0,\tilde Z_{K},\cdots, \tilde Z_{(n-1)K}\}}\leq (n-1) d_{mix}(K).
    \end{equation}
\end{theorem}
\noindent By Theorem~\ref{thm:coupling}, we can couple the sub-sampled Markovian sequence and its i.i.d. stationary counterpart such that
\begin{equation}
\begin{aligned}
&\mathbb P\left(
    (Z_{i,0},Z_{i,\tau},\ldots,Z_{i,(\overline T-1)\tau})
    \neq
    (\widetilde Z_{i,0},\widetilde Z_{i,\tau},\ldots,
    \widetilde Z_{i,(\overline T-1)\tau})
\right)\\
&\leq (\overline T-1)d_{i,\mathrm{mix}}(\tau)\leq \overline T d_{i,\mathrm{mix}}(\tau)\\
& = \overline T
d_{i,\mathrm{mix}}\!\left(\left\lceil\ell\bar\tau\right\rceil\right) \overset{(\bullet)}{\le} \overline T
d_{i,\mathrm{mix}}\!\left(\ell\bar\tau_i\right)\overset{(\bullet \bullet)}{\leq} \overline T\,2^{-\ell}\\
&\leq \frac{T}{\tau}2^{-\ell}
\leq T2^{-\ell}
\overset{(\bullet \bullet \bullet)}{\leq} \frac{\delta}{2N},
\end{aligned}
\label{eq:agent_markov_coupling}
\end{equation}
where \((\bullet)\) follows from the non-increasing property of
\(d_{i,\mathrm{mix}}(t)\) and the fact that
\(
\lceil\ell\bar\tau\rceil
\geq \ell\bar\tau
= \ell\max_{j\in[N]}\bar\tau_j
\geq \ell\bar\tau_i
\)
for all \(i\in[N]\). Next, \((\bullet\bullet)\) follows from the
geometric mixing bound \(d_{i,\mathrm{mix}}(\ell\bar\tau)\leq 2^{-\ell}\)
for all \(\ell\in\mathbb N\), established in \cite{dorfman}. We also use
\(\overline T\leq T/\tau\) and \(\tau\geq 1\), while
\((\bullet\hspace{-1mm}\bullet\hspace{-1mm}\bullet)\) follows from the choice
\(\ell=\lceil\log_2(2NT/\delta)\rceil\). Indeed, this choice ensures that
\(2^{-\ell}\leq \delta/(2NT)\). Multiplying by \(T/\tau\) gives
\((T/\tau)2^{-\ell}\leq\delta/(2N\tau)\leq\delta/(2N)\),
since \(\tau\) is a positive integer. Taking a union bound over all the agents \(i \in [N]\), we obtain an event \(\mc H\) such that
\(
    \mathbb P(\mc H)\ge 1-\frac{\delta}{2},
\)
and, on \(\mc H\), the sub-sampled Markovian state-action-next-state tuples agree with their i.i.d. stationary counterparts simultaneously for all agents over the entire horizon. That is,
\begin{equation}
    \left(Z_{i,0},Z_{i,\tau},\ldots,Z_{i,(\overline T-1)\tau}\right)
    =
    \left(\widetilde Z_{i,0},\widetilde Z_{i,\tau},\ldots,
    \widetilde Z_{i,(\overline T-1)\tau}\right)
    \quad
    \forall i\in[N].
\label{eq:all_agent_markov_coupling}
\end{equation}

We now incorporate the reward randomness. In our setting, the Markovian dependence enters only through \(Z_{i,t}=(s_{i,t},a_{i,t},s_{i,t+1})\), while the rewards are generated independently across agents and time conditional on the visited pair. 

We now consider two executions of the algorithm: one on the sub-sampled
Markov data and the other on the coupled i.i.d. data. For each honest
agent \(i\) and retained index \(m\), suppose the Markovian execution
observes stochastic reward \(r_{i,m\tau}\) with mean $R(s,a)$ on transition \(Z_{i,m\tau}=(s,a,s')\). On the event \(\mc{H}\), whenever \(Z_{i,m\tau}=\widetilde Z_{i,m\tau}\), we assign the same reward
\(r_{i,m\tau}\) to the corresponding transition in the i.i.d. execution.
This coupling is valid because both rewards are drawn from the same
reward law at \((s,a)\). Hence, on \(\mc H\), the two executions have
identical retained state-action-next-state-reward samples. Hence the retained reward-augmented samples agree on \(\mc H\): if \(W_{i,m}^{\tau}:=(s_{i,m\tau},a_{i,m\tau},r_{i,m\tau},s_{i,m\tau+1})\) and \(\widetilde W_{i,m}^{\tau}:=(\widetilde s_{i,m\tau},\widetilde a_{i,m\tau},r_{i,m\tau},\widetilde s_{i,m\tau+1})\), then \(W_{i,m}^{\tau}=\widetilde W_{i,m}^{\tau}\) for all \(i\in[N]\) and \(m\in\{0,\ldots,\overline T-1\}\). Consequently, the sub-sampled reward-augmented datasets before message corruption, defined by
\[
\mc D^\tau
:=\{W_{i,m}^{\tau}:i\in[N],\,m=0,\ldots,\overline T-1\}, \widetilde{\mc D}^{\tau}
:=\{\widetilde W_{i,m}^{\tau}:i\in[N],\,m=0,\ldots,\overline T-1\}\]
satisfy
\[\mc D^\tau=\widetilde{\mc D}^{\tau} \hspace{1 mm} \text{on} \hspace{1 mm} \mc H .\]

Now, let $\mc A$ denote the measurable output map of the \emph{sub-sampled}
\texttt{Robust Async-Fed-Q} algorithm, and let
$\textcolor{myblue}{\mc U}$ collect all auxiliary
randomness, including the sample splits in
\texttt{trimSC} and the construction of corrupted
messages. Under a common realization of $\textcolor{myblue}{\mc U}$, the
outputs of the sub-sampled Markovian and coupled i.i.d.\ experiments are
\(
Q_K=\mc A(\mc D^\tau;\textcolor{myblue}{\mc U}),
\quad
\widetilde Q_K
=\mc A(\widetilde{\mc D}^{\tau};\textcolor{myblue}{\mc U}).
\)
Thus, whenever the retained data agree, all subsequent algorithmic
operations agree as well. In particular, since
$\mc D^\tau=\widetilde{\mc D}^{\tau}$ on $\mc H$, we have
the following:
\begin{equation}
    Q_K
    =
    \mc A(\mc D^\tau;\textcolor{myblue}{\mc U})
    =
    \mc A(\widetilde{\mc D}^{\tau};\textcolor{myblue}{\mc U})
    =
    \widetilde Q_K
    \quad
    \text{on}\quad\mc H .
\label{eq:algorithm_output_equality_on_HZ}
\end{equation}
Therefore, for any \(\lambda>0\),
\begin{equation}
\begin{aligned}
&\mathbb P\!\left(\|Q_K-Q^\star\|_\infty\ge \lambda\right)
=
\mathbb P\!\left(\|Q_K-Q^\star\|_\infty\ge \lambda,\mc H\right)
+
\mathbb P\!\left(\|Q_K-Q^\star\|_\infty\ge \lambda,\mc H^c\right) \\
&\le
\mathbb P\!\left(\|\widetilde Q_K-Q^\star\|_\infty\ge \lambda,\mc H\right)
+
\mathbb P(\mc H^c) \le
\mathbb P\!\left(\|\widetilde Q_K-Q^\star\|_\infty\ge \lambda\right)
+
\mathbb P(\mc H^c).
\end{aligned}
\label{eq:markov_to_iid_algorithm_map}
\end{equation}
Since $\mathbb P(\mc H^c)\leq\delta/2$, it remains to control the first term in \eqref{eq:markov_to_iid_algorithm_map}. As $\widetilde Q_K$ is the i.i.d.\ asynchronous output with effective epoch length $\overline H=\lfloor H/\tau\rfloor$, the i.i.d.\ bound applies with $H$ replaced by $\overline H$. Adding the coupling failure probability then yields the Markovian guarantee.

\bibliographystyle{unsrt}
\bibliography{refs}
\end{document}